\documentclass[twoside,11pt]{article}

\usepackage{natbib}
\usepackage[colorlinks=false,allbordercolors={1 1 1}]{hyperref}
\usepackage{url}

\usepackage{epsfig}
\usepackage{graphicx}
\usepackage{amsmath}
\usepackage{amssymb}
\usepackage{amsthm}
\usepackage{algorithm}
\usepackage{algorithmic}
\usepackage{url}

\newtheorem{definition}{Definition}
\newtheorem{theorem}[definition]{Theorem}
\newtheorem{lemma}[definition]{Lemma}
\newtheorem{remark}[definition]{Remark}
\newtheorem{proposition}[definition]{Proposition}
\newtheorem{corollary}[definition]{Corollary}
\newcommand{\minus}[1]{{-#1}}
\newcommand{\NE}{{\mathcal{NE}}}
\newcommand{\G}{{\mathcal{G}}}

\newcommand{\argmax}{{\arg\max}}
\newcommand{\R}{\mathbb{R}}
\newcommand{\zero}{\mathbf{0}}
\newcommand{\one}{\mathbf{1}}
\newcommand{\I}{\mathbf{I}}
\renewcommand{\t}[1]{{#1}^{\rm T}}
\newcommand{\iverson}[1]{1[{#1}]}
\newcommand{\x}{\mathbf{x}}
\newcommand{\p}{\mathbf{p}}
\newcommand{\D}{\mathcal{D}}
\newcommand{\W}{\mathbf{W}}
\newcommand{\w}[1]{\mathbf{w}_{{#1},{\minus{#1}}}}
\renewcommand{\b}{\mathbf{b}}
\renewcommand{\v}{\mathbf{v}}
\newcommand{\y}{\mathbf{y}}
\newcommand{\si}[1]{^{(#1)}}
\renewcommand{\O}{\mathcal{O}}
\newcommand{\E}{\mathbb{E}}
\renewcommand{\P}{\mathbb{P}}
\newcommand{\PD}{\mathcal{P}}

\newcommand{\DD}{\mathcal{Q}}
\renewcommand{\exp}[1]{e^{#1}}
\newcommand{\greekbf}[1]{\text{\boldmath $#1$}}
\newcommand{\ellh}{\widehat{\ell}}
\newcommand{\Lh}{\widehat{\mathcal{L}}}
\newcommand{\Gh}{\widehat{\G}}
\newcommand{\qh}{\widehat{q}}
\newcommand{\pih}{\widehat{\pi}}
\newcommand{\thetah}{\widehat{\theta}}
\newcommand{\Lb}{\bar{\mathcal{L}}}
\newcommand{\Gb}{\bar{\G}}
\newcommand{\qb}{\bar{q}}
\newcommand{\pib}{\bar{\pi}}
\newcommand{\thetab}{\bar{\theta}}
\newcommand{\qt}{\widetilde{q}}
\newcommand{\N}{\mathcal{N}}
\newcommand{\A}{\mathcal{A}}

\newcommand{\Z}{\mathcal{Z}}
\newcommand{\ph}{\widehat{p}}
\newcommand{\HS}{\mathcal{H}}
\newcommand{\PS}{\Upsilon}
\newcommand{\hns}{\hspace{-0.025in}}
\newcommand{\q}{\mathbf{q}}
\newcommand{\z}{\mathbf{z}}
\newcommand{\captionx}[1]{\caption{\footnotesize{#1}}} 
\newcommand{\hl}[1]{\textit{\textbf{#1}}}
\renewcommand{\cite}[1]{\somethingunexistant}

\begin{document}

\title{\textbf{Learning the Structure and Parameters of Large-Population Graphical Games from Behavioral Data}}

\author{Jean Honorio\\
Computer Science and Artificial Intelligence Laboratory\\
Massachusetts Institute of Technology\\ 
Cambridge, MA 02139, USA\\
\texttt{jhonorio@csail.mit.edu}\\ \\
Luis Ortiz\\
Department of Computer Science\\
Stony Brook University\\
Stony Brook, NY 11794-4400, USA\\
\texttt{leortiz@cs.stonybrook.edu}}

\date{\vspace{-0.2in}}

\maketitle

\begin{abstract}%
We consider learning, from \emph{strictly} behavioral data, the
structure and parameters of \emph{linear influence games (LIGs)}, a
class of parametric graphical games introduced
by~\citet{Irfan13}. LIGs facilitate \emph{causal strategic inference (CSI)}:
Making inferences from causal interventions on stable behavior in
strategic settings. Applications include the identification of the
most influential individuals in large (social) networks. Such tasks
can also support policy-making analysis. Motivated by the computational
work on LIGs, we cast the learning problem as maximum-likelihood
estimation (MLE) of a generative model defined by \emph{pure-strategy
  Nash equilibria (PSNE)}. Our simple formulation uncovers the
fundamental interplay between goodness-of-fit and model complexity:
good models capture equilibrium behavior within the data while
controlling the true number of equilibria, including those unobserved.
We provide a generalization bound establishing the sample complexity
for MLE in our framework. We propose several algorithms including
\emph{convex loss minimization (CLM)} and sigmoidal approximations.
We prove that the number of exact PSNE in LIGs is small, with high probability; thus, CLM is sound.  We illustrate our approach on synthetic data and real-world U.S. congressional voting records. We briefly discuss our learning framework's generality and potential applicability to general graphical games.
\end{abstract}

\section{Introduction} \label{SecIntro}

Game theory has become a central tool for modeling multi-agent
  systems in AI.
Non-cooperative game theory has been considered as the appropriate
mathematical framework in which to formally study
  \emph{strategic} behavior in multi-agent scenarios.\footnote{See, e.g., the
  survey of~\citet{shoham08} and the books of~\citet{nisan07} and~\citet{shoham09} for more information.} The core
  solution concept of \emph{Nash equilibrium (NE)}~\citep{nash51} serves a descriptive role of the
  stable outcome of the overall behavior of systems involving
  self-interested individuals interacting strategically with each
  other in distributed settings for which no
  direct
  global control is possible. NE is also often used in
  predictive roles as the basis for what one might call \emph{causal
    strategic inference}, i.e., inferring the results of
  causal interventions on stable actions/behavior/outcomes in
  strategic settings (See, e.g.,
  \citealt{ballesteretal04,ballesteretal06,healandkunreuther03,healandkunreuther06,healandkunreuther07,kunreutherandmichel-kerjan07,ortizandkearns02,kearns05,Irfan13},
  and the references therein). Needless to say, the computation and analysis of
  NE in games is of significant interest to the computational
  game-theory community within AI.

The introduction of compact representations to game
  theory over the last decade have extended computational/algorithmic game theory's potential
  for large-scale, practical applications often encountered in the
  real-world. For the most part, such game model representations are analogous to probabilistic
  graphical models widely used in machine learning and
  AI.\footnote{The fundamental property such compact representation
    of games exploit is that of \emph{conditional
    independence}: each player's
    payoff function values are determined by the actions of the player and
    those of the player's neighbors only, and thus are \emph{conditionally
    (payoff) independent} of the
    actions of the non-neighboring players, \emph{given} the action of
    the neighboring players.} 
  Introduced within the AI community about
  a decade ago, \emph{graphical games}~\citep{kearns01} constitute an
  example of one of the first and arguably one of the most influential graphical models
  for game theory.\footnote{Other game-theoretic graphical models
    include \emph{game networks}~\citep{lamura00}, \emph{multi-agent influence diagrams
      (MAIDs)}~\citep{koller03}, and \emph{action-graph games}~\citep{AGG-full}.}

There has been considerable progress on problems of \emph{computing}
classical equilibrium solution concepts such as NE and
\emph{correlated equilibria (CE)}~\citep{aumann74} in graphical games (see, e.g., \citealt{kearns01,vickreyandkoller02,ortizandkearns02,blumetal06,kakadeetal03,papadimitriou08,jiangandleytonbrown11} and the references therein).
Indeed, graphical games played a prominent role in establishing the computational complexity of computing NE in general normal-form games (see, e.g., \citealt{daskalakisetal09} and the references therein).

An example of a recent computational application of
  non-cooperative game-theoretic graphical modeling and \emph{causal strategic inference (CSI)} that
  motivates the current paper is the work of
  \citet{Irfan13}.
  They proposed a new approach to the
  study of influence and the identification of the ``most
  influential'' individuals (or nodes)  in large (social)
  networks. Their approach is strictly game-theoretic in the sense
  that it relies on non-cooperative game theory and the central
 concept of \emph{pure-strategy Nash equilibria (PSNE)}\footnote{In
   this paper, because we concern ourselves primarily with PSNE,
   whenever we use the term ``equilibrium'' or ``equilibria'' without qualification, we
   mean PSNE.}
 as an 
 approximate \emph{predictor} of \emph{stable behavior} in
 \emph{strategic} settings, and, unlike other models of behavior in
 mathematical sociology,\footnote{Some of these models have recently
   gained interest and have been studied within computer science, specially those related to \emph{diffusion or
 contagion processes} (see, e.g.,
\citealt{granovetter78,morris00,domingos01,domingos05,even-dar07}).} it is
not interested and thus avoids explicit modeling of the complex dynamics by
 which such stable outcomes could have arisen or could be
 achieved. Instead, it concerns itself with the ``bottom-line''
 end-state stable outcomes (or steady state behavior). Hence, the proposed approach provides an alternative to models
  based on the diffusion of behavior through a social network
  (See~\citealt{kleinberg07} for an introduction and discussion targeted
  to computer scientists, and further references).

\emph{The underlying assumption for most work in computational game
  theory that deals with algorithms for computing equilibrium concepts
  is that the games under consideration
are already available, or have been ``hand-designed'' by the
analyst.} While this may be possible for systems involving a handful of
players, it is in general impossible in systems with at least tens of agent entities, if
not more, as we are interested in this paper.\footnote{Of course,
  modeling and hand-crafting games for systems with many agents may be possible if the system has
  particular structure one could exploit. To give an example, this
  would be analogous to how one
  can exploit the probabilistic structure
  of HMMs to deal with long stochastic processes
  in a representationally succinct and computationally tractable
  way. Yet, we believe it is fair to say that such systems are
  largely/likely the
  exception in real-world settings in practice.} For instance, in their paper,~\citet{Irfan13} propose a class of
games, called \emph{influence games}. In
particular, they concentrate on \emph{linear influence games (LIGs)},
and, as briefly mentioned above, study  a variety of
\emph{computational} problems resulting from their approach,
\emph{assuming such games are given as input}.

Research in computational game theory has paid relatively little
attention to the problem of \emph{learning} (both the structure and
parameters of) graphical games from data.
Addressing this problem is essential to the development, potential use and success of game-theoretic models in practical applications.
Indeed, we are beginning to see an increase in the availability of data collected from processes that are the result of deliberate actions of agents in complex system.
A lot of this data results from the interaction of a large number of
individuals, being not only people (i.e., individual human
decision-makers), but also companies, governments, groups or engineered autonomous systems (e.g., autonomous trading agents), for which any form of global control is usually weak.
The Internet is currently a major source of such data, and the smart grid, with its trumpeted ability to allow individual customers to install autonomous control devices and systems for electricity demand, will likely be another one in the near future.

In \emph{this} paper,
we investigate in considerable technical depth the problem of \emph{learning
LIGs from strictly behavioral data}: We do not
assume the availability of utility, payoff or cost information in the
data; the problem is \emph{precisely} to infer that information from just
the joint behavior collected in the data, \emph{up to the degree needed to
explain the joint behavior itself}.
\emph{We expect that, in most cases, the parameters quantifying a utility function or best-response condition are unavailable and hard to determine in real-world settings.}
The availability of data resulting from the observation of an
individual \emph{public behavior} is arguably a weaker assumption than
the availability of individual \emph{utility} observations, which are
often \emph{private}. In addition, \emph{we do not
assume prior knowledge of the conditional payoff/utility independence structure as
represented by the game graph.}

Motivated by the work of~\citet{Irfan13} on a
  strictly non-cooperative game-theoretic approach to influence and
  strategic behavior in networks, we present a formal
framework and design algorithms for learning the structure
and parameters of LIGs with a \emph{large} number of players.
We concentrate on data about what one might call ``the bottom line:''
i.e., data about``end-states'', ``steady-states'' or final behavior as
represented by \emph{possibly noisy} samples of joint
actions/pure-strategies from stable outcomes, which we assume come from
a \emph{hidden} underlying game. Thus, we do not use, consider or
assume available any \emph{temporal} data about the detailed
behavioral \emph{dynamics}. In fact, the data we consider does not
contain the dynamics that might have possibly led to the potentially
stable joint-action outcome! Since scalability is one of our main
goals, we aim to propose methods that are polynomial-time in the
number of players.

Given that LIGs belong to the class of 2-action
graphical games~\citep{kearns01} with \emph{parametric} payoff
functions, we first needed to deal with the relative dearth of work on
the broader problem of learning \emph{general graphical games} from
purely behavioral data. Hence, in addressing this problem, while inspired by the computational
approach of~\citet{Irfan13}, the learning problem formulation we propose
is in principle applicable to arbitrary games (although, again, the
emphasis is on the PSNE of such games). In particular, we introduce a
simple statistical generative mixture model, built ``on top of'' the
game-theoretic model,
with the only objective being to capture noise in the data.
Despite the simplicity of the generative
model, we are able to learn games from U.S. congressional voting
records, which we use as a source of real-world behavioral data, that,
as we will illustrate,
seem to capture interesting, non-trivial aspects of the U.S. congress.
While such models learned from real-world data are impossible to
validate, we argue that there exists a considerable amount of anecdotal evidence
for such aspects as captured by the models we
learned. Figure~\ref{Congress110} provides a brief
illustration. (Should there be further need for
clarification as to the why we present this figure,
please see Footnote~\ref{foot:figCong110}.)

As a final remark, given that LIGs constitute a non-trivial sub-class of
parametric graphical games, we view our work as a step in the direction of addressing the broader
problem of learning
general graphical games with a large number of players from strictly behavioral data.
We also hope our work helps to continue to bring and increase
attention from the machine-learning community to the problem of
inferring games from behavioral data (in which we attempt to \emph{learn a game} that
would ``rationalize'' players' observed behavior).\footnote{This is a type of problem arising from
game theory and economics that is different from the
problem of learning \emph{in} games (in which the focus is the study
of how individual players \emph{learn to play} a game by a sequence of
repeated interactions), a more matured and perhaps better known
problem within machine learning (see, e.g., \citealt{fudenberg99}).}

\begin{figure}
\begin{center}
\includegraphics[width=1\linewidth,trim=0 30 0 0,clip]{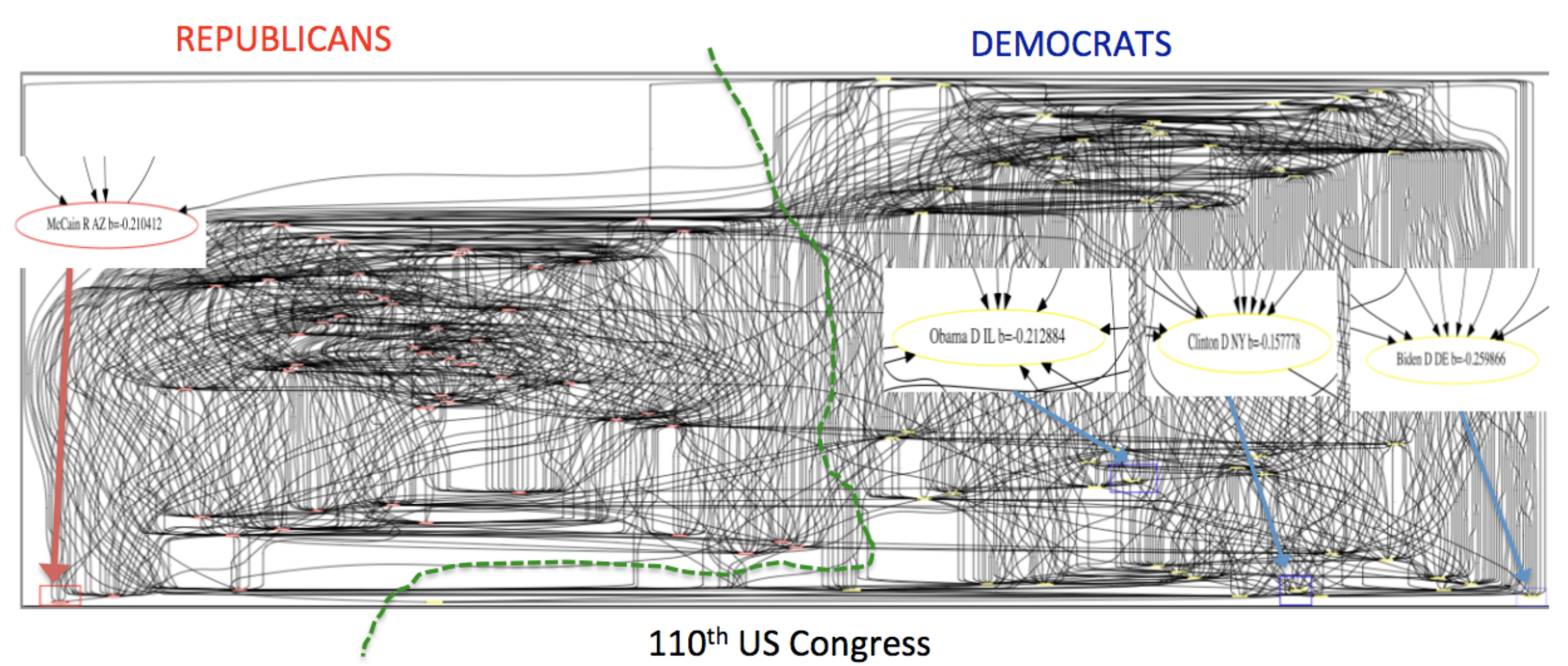}
\end{center}
\label{Congress110}
\vspace{-0.2in}
\captionx{
{\bf 110th US Congress's LIG (January 3, 2007-09):}
We provide an illustration of the application of our approach to real congressional voting data.
\citet{Irfan13} use such games to address a variety of computational
problems, including the identification of \emph{most influential}
senators. (We refer the reader to their paper for further details.)
We show the graph connectivity of a game learned by independent
$\ell_1$-regularized logistic regression (see
Section~\ref{SubSecCLM}). 
The reader should focus on the overall characteristics of the graph
and not the details of the connectivity or the actual ``influence'' weights between senators.
We highlight some particularly interesting characteristics consistent with anecdotal evidence.
First, senators are more likely to be influenced by members of the same party than by members of the opposite party (the dashed green line denotes the separation between the parties).
Republicans were ``more strongly united'' (tighter connectivity) than Democrats at the time.
Second, the current US Vice President Biden (Dem./Delaware) and McCain (Rep./Arizona) are displayed at the ``extreme of each party'' (Biden at the bottom-right corner, McCain at the bottom-left) eliciting their opposite ideologies.
Third, note that Biden, McCain, the current US President Obama (Dem./Illinois) and US Secretary of State Hillary Clinton (Dem./New York) have very few outgoing arcs; e.g., Obama only directly influences Feingold (Dem./Wisconsin), a prominent senior member with strongly liberal stands.
One may wonder why do such prominent senators seem to have so little direct influence on others?
A possible explanation is that US President Bush was about to complete his second term (the maximum allowed).
Both parties had \emph{very long} presidential primaries.
All those senators contended for the presidential candidacy within their parties.
Hence, one may posit that those senators were focusing on running their campaigns and that their influence in the \emph{day-to-day} business of congress was channeled through other prominent senior members of their parties.
}
\vspace{-0.1in}
\end{figure}

\subsection{A Framework for Learning Games: Desiderata} \label{sec:desiderate}

The following list summarizes the discussion above and guides our
choices in our pursuit of a machine-learning framework
for learning game-theoretic graphical models from strictly behavioral
data.
\begin{itemize}
\setlength{\itemsep}{0pt}\setlength{\parskip}{0pt} 
\item  The learning algorithm 
\begin{itemize}
\setlength{\itemsep}{0pt}\setlength{\parskip}{0pt} 
\item must output an LIG (which is a special
  type of graphical game); and
\item should be practical and tractably deal
  with a \emph{large} number of players (typically
  in the hundreds, and certainly at least $4$).
\end{itemize}
\item The learned model objective is the ``bottom line'' in the sense
  that the basis for its
  evaluation is the prediction of \emph{end-state} (or
  steady-state) joint decision-making behavior, and \emph{not} the temporal
  behavioral \emph{dynamics} that might have lead to end-state or the stable steady-state
  joint behavior.\footnote{Note that we are in no way precluding
    dynamic models as a way to end-state prediction. But 
  there is no inherent need to make any \emph{explicit}
  attempt or effort
  to model or predict the temporal
  behavioral \emph{dynamics} that might have lead to end-state or the stable steady-state
  joint behavior, including pre-play ``cheap talk,'' which are often overly
  complex processes. (See Appendix~\ref{App:EndState} for
    further discussion.)}
\item The learning framework 
\begin{itemize}
\setlength{\itemsep}{0pt}\setlength{\parskip}{0pt} 
\item would \emph{only} have available strictly behavioral data
  on actual decisions/actions taken. It cannot require or use any kind of payoff-related information.
\item should be agnostic as to the type or nature of the
  decision-maker and does not assume each player is a \emph{single
    human}. Players can be
institutions or governments, or associated with the decision-making
process of a group of individuals representing, e.g., a company (or
sub-units, office sites within a company, etc.), a
nation state (like in the UN, NATO, etc.), or a voting district. In
other words, the
recorded behavioral actions of each player may really be a representative of larger
entities or groups of individuals, not necessarily a single human.
\item must provide computationally efficient learning algorithm with provable guarantees: worst-case polynomial running time in the number of players.
\item should be ``data efficient'' and
  provide provable guarantees on sample complexity (given  in terms of ``generalization'' bounds).
\end{itemize}
\end{itemize}

\subsection{Technical Contributions}

While our probabilistic model is inspired by the concept of equilibrium from game theory, our technical contributions are not in the field of game theory nor computational game theory.
Our technical contributions and the tools that we use are the ones in classical machine learning.

Our technical contributions include a novel generative model of behavioral data in Section \ref{SecPrelims} for general games.
Motivated by the LIGs and the computational game-theoretic framework
put forward by~\citet{Irfan13}, we formally define ``identifiability'' and ``triviality'' within the context of
non-cooperative graphical
games based on PSNE as the solution concept for stable outcomes in
large
strategic systems.
We provide conditions that ensure identifiability among non-trivial games.
We then present the maximum-likelihood estimation (MLE) problem for general (non-trivial identifiable) games.
In Section \ref{SecGeneralization}, we show a generalization bound for
the MLE problem as well as an upper bound of the
functional/strategic complexity (i.e., analogous to the``VC-dimension'' in
supervised learning) of LIGs.
In Section~\ref{SecAlgorithm}, we provide technical evidence
justifying the \emph{approximation} of the original problem by
maximizing the number of \emph{observed} equilibria in the data as suitable for a hypothesis-space of games with small \emph{true} number of equilibria.
We then present our convex loss minimization approach and a baseline
sigmoidal approximation for LIGs.
For completeness, we also present exhaustive search methods for both
general games as well as LIGs.
In Section \ref{SecSmallTPE}, we formally define the concept of \emph{absolute-indifference} of players and show that our convex loss minimization approach produces games in which all players are \emph{non-absolutely-indifferent}.
We provide a bound which shows that LIGs have small \emph{true} number of equilibria with high probability.

\section{Related Work} \label{SecRelatedWork}

We provide a brief summary overview of previous work on learning games
here, and delay 
discussion of the work presented below
until after we formally present our model; this will provide better
context and make ``comparing and contrasting'' easier for those
interested, without affecting expert readers who may want to get to
the technical aspects of the paper without much delay.

Table~\ref{Approaches} constitutes our best attempt at a simple
visualization to fairly present the differences and similarities of
previous approaches to modeling behavioral data within the
computational game-theory community in AI.

{
\renewcommand{\a}{$^{\rm a}$}
\renewcommand{\b}{$^{\rm b}$}
\renewcommand{\c}{$^{\rm c}$}
\renewcommand{\d}{$^{\rm d}$}
\newcommand{\e}{$^{\rm e}$}
\newcommand{\f}{$^{\rm f}$}
\newcommand{\g}{$^{\rm g}$}
\begin{table}
\begin{center}
\begin{footnotesize}
\begin{tabular}{@{}l@{\hspace{0.05in}}c@{\hspace{0.05in}}c@{\hspace{0.05in}}c@{\hspace{0.05in}}c@{\hspace{0.05in}}c@{}c@{}c@{}r@{}}
  \hline
  \textbf{Reference} & \textbf{Class} & \textbf{Needs} & \textbf{Learns} & \textbf{Learns} & \textbf{Guarant.} & \textbf{Equil.} & \textbf{Dyn.} & \textbf{Num.} \\
  \textbf{} & \textbf{} & \textbf{Payoff} & \textbf{Param.} & \textbf{Struct.} & \textbf{} & \textbf{Concept} & \textbf{} & \textbf{Agents} \\
  \hline
  \citet{Wright10}      & NF  & Y   & N\a & -   & N   & QRE   & N   & 2      \\
  \citet{Wright12}      & NF  & Y   & N\a & -   & N   & QRE   & N   & 2      \\
  \citet{Gao10}         & NF  & Y   & Y   & -    & N   & QRE   & N   & 2      \\
  \citet{Vorobeychik07} & NF  & Y   & Y   & -   & N   & MSNE    & N   & 2-5    \\
  \citet{Ficici08}      & NF  & Y   & Y   & -   & N   & MSNE    & N   & 10-200 \\
  \citet{Duong08}       & NGT   & Y   & N\a & N   & N   & -     & N   & 4,10     \\
  \citet{Duong10}       & NGT\b & Y   & N\c & N   & N   & -     & Y\d & 10     \\
  \citet{Duong12}       & NGT\b & Y   & N\c & Y\e & N   & -     & Y\d & 36     \\
  \citet{Duong09}       & GG   & Y   & Y   & Y\f & N   & PSNE   & N   & 2-13   \\
  \hline
  \citet{KearnsW08}     & NGT & N   & -   & -   & Y   & -     & Y   & 100    \\
  \citet{Ziebart10}     & NF  & N   & Y   & -   & N   & CE    & N   & 2-3    \\
  \citet{Waugh11}       & NF  & N   & Y   & -   & Y   & CE    & Y   & 7      \\
  \hl{Our approach}     & \hl{GG} & \hl N & \hl Y & \hl{Y\g} & \hl Y & \hl{PSNE} & \hl N & \hl{100\g}\\
  \hline
\end{tabular}
\end{footnotesize}
\end{center}
\vspace{-0.15in}
\captionx{{\bf Summary of approaches for learning models of behavior.} See
main text for a discussion.
For each method we show its model class (GG: graphical games, NF:
normal-form non-graphical games, NGT:
non-game-theoretic model); whether it needs
observed payoffs, learns \emph{utility} parameters, learns
\emph{graphical} structure or provides guarantees(e.g.,
generalization, sample complexity or PAC learnability); its equilibria
concept (PSNE: pure strategy or MSNE: mixed strategy Nash equilibria,
CE: correlated equilibria, QRE: quantal response equilibria), whether
it is dynamic (i.e., behavior predicted from past behavior); and the number of agents in the experimental validation.
Note that there are relatively fewer models that do not assume observed payoff; among them, our method is the only one that learns the structure of graphical games, furthermore, we provide guarantees and a polynomial-time algorithm.
\a Learns only the ``rationality parameter''.
\b A graphical game could in principle be extracted, after removing the temporal/dynamic part.
\c It learns parameters for the ``potential functions.''
\d If the dynamic part is kept, it is not a graphical game.
\e It performs greedy search by constraining the maximum degree.
\f It performs branch and bound.
\g It has polynomial time-complexity in the number of agents, thus it can scale to thousands.
}
\label{Approaches}
\end{table}
}

The research interest of previous work
  varies in what they intend to capture in terms of different aspects of behavior (e.g.,
  dynamics, probabilistic vs. strategic) or simply different settings/domains (i.e., modeling ``real human behavior,'' knowledge
  of achieved payoff or utility, etc.).

With the exception of~\citet{Ziebart10,Waugh11,KearnsW08}, previous
methods assume that the actions as well as corresponding payoffs (or
noisy samples from the true payoff function) are \emph{observed} in the data.
Our setting largely differs from~\citet{Ziebart10,KearnsW08} because
of their focus on system dynamics, in which future behavior is
predicted from a sequence of past behavior. \citet{KearnsW08} proposed
a learning-theory framework to model \emph{collective} behavior based on \emph{stochastic} models.

Our problem is clearly different from methods in \emph{quantal response
models}~\citep{mckelvey95,Wright10,Wright12}
and \emph{graphical multiagent models
  (GMMs)}~\citep{Duong08,Duong10} that assume \emph{known structure} and \emph{observed payoffs}.
\citet{Duong12} learns the structure of games that are \emph{not graphical}, i.e., the payoff depends on all other players.
Their approach also assumes \emph{observed} payoff and consider a \emph{dynamic} consensus scenario, where agents on a network attempt to reach a unanimous vote.
In analogy to voting, we do not assume the availability of the dynamics (i.e., the previous actions) that led to the \emph{final} vote.
They also use fixed information on the conditioning sets of neighbors
during their
search for graph structure. We also note that the work of~\citet{Vorobeychik07,Gao10,Ziebart10} present experimental validation mostly for 2 players only, 7 players in~\citet{Waugh11} and up to 13 players in~\citet{Duong09}.

In several cases in previous work, researchers define probabilistic
models using knowledge of the payoff functions explicitly (i.e., a
\emph{Gibbs distribution} with potentials that are functions of the players
payoffs, regrets, etc.) to model joint behavior (i.e., joint
pure-strategies); see, e.g.,~\citet{Duong08,Duong10,Duong12}, and to
some degree also~\citet{Wright10,Wright12}.  It should be clear to the reader that this is not
the same as our generative model, which is defined \emph{directly} on
the PSNE (or stable outcomes) of the game, which the players' payoffs
determine only \emph{indirectly}.

In contrast, in this paper, we assume that the joint actions are the
\emph{only} observable information and that both the game graph
structure and payoff functions are \emph{unknown, unobserved and unavailable}.
We present the first techniques for
learning the structure and parameters of a non-trivial class of large-population graphical games from joint actions only.
Furthermore, we present experimental validation in games of up to $100$ players.
Our convex loss minimization approach could potentially be applied to
larger problems since it has \emph{polynomial time} complexity in the number of players.

\subsection{On Learning Probabilistic Graphical Models} \label{sec:probmod}

There has been a significant amount of work on learning the structure of \emph{probabilistic} graphical models from data.
We mention only a few references that follow a maximum likelihood approach for Markov random fields~\citep{Lee06}, bounded tree-width distributions~\citep{Chow68,Srebro01}, Ising models~\citep{Wainwright06,Banerjee08,Hofling09}, Gaussian graphical models~\citep{Banerjee06}, Bayesian networks~\citep{Guo06,Schmidt07} and directed cyclic graphs~\citep{Schmidt09}.

Our approach learns the structure and parameters of games by maximum likelihood estimation on a related probabilistic model.
Our probabilistic model does not fit into any of the types described above.
Although a (directed) graphical game has a directed cyclic graph, there is a semantic difference with respect to graphical models.
Structure in a graphical model implies a factorization of the probabilistic model.
In a graphical game, the graph structure implies \emph{strategic} dependence between players, and has no immediate probabilistic implication.
Furthermore, our general model differs from~\citet{Schmidt09} since our generative model does not decompose as a multiplication of potential functions.

Finally, it is very important to note that our specific aim is to model \emph{behavioral data} that is \emph{strategic}
in nature.
Hence, our modeling and learning approach deviates from those
for \emph{probabilistic} graphical models which are of course better suited
for other types of data, mostly \emph{probabilistic} in nature (i.e.,
resulting from a \emph{fixed} underlying probability distribution). As
a consequence, it is also very important to keep
in mind that our work is not in competition with the work in
probabilistic graphical models, and is not meant to replace it
(except in the context of data sets collected from complex strategic
behavior just mentioned). Each approach has its own
aim, merits and pitfalls
in terms of the nature of data sets that each seeks to model.
We return to this point in Section~\ref{SecResults} (Experimental Results).

\subsection{On Linear Threshold Models and Econometrics} \label{SubSecThreshold}

\citet{Irfan13} introduced LIGs in the AI community, showed that such games are useful, and addressed a variety of computational problems, including the identification of \emph{most influential} senators.
The class of LIGs is related to the well-known
\emph{linear threshold model (LTM)} in
sociology~\citep{granovetter78}, recently very popular within the
social network and theoretical computer science
community~\citep{kleinberg07}.\footnote{\citet{Lopez-Pintado01112008}
  also provide an excellent summary of the various models in this area of
  mathematical social science.}
\citet{Irfan13} discusses linear threshold models in depth; we briefly discuss them here for self-containment.
LTMs are usually studied as the basis for some kind of diffusion process.
A typical problem is the identification of most influential individuals in a social network.
An LTM is not in itself a game-theoretic model and, in fact, Granovetter himself argues against this view in the context of the setting and the type of questions in which he was most interested~\citep{granovetter78}.
Our reading of the relevant literature suggests that subsequent work on LTMs has not taken a
strictly game-theoretic view either. The problem of
learning mathematical models of influence from behavioral data has just started
to receive attention. There has been a number of articles in the last couple of
years addressing the problem of learning the parameters of a variety of
\emph{diffusion} models of
influence~\citep{Saitoetal08,Saitoetal09,Saitoetal10,Goyaletal10,GomezRodriguezetal10,caoetal11}.\footnote{Often
learning consists of estimating the threshold parameter from
data given as temporal sequences from``traces'' or ``action logs.''
Sometimes the ``influence weights'' are estimated assuming a given
graph, and almost always the weights are
assumed \emph{positive} and estimated
as ``probabilities of influence.''
For example, \citet{Saitoetal10} considers a dynamic (continuous time) LTM that has
only positive influence weights and a randomly generated threshold
value. \citet{caoetal11} uses active learning to estimate the
threshold values of an LTM leading to a maximum spread of influence.}

Our model is also related to a particular model of \emph{discrete
  choice with social interactions} in econometrics (see,
e.g.~\citealt{brock_and_durlauf01}). The main difference is that we take
a strictly non-cooperative game-theoretic approach within the
classical ``static''/one-shot game framework and do not use a
\emph{random utility model}. We follow the approach of \citet{Irfan13} who takes a strictly non-cooperative game-theoretic approach within the classical ``static''/one-shot game framework, and thus we do not use a \emph{random utility model}.
In addition, we do not make the assumption of \emph{rational
  expectations}, which in the context of models of discrete choice
with social interactions essentially implies the assumption that all
players use \emph{exactly the same mixed strategy}.\footnote{A formal definition of ``rational expectations'' is beyond the scope of this
paper. We refer the reader to the early part of the article by~\citet{brock_and_durlauf01} where they
explain why assuming rational expectations leads to the conclusion
that all players use exactly the same mixed strategy. That is the
relevant part of that work to ours.}

\section{Background: Game Theory and Linear Influence Games} \label{SecBackground}

In classical game-theory (see, e.g.~\citealt{fudenbergandtirole91} for a textbook introduction), a \emph{normal-form game} is defined by a set of \emph{players} $V$ (e.g., we can let $V = \{1,\dots,n\}$ if there are $n$ players), and for each player $i$, a set of \emph{actions}, or \emph{pure-strategies} $A_i$, and a payoff function $u_i : \times_{j \in V} A_j \to \R$ mapping the joint actions of all the players, given by the Cartesian product $\A \equiv \times_{j \in V} A_j$, to a real number.
In non-cooperative game theory we assume players are greedy, rational and act independently, by which we mean that each player $i$ always want to maximize their own utility, subject to the actions selected by others, irrespective of how the optimal action chosen help or hurt others.

A core solution concept in non-cooperative game theory is that of an \emph{Nash equilibrium}.
A joint action $\x^* \in \A$ is a \emph{pure-strategy Nash equilibrium
(PSNE)} of a non-cooperative game if, for each player $i$, $x^*_i \in \argmax_{x_i \in A_i} u_i(x_i,\x^*_\minus{i})$;
that is, $\x^*$ constitutes a \emph{mutual best-response}, no player $i$ has any incentive to unilaterally deviate from the prescribed action $x^*_i$, given the joint action of the other players $\x^*_\minus{i} \in \times_{j \in V-\{i\}} A_j$ in the equilibrium.
In what follows, we denote a game by $\G$, and the set of all
\emph{pure-strategy Nash equilibria} of $\G$ by\footnote{Because
  this paper concerns mostly PSNE, we denote
  the set of PSNE of game $\G$ as $\NE(\G)$ to simplify notation.}
\begin{equation*}
\NE(\G) \equiv \{ \x^* \mid (\forall i \in V){\rm\ }x^*_i \in \argmax_{x_i \in A_i}
u_i(x_i,\x^*_\minus{i}) \} \; .
\end{equation*}

A \emph{(directed) graphical game} is a game-theoretic graphical model~\citep{kearns01}.
It provides a succinct representation of normal-form games.
In a graphical game, we have a (directed) graph $G = (V,E)$ in which each node in $V$ corresponds to a player in the game.
The interpretation of the edges/arcs $E$ of $G$ is that the payoff function of player $i$ is only a function of the set of parents/neighbors $\N_i \equiv \{ j \mid (i,j) \in E \}$ in $G$ (i.e., the set of players corresponding to nodes that point to the node corresponding to player $i$ in the graph).
In the context of a graphical game, we refer to the $u_i$'s as the \emph{local payoff functions/matrices}.

\emph{Linear influence games (LIGs)}~\citep{Irfan13} are a sub-class of
$2$-action graphical games with \emph{parametric} payoff functions.
For LIGs, we assume that we are given a matrix of influence weights $\W\in \R^{n \times n}$, with $\mathbf{diag}(\W)=\zero$, and a threshold vector $\b\in \R^n$.
For each player $i$, we define the influence function $f_i(\x_\minus{i}) \equiv \sum_{j \in \N_i} w_{ij} x_j - b_i = \t{\w{i}}\x_\minus{i}-b_i$ and the payoff function $u_i(\x) \equiv x_i f_i(\x_\minus{i})$.
We further assume binary actions: $A_i \equiv \{-1,+1\}$ for all $i$.
The \emph{best response} $x_i^*$ of player $i$ to the joint action $\x_\minus{i}$ of the other players is defined as
\begin{equation*}
\left\{\begin{array}{@{}l@{}}
\t{\w{i}}\x_\minus{i} > b_i \Rightarrow x^*_i = +1, \\
\t{\w{i}}\x_\minus{i} < b_i \Rightarrow x^*_i = -1 \text{ and } \\
\t{\w{i}}\x_\minus{i} = b_i \Rightarrow x^*_i \in \{-1,+1\}
\end{array}\right\}
\Leftrightarrow x^*_i(\t{\w{i}}\x_\minus{i}-b_i) \geq 0 \; .
\end{equation*}
\emph{Intuitively}, for any other player $j$, we \emph{can think of}
$w_{ij} \in \R$ as a \emph{weight} parameter quantifying the
``influence factor'' that $j$ has on $i$, and we \emph{can think of}
$b_i \in \R$ as a \emph{threshold} parameter quantifying the level of
``tolerance'' that player $i$ has for playing $-1$.\footnote{As we
  \emph{formally/mathematically define} here, LIGs
are $2$-action graphical games with \emph{linear-quadratic} payoff
functions. Given our main interest in this
paper on the PSNE solution concept, for the most part, we simply view \emph{LIGs as compact
representations of the PSNE of graphical games that the
algorithms of~\citet{Irfan13} use for CSI.} (This is in contrast to a perhaps more natural,
``intuitive'' but still \emph{informal} description/interpretation one may provide
for instructive/pedagogical purposes based on ``direct influences,'' as we do here.) This view of LIGs is analogous to
the modern, predominant view of Bayesian networks as compact
representations of joint probability distributions that are also very useful for
modeling uncertainty in complex systems and practical for
probabilistic inference~\citep{Koller09}. (And also analogous is the
``intuitive'' descriptions/interpretations of BN structures, used for
instructive/pedagogical purposes, based on ``causal'' interactions
between the random variables~\citealp{Koller09}.)}

As discussed in \citet{Irfan13}, LIGs are also a sub-class of polymatrix games \citep{Janovskaja_1968_MR_by_Isbell}.
Furthermore, in the special case of $\b=\zero$ and symmetric $\W$, a
LIG becomes a \emph{party-affiliation
  game}~\citep{fabrikantetal04}.

In this paper, the use of the verb ``influence'' strictly refers to
influences defined by the model. 

Figure~\ref{Congress110} provides a preview illustration of the
application of our approach to congressional
voting.\footnote{\label{foot:figCong110}We present this game graph because
  many people express interest in ``seeing'' the type of games we learn on
  this particular data set. The reader should please understand that by
  presenting this graph we are \emph{definitely not} implying or arguing that we can identify
the ground-truth graph of ``direct influences.'' (We say this even in the very
unlikely event that the
``ground-truth model'' be an LIG that faithfully capture the ``true direct
influences'' in this U.S. Congress, something arguably \emph{no model}
could ever do.)
As we show later in Section~\ref{sec:ident}, LIGs are \emph{not}
identifiable with respect to their \emph{local} compact parametric representation
encoding the game graph through their weights and biases, but only
with respect to their PSNE, which are joint actions capturing a \emph{global} property of a
game that
we really care about for CSI.
Certainly, we could \emph{never} validate the model parameters of an
LIG at the
\emph{local}, microscopic
level of ``direct influences'' using only the type of
\emph{observational} data we used to learn the model depicted by the graph in
the figure. For that, we would need help from domain experts to design
\emph{controlled} experiments that would yield the right type of data
for proper/rigorous scientific validation.  
}

\section{Our Proposed Framework for Learning LIGs} \label{SecPrelims}

Our \emph{goal} is to learn the structure and parameters of an LIG from
observed joint actions only (i.e., without any payoff
data/information).\footnote{
In principle, the \emph{learning framework}
  itself is technically immediately/easily applicable to \emph{any} class of
  simultaneous/one-shot games. Generalizing the algorithms
  and other theoretical results (e.g., on generalization error)
  while maintaining the tractability in sample
  complexity and computation may
  require significant effort.} Yet, for simplicity, most of the
presentation in this section is actually in terms of general 2-action games. While we make
sporadic references to LIGs throughout the section, it is not until we
reach the end of the section
that we present and discuss the particular instantiation of our
proposed framework with LIGs.

Our main \emph{performance measure} will be average log-likelihood (although later
we will be considering misclassification-type error measures in the
context of simultaneous-classification, as a
result of an approximation of the average log-likelihood). Our \emph{emphasis on a PSNE-based statistical model for the behavioral data} results from the approach to causal strategic
inference taken by~\citet{Irfan13}, which is strongly founded on
PSNE.\footnote{The possibility that PSNE may not exist in some LIGs does not
present a significant problem in our case because \emph{we are
  learning the game, and can require that the LIG output has at least
  one PSNE.} Indeed, in our approach, games with no PSNE achieve the
lowest possible likelihood within our generative model of the data;
said differently, games with PSNE 
have 
higher 
likelihoods than those that do not have any PSNE.}

Note that our problem is \emph{unsupervised}, i.e., we do not know a priori which joint actions are PSNE and which ones are not.
If our only goal were to find a game $\G$ in which all the given
observed data is an equilibrium, then any ``dummy'' game, such as the
``dummy'' LIG 
$\G=(\W,\b),\W=\zero,\b=\zero$, would be an optimal solution because
$|\NE(\G)|=2^n$.\footnote{\citet{ngandrussell00} made a similar observation in
  the context of single-agent \emph{inverse reinforcement learning (IRL)}.}
In this section, we present a probabilistic formulation that allows finding games that maximize the \emph{empirical proportion of equilibria} in the data while keeping the \emph{true proportion of equilibria} as low as possible.
Furthermore, we show that \emph{trivial} games such as LIGs with $\W=\zero,\b=\zero$, obtain the lowest log-likelihood.

\subsection{Our Proposed Generative Model of Behavioral Data}

We propose the following simple generative (mixture) model for behavioral data based
strictly in the context of ``simultaneous''/one-shot play in
non-cooperative game theory, again motivated by~\citet{Irfan13}'s
PSNE-based approach to \emph{causal strategic inference (CSI)}.\footnote{Model ``simplicity'' and ``abstractions'' are not necessarily a bad
  thing in practice. More ``realism'' often leads to more
  ``complexity'' in terms of model representation and
  computation; and to potentially
  poorer generalization performance as well~\citep{Kearns94}. We
  believe that even if the data could be the result of
  complex cognitive, behavioral or neuronal processes underlying human decision making and
  social interactions, the practical guiding principle of model
  selection in ML, which governs the fundamental tradeoff between model
  complexity and generalization performance, still applies.}
Let $\G$ be a game.
With some probability $0<q<1$, a joint action $\x$ is chosen uniformly at random from $\NE(\G)$;
otherwise, $\x$ is chosen uniformly at random from its complement set $\{-1,+1\}^n - \NE(\G)$.
Hence, the generative model is a mixture model with mixture parameter
$q$ corresponding to the probability that a stable outcome (i.e., a
PSNE) of the game is observed, uniform over PSNE.
Formally, the probability mass function (PMF) over joint-behaviors $\{-1,+1\}^n$ parameterized by $(\G,q)$ is
\begin{equation} \label{Prob}
p_{(\G,q)}(\x) = q \, \frac{\iverson{\x \in \NE(\G)}}{|\NE(\G)|} + (1-q) \, \frac{\iverson{\x \notin \NE(\G)}}{2^n - |\NE(\G)|} \; ,
\end{equation}
\noindent where we can think of $q$ as the ``signal'' level, and thus $1-q$ as the ``noise'' level in the data set.

\begin{remark} \label{RemUnif}
Note that in order for Eq.~\eqref {Prob} to be a valid PMF for any $\G$, we need to enforce the following conditions $|\NE(\G)|=0 \Rightarrow q=0$ and $|\NE(\G)|=2^n \Rightarrow q=1$.
Furthermore, note that in both cases ($|\NE(\G)| \in \{0,2^n\}$) the PMF becomes a uniform distribution.
We also enforce the following condition:\footnote{We can easily remove this condition
  at the expense of complicating the theoretical analysis on the
  generalization bounds because of having to
  deal with those extreme cases.} if $0<|\NE(\G)|<2^n$ then
$q \not\in\{0,1\}$.
\end{remark}

\subsection{On PSNE-Equivalence and PSNE-Identifiability of Games}
\label{sec:ident}

For any valid value of mixture parameter $q$, the PSNE of a game $\G$ completely
determines our generative model $p_{(\G,q)}$. Thus, given any such
mixture parameter, two games with the same set
of PSNE will induce the same PMF over the space
of joint actions.\footnote{It is not hard to come up with examples of
multiple games that have the same PSNE set. In fact, later in this section, we show three instances of
LIGs with very different weight-matrix parameter that have this
property. Note that this is not a roadblock to our objectives of
learning LIGs because
our main interest is the PSNE of the game, not the
individual parameters that define it. We note that this situation
  is hardly exclusive to game-theoretic models: an
  analogous issue occurs in probabilistic graphical models (e.g.,
  Bayesian networks).}

\begin{definition} \label{DefEquiv}
We say that two games $\G_1$ and $\G_2$ are \emph{PSNE-equivalent} if and only if their PSNE sets are identical, i.e., $\G_1 \equiv_\NE \G_2 \Leftrightarrow \NE(\G_1)=\NE(\G_2)$.
\end{definition}
We often drop the ``PSNE-'' qualifier when clear from context.
\begin{definition} \label{DefIdentif}
We say a set $\PS^*$ of valid parameters $(\G,q)$ for the generative model is \emph{PSNE-identifiable with respect
to the PMF $p_{(\G,q)}$ defined in Eq.~\eqref {Prob}}, if and only if,
for every pair $(\G_1,q_1),(\G_2,q_2) \in \PS^*$, if $p_{(\G_1,q_1)}(\x)
= p_{(\G_2,q_2)}(\x)$ for all $\x \in \{-1,+1\}^n$ then $\G_1
\equiv_\NE \G_2$ and $q_1=q_2$. We say a game $\G$ is \emph{PSNE-identifiable
with respect to $\PS^*$ and the $p_{(\G,q)}$}, if and only if, there
exists a $q$ such that $(\G,q) \in \PS^*$.
\end{definition}

\begin{definition}
\label{DefTPropTriv}
We define the \emph{true proportion of equilibria} in the game $\G$
relative to all possible joint actions as
\begin{equation} \label{PiDef}
\pi(\G) \equiv |\NE(\G)|/2^n \; .
\end{equation}
We also say that a game $\G$ is \emph{trivial} if and only if $|\NE(\G)|\in \{0,2^n\}$ (or equivalently $\pi(\G)\in \{0,1\}$), and \emph{non-trivial} if and only if $0<|\NE(\G)|<2^n$ (or equivalently $0<\pi(\G)<1$).
\end{definition}

The following propositions establish that the condition $q>\pi(\G)$
ensures that the probability of an equilibrium is strictly greater
than a non-equilibrium. The condition also guarantees that non-trivial
games are identifiable.
\begin{proposition} \label{PrpEqGTNoEq}
Given a non-trivial game $\G$, the mixture parameter $q>\pi(\G)$ if and only if $p_{(\G,q)}(\x_1) > p_{(\G,q)}(\x_2)$ for any $\x_1 \in \NE(\G)$ and $\x_2 \notin \NE(\G)$.
\end{proposition}
\begin{proof}
Note that $p_{(\G,q)}(\x_1)=q/|\NE(\G)| > p_{(\G,q)}(\x_2)=(1-q)/(2^n - |\NE(\G)|) \Leftrightarrow q>|\NE(\G)|/2^n$ and given Eq.~\eqref {PiDef}, we prove our claim.
\qedhere
\end{proof}
\begin{proposition} \label{PrpIdentifProb}
Let $(\G_1,q_1)$ and $(\G_2,q_2)$ be two valid generative-model parameter tuples.
\begin{itemize}
\item[(a)] If $\G_1 \equiv_{\NE} \G_2$ and $q_1=q_2$
then $(\forall \x){\rm\ }p_{(\G_1,q_1)}(\x) = p_{(\G_2,q_2)}(\x)$,
\item[(b)] Let $\G_1$ and $\G_2$ be also two non-trivial games such
  that $q_1 > \pi(\G_1)$ and $q_2 > \pi(\G_2)$. If $(\forall \x){\rm\
  }p_{(\G_1,q_1)}(\x) = p_{(\G_2,q_2)}(\x)$, then  $\G_1
\equiv_\NE \G_2$ and $q_1=q_2$.
\end{itemize}
\end{proposition}
\begin{proof}
Let $\NE_k \equiv \NE(\G_k)$. First, we prove part (a).
By Definition \ref{DefEquiv}, $\G_1 \equiv_\NE \G_2
\Rightarrow \NE_1=\NE_2$. Note that $p_{(\G_k,q_k)}(\x)$ in
Eq.~\eqref {Prob} depends only on characteristic functions
$\iverson{\x\in \NE_k}$. Therefore, $(\forall \x){\rm\
}p_{(\G_1,q_1)}(\x)=p_{(\G_2,q_2)}(\x)$.

Second, we prove 
part (b) by contradiction. Assume, wlog, $(\exists \x){\rm\ }\x\in \NE_1
\wedge \x\notin \NE_2$. Then $p_{(\G_1,q_1)}(\x)=p_{(\G_2,q_2)}(\x)$
implies by Eq.~\eqref {Prob} that $q_1/|\NE_1|=(1-q_2)/(2^n - |\NE_2|)
\Rightarrow q_1/\pi(\G_1)=(1-q_2)/(1 - \pi(\G_2))$ by
Eq.~\eqref {PiDef}. By assumption, we have $q_1 > \pi(\G_1)$, which implies that
$(1-q_2)/(1 - \pi(\G_2)) > 1 \Rightarrow
q_2 < \pi(\G_2)$, a contradiction. Hence, we have $\G_1
\equiv_\NE \G_2$. Assume, $q_1 \neq q_2$. Then we have $p_{(\G_1,q_1)}(\x)=p_{(\G_2,q_2)}(\x)$
implies by Eq.~\eqref {Prob} and $\G_1
\equiv_\NE \G_2$ (and after some algebraic manipulations) that $(q_1 -
q_2) \left( \frac{\iverson{\x \in \NE(\G_1)}}{|\NE(\G_1)|} -
  \frac{\iverson{\x \notin \NE(\G_1)}}{2^n - |\NE(\G_1)|}\right) =
0 \Rightarrow \frac{\iverson{\x \in \NE(\G_1)}}{|\NE(\G_1)|} =
  \frac{\iverson{\x \notin \NE(\G_1)}}{2^n - |\NE(\G_1)|}$, a
  contradiction.
\qedhere
\end{proof}

The last proposition, along with our definitions of ``trivial'' (as given in Definition~\ref{DefTPropTriv}) and ``identifiable'' (Definition~\ref{DefIdentif}), allows us to
formally define our \emph{hypothesis space}.
\begin{definition}
\label{DefHS}
Let $\HS$ be a class of games of interest. We call the set $\PS \equiv
\{(\G,q) \, \mid \, \G \in \HS \wedge 0 < \pi(\G) < q < 1\}$ the
\emph{hypothesis space of non-trivial identifiable
  games and mixture parameters}. We also refer to a game $\G \in \HS$ that is also in some
tuple $(\G,q) \in \PS$
for some $q$, as a \emph{non-trivial
identifiable game}.\footnote{Technically, we should call the set $\PS$ ``the
  hypothesis space consisting of tuples of non-trivial
  games from $\HS$ and mixture parameters identifiable up to PSNE, with respect to the probabilistic
  model defined in Eq.~\eqref {Prob}.'' Similarly, we should call such
  game $\G$
  ``a non-trivial
  game from $\HS$ identifiable up to PSNE,
with respect to the probabilistic
  model defined in Eq.~\eqref {Prob}.'' We opted for brevity.}
\end{definition}

\begin{remark}
Recall that a trivial game induces a uniform PMF by Remark~\ref{RemUnif}.
Therefore, a non-trivial game is not \emph{equivalent} to a trivial
game since by Proposition \ref{PrpEqGTNoEq}, non-trivial games do not
induce uniform PMFs.\footnote{In general, Section~\ref{sec:ident} characterizes
  our hypothesis space (non-trivial identifiable games and mixture parameters) via two
  specific conditions.
The first condition, \emph{non-triviality} (explained in Remark~\ref{RemUnif}), is $0 < \pi(\G) < 1$.
The second condition, \emph{identifiability of the PSNE set from its
  related PMF} (discussed in Propositions~\ref{PrpEqGTNoEq}  and \ref{PrpIdentifProb}), is $\pi(G) < q$.
For completeness, in this remark, we clarify that the class of trivial games (uniform PMFs) is different from the class of non-trivial games (non-uniform PMFs).
Thus, in the rest of the paper we focus exclusively on non-trivial
identifiable games; that is, games that produce non-uniform PMFs and
for which the PSNE set is uniquely identified from their PMFs.}
\end{remark}

\subsection{Additional Discussion on Modeling Choices}

We now discuss other equilibrium concepts, such as mixed-strategy Nash equilibria
(MSNE) and
quantal response equilibria (QRE). We also discuss a more
sophisticated noise process
as well as a generalization of our model to non-uniform distributions;
while likely more realistic, the alternative models are mathematically more complex and potentially less tractable computationally.

\subsubsection{On Other Equilibrium Concepts} \label{SubSecConcepts}

There is still quite a bit of debate as to the appropriateness of game-theoretic equilibrium concepts to model individual human behavior in a social context.
Camerer's book on behavioral game theory~\citep{Camerer03} addresses
some of the issues.
Our interpretation of Camerer's position is not that Nash equilibria is universally a bad predictor but that it is not \emph{consistently} the best, for reasons that are still not well understood.
This point is best illustrated in Chapter 3, Figure 3.1 of~\citet{Camerer03}.

\emph{(Logit) quantal response equilibria (QRE)}~\citep{mckelvey95}
has been proposed as an alternative to Nash in the context of
behavioral game theory.
Models based on QRE have been shown superior during \emph{initial
  play} in some experimental settings, but prior work assumes known structure and observed
payoffs, and only the ``precision/rationality
parameter'' is estimated, e.g.~\citet{Wright10,Wright12}. In
  a logit QRE, the precision parameter, typically denoted by $\lambda$, can be mathematically interpreted as the
  inverse-temperature parameter of individual Gibbs distributions over the
  pure-strategy of each player $i$. 

It is customary to compute the MLE for $\lambda$ from available
data. To
the best of our knowledge, all work in QRE assumes exact knowledge of
the game payoffs, and thus, no method has been proposed to
simultaneously estimate the payoff functions
$u_i$ when they are unknown. Indeed, computing MLE for $\lambda$,
\emph{given the payoff functions},
is relatively efficient for normal-form games using standard
techniques, but may be hard for graphical games; on the other hand, MLE estimation of
the payoff functions themselves within a QRE model of behavior seems like a
highly non-trivial optimization problem, and is unclear that it is even
computationally tractable, even in normal-form games. At the very least, standard techniques do
not apply and more sophisticated optimization algorithms or heuristics
would have to be derived. Such extensions are clearly beyond the scope
of this paper.\footnote{Note that despite the apparent similarity in
  mathematical expression between logit QRE and the PSNE of the LIG we
  obtain by using
  individual logistic regression, they are fundamentally different
  because of the complex correlations that QRE conditions impose on
  the parameters $(\W,\b)$ of the payoff functions. It is unclear how
  to adapt techniques for logistic regression similar to the ones we
  used here to efficiently/tractably compute MLE within the logit QRE framework.}

\citet{Wright12} also considers even more mathematically complex variants of behavioral models that
  combine QRE with different models that account for constraints in
  ``cognitive levels'' of reasoning ability/effort, yet the estimation
  and usage of such models still
  assumes knowledge of the payoff functions.

It would be fair to say that most of the human-subject experiments in behavioral game
theory involve only a handful of players~\citep{Camerer03}. The scalability of those
results to games with a large population of players is unclear.

Now, just as an example, we do not necessarily view the Senators final
votes as those of
  a \emph{single} human individual anyway: after all, such a decision  is (or should
  be) obtained with
  consultation with their staff and (one would at least hope) the
  constituency of the state they represent.
Also,
  the final voting decision is
  taken after consultation or meetings between the staff of the
  different senators. We view this underlying process as one of
  ``cheap talk.'' While cheap talk may be an important process to
  study, in this paper, we just concentrate on the end result: the final vote. The
  reason is more than just scientific; as the congressional voting
  setting illustrates, data for such process is sometimes not
  available, or would seem very hard to infer from
  the end-states alone. While our experiments concentrate on
  congressional voting data, because it is publicly available and easy
  to obtain, the same would hold for other settings such as Supreme
  court decisions, voluntary vaccination, UN voting records and
  governmental decisions, to name a few. We speculate that in almost all those cases,
  only the end-result is likely to be recorded and little information
  would be available about the ``cheap talk'' process or ``pre-play''
  period leading to the final decision.

In our work we consider PSNE because of our motivation to provide
LIGs for use within the casual strategic inference framework for
modeling ``influence'' in large-population networks
of~\citet{Irfan13}.
Note that the universality of MSNE does not diminish the importance of
PSNE in game theory.\footnote{Research work on the properties and
  computation of PSNE include
\citet{rosenthal73,gilboa89,Stanford1995238,Rinott2000274,fabrikantetal04,gottlobetal05,Sureka05,DP06,Dunkel06,DunkelSchulz06,dilkinaetal07,Ackermann07,Hasan08,HasanGaliana08,Ryan10,Chapman10,HasanGaliana10}.}
Indeed, a debate still exist within the game theory community as to
the justification for randomization, specially in human contexts.
While concentrating exclusively on PSNE may not make sense in \emph{all} settings,
it does make sense in \emph{some}.\footnote{For example, in
the context of congressional voting, we believe Senators almost always
have full-information about how some, if not all other Senators they
care about would vote. Said differently, we believe
\emph{uncertainty} in a Senator's final vote, \emph{by the time the vote is actually
taken}, is rare, and certainly not the norm. Hence, it is unclear how
much
there is to gain, \emph{in this particular setting}, by
considering possible randomization in the Senators' voting
strategies.}
In addition, were we to introduce mixed-strategies into the inference and learning
framework and model, we would be adding a considerable amount of
complexity in almost all respects, thus requiring a substantive effort
to study on its own.\footnote{For example, note that because in our setting we learn
exclusively from observed joint actions, we could not assume knowledge
of the internal mixed-strategies of players. Perhaps we could generalize our model to allow for mixed-strategies by
defining a process in which a joint mixed strategy $\p$ from the set
of MSNE (or its complement) is drawn
according to some distribution, then a (pure-strategy) realization
$\x$ is drawn from $\p$ that would correspond to the observed
joint actions. One problem we might face with this approach is that
little is known about the structure of MSNE in general multi-player
games. For example, it is not even clear that the set of MSNE is always measurable
in general!}

\subsubsection{On the Noise Process} \label{SubSecNoise}

Here we discuss a more sophisticated noise process as well as a generalization of our model to non-uniform distributions.
The problem with these models is that they lead to a significantly more complex expression for the generative model and thus likelihood functions.
This is in contrast to the simplicity afforded us by the generative
model with a more global noise process defined above. (See
Appendix~\ref{App:GenDist} for further discussion.)

In this paper we considered a ``global'' noise process, modeled using
a parameter $q$ corresponding to the probability that a sample
observation is an equilibrium of the underlying hidden game.
One could easily envision potentially better and more natural/realistic ``local'' noise processes, at the expense of producing a significantly more complex
generative model, and less computationally amenable, than the one considered in this
paper.  For instance, we could use a noise process that is formed of
many independent, individual noise processes, one for each
player. (See Appendix~\ref{App:IndNoise} for further discussion.)

\subsection{Learning Games via MLE}

We now formulate the problem of learning games as one of maximum
likelihood estimation with respect to our PSNE-based
generative model defined in Eq.~\eqref{Prob} and the hypothesis space of
non-trivial identifiable games and mixture parameters (Definition~\ref{DefHS}). We remind the
reader that our problem is unsupervised; that is, we do not know a priori which joint actions are equilibria and which ones are not.
We base our framework on the fact that \emph{games are PSNE-identifiable}
  with respect to their induced PMF, under the condition that $q > \pi(\G)$, by Proposition \ref{PrpIdentifProb}.

First, we introduce a shorthand notation for the \emph{Kullback-Leibler (KL) divergence} between two Bernoulli distributions parameterized by $0\leq p_1\leq 1$ and $0\leq p_2\leq 1$:
\begin{equation} \label{KL}
\begin{array}{@{}l@{\hspace{0.025in}}l@{}}
KL(p_1\|p_2) & \equiv KL({\rm Bernoulli}(p_1)\|{\rm Bernoulli}(p_2)) \\
 & = p_1\log\frac{p_1}{p_2} + (1-p_1)\log\frac{1-p_1}{1-p_2} \; .
\end{array}
\end{equation}
Using this function, we can derive the following expression of the MLE problem.
\begin{lemma} \label{LemMLEGame}
Given a data set $\D = \{ \x\si{1},\dots,\x\si{m} \}$, define the
\emph{empirical proportion of equilibria}, i.e., the proportion of
samples in $\D$ that are equilibria of $\G$, as
\begin{equation} \label{PihDef}
\textstyle{ \pih(\G) \equiv \frac{1}{m}\sum_l \iverson{\x\si{l} \in
    \NE(\G)} } \; .
\end{equation}
\noindent The MLE problem for the probabilistic model given in
Eq.~\eqref {Prob} can be expressed as finding:
\begin{equation} \label{MLEGame}
(\Gh,\qh) \in \argmax_{(\G,q) \in \PS}{\Lh(\G,q)}{\rm ,  where \ } \Lh(\G,q) =
KL(\pih(\G)\|\pi(\G)) - KL(\pih(\G)\|q) - n\log2 \; ,
\end{equation}
\noindent 
where $\HS$ and $\PS$ are as in Definition~\ref{DefHS}, and $\pi(\G)$ is defined as in
Eq.~\eqref {PiDef}. Also, the optimal mixture parameter $\qh=\min(\pih(\G),1-\frac{1}{2m})$.
\end{lemma}
\begin{proof}
Let $\NE \equiv \NE(\G)$, $\pi \equiv \pi(\G)$ and $\pih \equiv \pih(\G)$. First, for a non-trivial $\G$, $\log p_{(\G,q)}(\x\si{l})=\log\frac{q}{|\NE|}$ for $\x\si{l} \in \NE$, and $\log p_{(\G,q)}(\x\si{l})=\log\frac{1-q}{2^n - |\NE|}$ for $\x\si{l}\notin \NE$. The average log-likelihood $\Lh(\G,q) =
\frac{1}{m}\sum_l{\log p_{\G,q}(\x\si{l})} =
\pih \log\frac{q}{|\NE|} + (1-\pih)\log\frac{1-q}{2^n - |\NE|} =
\pih \log\frac{q}{\pi} + (1-\pih)\log\frac{1-q}{1-\pi} - n\log2$. By adding $0 =
-\pih\log\pih + \pih\log\pih - (1-\pih)\log(1-\pih) + (1-\pih)\log(1-\pih)$, this can be rewritten as $\Lh(\G,q) =
\pih \log\frac{\pih}{\pi} + (1-\pih)\log\frac{1-\pih}{1-\pi} - \pih \log\frac{\pih}{q} - (1-\pih)\log\frac{1-\pih}{1-q} - n\log2$, and by using Eq.~\eqref {KL} we prove our claim.

Note that by maximizing with respect to the mixture parameter $q$ and by properties of the KL divergence, we get $KL(\pih\|\qh)=0 \Leftrightarrow \qh=\pih$.
We define our hypothesis space $\PS$ given the conditions in Remark \ref{RemUnif} and Propositions \ref{PrpEqGTNoEq} and \ref{PrpIdentifProb}.
For the case $\pih=1$, we ``shrink'' the optimal mixture parameter
$\qh$ to $1-\frac{1}{2m}$ in order to
enforce the second condition given in Remark \ref{RemUnif}.
\qedhere
\end{proof}

\begin{remark} \label{RemWorst}
Recall that a trivial game (e.g., LIG $\G=(\W,\b),\W=\zero,\b=\zero,\pi(\G)=1$) induces a uniform PMF by Remark \ref{RemUnif}, and therefore its log-likelihood is $-n\log2$.
Note that the lowest log-likelihood for non-trivial identifiable games
in Eq.~\eqref {MLEGame} is $-n\log2$ by setting the optimal mixture parameter $\qh=\pih(\G)$ and given that $KL(\pih(\G)\|\pi(\G))\geq 0$.
\end{remark}

Furthermore, Eq.~\eqref {MLEGame} implies that for non-trivial identifiable games $\G$, we expect the \emph{true proportion of equilibria} $\pi(\G)$ to be strictly less than the \emph{empirical proportion of equilibria} $\pih(\G)$ in the given data.
This is by setting the optimal mixture parameter $\qh=\pih(\G)$ and the condition $q>\pi(\G)$ in our hypothesis space.

\subsubsection{Learning LIGs via MLE: Model Selection}

Our main \emph{learning problem} consists of inferring the structure and
parameters of an LIG from data \emph{with the main purpose being modeling the game's
  PSNE, as reflected in the
  generative model.} Note that, as we have previously stated, 
different games (i.e., with different payoff functions) can be
PSNE-equivalent. For instance, the three following LIGs, with
different weight parameter matrices, induce the
same PSNE sets, i.e., $\NE(\W_k,\zero)=\{(-1,-1,-1), (+1,+1,+1)\}$ for $k=1,2,3$:\footnote{Using the formal mathematical definition of ``identifiability'' in statistics, we
would say that the
LIG examples prove that the model parameters $(\W,\b)$ of an LIG $\G$
are
\emph{not} identifiable with respect to the generative model
$p_{(\G,q)}$ defined in Eq.~\eqref{Prob}. We note that this situation
  is hardly exclusive to game-theoretic models. As example of an
  analogous issue in probabilistic graphical models is the fact that
  two Bayesian networks with different graph structures can
  represent not only the same conditional
  independence properties but also \emph{exactly} the same set of
  joint probability distributions~\citep{Chickering02,Koller09}. 

As a side note, we can
  distinguish these games
  with respect to their larger set of
  mixed-strategy Nash equilibria (MSNE), but, as stated previously,
  we do not consider MSNE
  in this paper because our primary motivation is the work of~\citet{Irfan13}, which is
  based on the concept of PSNE.}
\begin{equation*}
\W_1=\left[ \begin{array}{@{\hspace{0.025in}}c@{\hspace{0.05in}}c@{\hspace{0.05in}}c@{\hspace{0.025in}}}
0 & 0 & 0 \\
1/2 & 0 & 0 \\
0 & 1 & 0 \\
\end{array} \right]
{\rm ,\ \ }
\W_2=\left[ \begin{array}{@{\hspace{0.025in}}c@{\hspace{0.05in}}c@{\hspace{0.05in}}c@{\hspace{0.025in}}}
0 & 0 & 0 \\
2 & 0 & 0 \\
1 & 0 & 0 \\
\end{array} \right]
{\rm ,\ \ }
\W_3=\left[ \begin{array}{@{\hspace{0.025in}}c@{\hspace{0.05in}}c@{\hspace{0.05in}}c@{\hspace{0.025in}}}
0 & 1 & 1 \\
1 & 0 & 1 \\
1 & 1 & 0 \\
\end{array} \right] \; .
\end{equation*} 
Thus, not only the MLE may not be unique, but also all
such PSNE-equivalent MLE games will achieve the same level of
\emph{generalization} performance. But, as reflected by our
generative model, our main interest in
the model parameters of the LIGs is only with respect to the PSNE they
induce, not the model parameters \emph{per se}. Hence, all
we need is a way to select
among PSNE-equivalent LIGs.

In our work, \emph{the indentifiability or interpretability of exact model parameters of
LIGs  is not our main interest.}
 That is, in the research presented
here, we did not seek or attempt to work on creating alternative generative models with the
objective to
provide a theoretical guarantee that, given an infinite
amount of data, we can recover the model parameters of an unknown
ground-truth model generating the data, assuming the ground-truth
model is an LIG.
We opted for a more practical ML approach
in which we just want to learn a \emph{single} LIG
that has good generalization performance (i.e., predictive performance
in terms of average log-likelihood) with respect to our generative
model. Given the nature of our generative model, this essentially
translate to learning an LIG that captures as best as possible the PSNE of the
unknown ground-truth game. Unfortunately, as we just illustrated, several LIGs
with different model parameter values can have the same set of
PSNE. Thus, they all would have the same (generalization) performance ability.

As we all know, model selection is core to ML. One of the reason we
chose an ML-approach to learning games is precisely the elegant way in
which ML deals with the problem of how to select among multiple models
that achieve the same level of performance: invoke the principle of Ockham's
razor and select the ``simplest'' model among those with the same
(generalization) performance. This ML philosophy is not \emph{ad hoc}.
It is instead well established in practice and well supported by
theory. Seminal results from
the various theories of learning, such as computational and statistical
learning theory and PAC learning, support the well-known ML adage
that ``learning requires bias.'' In short, as is by now standard in an ML-approach, we measure the quality of our data-induced models via their generalization ability and invoke the principle of Ockham's razor to bias our search toward simpler models using well-known and -studied regularization techniques.

Now, as we also all know, exactly what ``simple'' and ``bias'' means
depends on the problem. In our case, we would prefer games with sparse graphs, if for no reason
other than to simplify analysis, exploration, study, and (\emph{visual}) ``interpretability'' of the game
model \emph{by human experts}, everything else being equal
(i.e., models with the same explanatory power on the data as measured
by the likelihoods).\footnote{Just to be clear, here we mean ``interpretability''
  not in any formal mathematical sense, or as often used in some areas of the
  social sciences such as economics. But, instead, as we typically use
  it in ML/AI textbooks, such as
  for example, when referring to shallow/sparse decision trees,
  generally considered to be easier to explain and understand. Similarly, the hope here is that the
  ``sparsity''  or ``simplicity'' of the game graph/model would make
  it also 
  simpler for human experts to explain or understand what about the model is
  leading them to generate novel hypotheses, reach certain conclusions
  or make certain inferences about the \emph{global}
  strategic behavior of the agents/players, such as those based on the
  game's PSNE and facilitated by
  CSI. We should also point out that, in preliminary empirical work, we have observed that the
  \emph{representationally sparser} the LIG graph, the
\emph{computationally easier} it is for 
algorithms and other heuristics that operate on the LIG, as those
of~\citet{Irfan13} for CSI, for example.} For example, among the LIGs
presented above, using structural properties alone, we would generally prefer
the former two models to the latter, all else being equal
(e.g., generalization performance).

\section{Generalization Bound and VC-Dimension} \label{SecGeneralization}

In this section, we show a generalization bound for the maximum likelihood problem as well as an upper bound of the VC-dimension of LIGs.
Our objective is to establish that with probability at least $1-\delta$, for some confidence parameter $0<\delta<1$, the maximum likelihood estimate is within $\epsilon > 0$ of the optimal parameters, in terms of achievable expected log-likelihood.

Given the ground-truth distribution $\DD$ of the data, let $\pib(\G)$ be the \emph{expected proportion of equilibria}, i.e.,
\begin{equation*}
\pib(\G) = \P_\DD[\x\in \NE(\G)] \; ,
\end{equation*}
\noindent and let $\Lb(\G,q)$ be the \emph{expected log-likelihood} of a generative model from game $\G$ and mixture parameter $q$, i.e.,
\begin{equation*}
\Lb(\G,q) = \E_\DD[\log p_{(\G,q)}(\x)] \; .
\end{equation*}
Let $\thetah \equiv (\Gh,\qh)$ be a maximum-likelihood estimate as in
Eq.~\eqref {MLEGame} (i.e., $\thetah \in \argmax_{\theta \in \PS}{\Lh(\theta)}$), and $\thetab\equiv (\Gb,\qb)$ be the \emph{maximum-expected-likelihood estimate}: $\thetab \in \argmax_{\theta \in
\PS}{\Lb(\theta)}$.\footnote{If the ground-truth model belongs to the
  class of LIGs, then $\thetab$ is also the ground-truth model, or
  PSNE-equivalent to it.} We use, without formally re-stating, the last definitions in the
technical results presented in the remaining of this section.

Note that our hypothesis space $\PS$ 
as stated in Definition~\ref{DefHS} includes a continuous parameter
$q$ that could potentially have infinite VC-dimension. 
The following lemma will allow us later to prove that uniform convergence for the extreme values of $q$ implies uniform convergence for all $q$ in the domain.
\begin{lemma} \label{LemQ}
Consider any game $\G$ and, for $0 < q'' < q' < q < 1$, let $\theta = (\G,q)$, $\theta'=(\G,q')$ and $\theta''=(\G,q'')$.
If, for any $\epsilon > 0$ we have $| \Lh(\theta) - \Lb(\theta) | \leq \epsilon/2$ and $| \Lh(\theta'') - \Lb(\theta'') | \leq \epsilon/2$, then $| \Lh(\theta') - \Lb(\theta') | \leq \epsilon/2$.
\end{lemma}
\begin{proof}
Let $\NE \equiv \NE(\G)$, $\pi \equiv \pi(\G)$, $\pih \equiv \pih(\G)$, $\pib \equiv \pib(\G)$, and $\E[\cdot]$ and $\P[\cdot]$ be the expectation and probability with respect to the ground-truth distribution $\DD$ of the data.

First note that for any $\theta = (\G,q)$, we have $\Lb(\theta) = \E[\log p_{(\G,q)}(\x)] =
\E[ \iverson{\x \in \NE} \log\frac{q}{| \NE |} + \iverson{\x \notin \NE} \log\frac{1-q}{2^n - | \NE |} ] =
\P[\x \in \NE] \log\frac{q}{| \NE |} + \P[\x \notin \NE] \log\frac{1-q}{2^n - | \NE |} =
\pib \log\frac{q}{| \NE |} + (1-\pib)  \log\frac{1-q}{2^n - | \NE |} =
\pib \log\left(\frac{q}{1-q} \cdot \frac {2^n - | \NE |}{| \NE |}\right) + \log\frac{1-q}{2^n - | \NE |} =
\pib \log\left(\frac{q}{1-q} \cdot \frac {1 - \pi }{\pi}\right) + \log\frac{1-q}{1 - \pi } - n\log2$.

Similarly, for any $\theta = (\G,q)$, we have $\Lh(\theta) = \pih \log\left(\frac{q}{1-q} \cdot \frac {1 - \pi }{\pi}\right) + \log\frac{1-q}{1 - \pi } - n\log2$.
So that $\Lh(\theta) - \Lb(\theta) = (\pih - \pib) \log\left(\frac{q}{1-q} \cdot \frac {1 - \pi }{\pi}\right)$.

Furthermore, the function $\frac{q}{1-q}$ is strictly monotonically increasing for $0 \leq q < 1$.
If $\pih > \pib$ then $-\epsilon/2 \leq \Lh(\theta'') - \Lb(\theta'') < \Lh(\theta') - \Lb(\theta') < \Lh(\theta) - \Lb(\theta) \leq \epsilon/2$.
Else, if $\pih < \pib$, we have $\epsilon/2 \geq \Lh(\theta'') - \Lb(\theta'') > \Lh(\theta') - \Lb(\theta') > \Lh(\theta) - \Lb(\theta) \geq -\epsilon/2$.
Finally, if $\pih = \pib$ then $\Lh(\theta'') - \Lb(\theta'') = \Lh(\theta') - \Lb(\theta') = \Lh(\theta) - \Lb(\theta) = 0$.
\end{proof}

In the remaining of this section, denote by $d(\HS) \equiv \left| \cup_{\G
    \in \HS} \{\NE(\G)\} \right|$ the number of all possible PSNE
sets induced by games in $\HS$, the class of games of interest.

The following theorem shows that the expected log-likelihood of the
maximum likelihood estimate $\thetah$ converges in probability to that
of the optimal $\thetab=(\Gb,\qb)$, as the data size $m$ increases.
\begin{theorem} \label{ThmGen}
The following holds with $\DD$-probability at least $1-\delta$:
\begin{equation*}
\textstyle{ \Lb(\thetah) \geq \Lb(\thetab) -
  \left(\log\max(2m,\frac{1}{1-\qb}) +
    n\log2\right)\sqrt{\frac{2}{m}\left(\log
      d(\HS)+\log\frac{4}{\delta}\right)} } \; .
\end{equation*}
\end{theorem}
\begin{proof}
First our objective is to find a lower bound for $\P[\Lb(\thetah) - \Lb(\thetab) \geq -\epsilon] \geq
\P[\Lb(\thetah) - \Lb(\thetab) \geq -\epsilon + (\Lh(\thetah) - \Lh(\thetab))] \geq
\P[- \Lh(\thetah) + \Lb(\thetah) \geq -\frac{\epsilon}{2}, \Lh(\thetab) - \Lb(\thetab) \geq -\frac{\epsilon}{2}] =
\P[\Lh(\thetah) - \Lb(\thetah) \leq \frac{\epsilon}{2}, \Lh(\thetab) - \Lb(\thetab) \geq -\frac{\epsilon}{2}] =
1 - \P[\Lh(\thetah) - \Lb(\thetah) > \frac{\epsilon}{2} \vee \Lh(\thetab) - \Lb(\thetab) < -\frac{\epsilon}{2}]$.

Let $\qt \equiv \max(1-\frac{1}{2m},\qb)$. Now, we have $\P[\Lh(\thetah) - \Lb(\thetah) > \frac{\epsilon}{2} \vee \Lh(\thetab) - \Lb(\thetab) < -\frac{\epsilon}{2}] \leq
\P[(\exists \theta\in \PS, q \leq \qt){\rm\ }| \Lh(\theta) - \Lb(\theta) | > \frac{\epsilon}{2}] =
\P[(\exists \theta, \G\in\HS, q \in\{\pi(\G),\qt\}){\rm\ }| \Lh(\theta) - \Lb(\theta) | > \frac{\epsilon}{2}]$. The last equality follows from invoking Lemma \ref{LemQ}.

Note that $\E[\Lh(\theta)] = \Lb(\theta)$ and that since $\pi(\G)\leq q\leq \qt$, the log-likelihood is bounded as $(\forall \x){\rm\ }-B \leq \log p_{(\G,q)}(\x) \leq 0$, where $B = \log\frac{1}{1-\qt} + n\log2 = \log\max(2m,\frac{1}{1-\qb}) + n\log2$.
Therefore, by Hoeffding's inequality, we have $\P[| \Lh(\theta) - \Lb(\theta) | > \frac{\epsilon}{2}] \leq 2\exp{-\frac{m\epsilon^2}{2B^2}}$.

Furthermore, note that there are $2d(\HS)$ possible parameters
$\theta$, since we need to consider only two values of $q\in
\{\pi(\G),\qt\})$ and because the number of all possible PSNE sets
induced by games in $\HS$ is $d(\HS) \equiv \left| \cup_{\G \in \HS} \{\NE(\G)\} \right|$.
Therefore, by the union bound we get the following uniform convergence $\P[(\exists \theta, \G\in\HS, q \in\{\pi(\G),\qt\}){\rm\ }| \Lh(\theta) - \Lb(\theta) | > \frac{\epsilon}{2}] \leq
4d(\HS)\P[| \Lh(\theta) - \Lb(\theta) | > \frac{\epsilon}{2}] \leq
4d(\HS)\exp{-\frac{m\epsilon^2}{2B^2}} = \delta$.
Finally, by solving for $\delta$ we prove our claim.
\qedhere
\end{proof}

\begin{remark}
A more elaborate analysis allows to tighten the bound in Theorem \ref{ThmGen} from $\O(\log\frac{1}{1-\qb})$ to $\O(\log\frac{\qb}{1-\qb})$.
We chose to provide the former result for clarity of presentation.
\end{remark}

The following theorem establishes the complexity of the class of LIGs, which implies that the term $\log d(\HS)$ of the generalization bound in Theorem \ref{ThmGen} is only polynomial in the number of players $n$.
\begin{theorem} \label{ThmVCDim}
If $\HS$ is the class of \emph{LIGs}, then $d(\HS)\equiv \left| \cup_{\G \in \HS} \{\NE(\G)\} \right| \leq 2^{n \frac{n(n+1)}{2}+1} \leq 2^{n^3}$.
\end{theorem}
\begin{proof}
The logarithm of the number of possible pure-strategy Nash equilibria sets supported by $\HS$ (i.e., that can be produced by some game in $\HS$) is upper bounded by the VC-dimension of the class of neural networks with a single hidden layer of $n$ units and $n+ {n \choose 2}$ input units, linear threshold activation functions, and constant output weights.

For every LIG $\G=(\W,\b)$ in $\HS$, define the neural network with a single layer of $n$ hidden units, $n$ of the inputs corresponds to the linear terms $x_1,\ldots,x_n$ and ${n \choose 2}$ corresponds to the quadratic polynomial terms $x_ix_j$ for all pairs of players $(i,j)$, $1\leq i < j \leq n$. For every hidden unit $i$, the weights corresponding to the linear terms $x_1,\ldots,x_n$ are $-b_1,\ldots,-b_n$, respectively, while the weights corresponding to the quadratic terms $x_ix_j$ are $-w_{ij}$, for all pairs of players $(i,j)$, $1\leq i < j \leq n$, respectively. The weights of the bias term of all the hidden units are set to $0$. All $n$ output weights are set to $1$ while the weight of the output bias term is set to $0$. The output of the neural network is $\iverson{\x \notin \NE(\G)}$. Note that we define the neural network to classify non-equilibrium as opposed to equilibrium to keep the convention in the neural network literature to define the threshold function to output $0$ for input $0$. The alternative is to redefine the threshold function to output $1$ instead for input $0$.

Finally, we use the VC-dimension of neural networks~\citep{sontag98vcdimension}.
\qedhere
\end{proof}

From Theorems \ref{ThmGen} and \ref{ThmVCDim}, we state the
generalization bounds for LIGs.
\begin{corollary}
The following holds with $\DD$-probability at least $1-\delta$:
\begin{equation*}
\textstyle{ \Lb(\thetah) \geq \Lb(\thetab) - \left(\log\max(2m,\frac{1}{1-\qb}) + n\log2\right)\sqrt{\frac{2}{m}\left(n^3\log2+\log\frac{4}{\delta}\right)} } \; ,
\end{equation*}
\noindent where $\HS$ is the class of \emph{LIGs}, in which case $\PS \equiv \{(\G,q) \mid
\G\in\HS \wedge 0<\pi(\G)<q<1\}$ (Definition~\ref{DefHS}) becomes the hypothesis space of non-trivial identifiable \emph{LIGs} and mixture parameters.
\end{corollary}

\section{Algorithms} \label{SecAlgorithm}

In this section, we approximate the maximum likelihood problem by
maximizing the number of observed equilibria in the data, suitable for
a hypothesis space of games with small true proportion of
equilibria.
We then present our convex loss minimization approach.
We also discuss baseline methods such as sigmoidal approximation and exhaustive search.

But first, let us discuss some negative results that justifies the use
of simple approaches.  \citet{Irfan13} showed that counting the number of Nash equilibria in
LIGs is \#P-complete; thus, computing the log-likelihood function, and
therefore MLE, is NP-hard.\footnote{This is not a disadvantage relative to probabilistic graphical models, since computing the log-likelihood function is also NP-hard for Ising models and Markov random fields in general, while learning is also NP-hard for Bayesian networks.}
General approximation techniques such as pseudo-likelihood estimation
do not lead to tractable methods for learning LIGs.\footnote{We show that evaluating the pseudo-likelihood
  function for our generative model is NP-hard. Consider a
  non-trivial LIG $\G=(\W,\b)$. 
Furthermore, assume $\G$ has a
  single \emph{non-absolutely-indifferent} player $i$ and
  \emph{absolutely-indifferent} players $\forall j\neq i$; that is, assume that $(\w{i},b_i) \neq \zero$ and $(\forall j\neq i){\rm\ }(\w{j},b_j)=\zero$ (See Definition \ref{DefAbsIndifferent}).
  Let $f_i(\x_\minus{i}) \equiv \t{\w{i}}\x_\minus{i}-b_i$, we have $\iverson{\x \in \NE(\G)} =
  \iverson{x_i f_i(\x_\minus{i})\geq 0}$ and therefore $p_{(\G,q)}(\x)
  = q \frac{\iverson{x_i f_i(\x_\minus{i})\geq 0}}{|\NE(\G)|} + (1-q)
  \frac{1 - \iverson{x_i f_i(\x_\minus{i})\geq 0}}{2^n -
    |\NE(\G)|}$. The result follows because computing $|\NE(\G)|$ is
  \#P-complete, even for this specific instance of a single \emph{non-absolutely-indifferent} player~\citep{Irfan13}.}
From an optimization perspective, the log-likelihood function is not continuous because of the number of equilibria.
Therefore, we cannot rely on concepts such as Lipschitz
continuity.\footnote{To prove that counting the number of equilibria is not
  (Lipschitz) continuous, we show how small changes in the parameters $\G=(\W,\b)$ can produce big changes in $|\NE(\G)|$. For instance, consider two games $\G_k=(\W_k,\b_k)$, where $\W_1=\zero,\b_1=\zero,|\NE(\G_1)|=2^n$ and $\W_2=\varepsilon(\one\t{\one}-\I),\b_2=\zero,|\NE(\G_2)|=2$ for $\varepsilon>0$. For $\varepsilon \rightarrow 0$, any $\ell_p$-norm $\|\W_1 - \W_2\|_p \rightarrow 0$ but $|\NE(\G_1)| - |\NE(\G_2)| = 2^n-2$ remains constant.}
Furthermore, bounding the number of equilibria by known bounds for
Ising models leads to trivial bounds.\footnote{The log-partition
  function of an Ising model is a trivial bound for counting the
  number of equilibria. To see this, let $f_i(\x_\minus{i}) \equiv \t{\w{i}}\x_\minus{i}-b_i$, $|\NE(\G)| = \sum_\x\prod_i\iverson{x_i f_i(\x_\minus{i})\geq 0} \leq
\sum_\x\prod_i{\exp{x_i f_i(\x_\minus{i})}} =
\sum_\x{\exp{\t{\x}\W\x-\t{\b}\x}} =
\Z(\frac{1}{2}(\W+\t{\W}),\b)$, where $\Z$ denotes the partition
function of an Ising model. Given the convexity of $\Z$~\citep{Koller09}, and that the gradient vanishes at $\W=\zero,\b=\zero$, we know that $\Z(\frac{1}{2}(\W+\t{\W}),\b)\geq 2^n$, which is the maximum $|\NE(\G)|$.
}

\subsection{An Exact Quasi-Linear Method for General Games: Sample-Picking} \label{SubSecSamplePicking}

As a first approach, consider solving the maximum likelihood estimation problem in Eq.~\eqref {MLEGame} by an exact exhaustive search algorithm.
This algorithm iterates through all possible Nash equilibria sets, i.e., for $s = 0,\dots,2^n$, we generate all possible sets of size $s$ with elements from the joint-action space $\{-1,+1\}^n$.
Recall that there exist ${2^n \choose s}$ of such sets of size $s$ and since $\sum_{s=0}^{2^n}{2^n \choose s} = 2^{2^n}$ the search space is super-exponential in the number of players $n$.

\begin{algorithm}[tb]
\begin{small}
\captionx{Sample-Picking for General Games}
\label{SamplePicking} 
\begin{algorithmic}
    \STATE {\bfseries Input:} Data set $\D = \x\si{1},\dots,\x\si{m}$.
    \STATE Compute the unique samples $\y\si{1},\dots,\y\si{U}$ and their frequency $\ph\si{1},\dots,\ph\si{U}$ in the data set $\D$.
    \STATE Sort joint actions by their frequency such that $\ph\si{1} \geq \ph\si{2} \geq \dots \geq \ph\si{U}$.
    \FOR{ each unique sample $k = 1,\dots,U$}
        \STATE Define $\G_k$ by the Nash equilibria set $\NE(\G_k) = \{\y\si{1},\dots,\y\si{k}\}$.
        \STATE Compute the log-likelihood $\Lh(\G_k,\qh_k)$ in Eq.~\eqref {MLEGame} (Note that $\qh_k = \pih(\G) = \frac{1}{m}{(\ph\si{1}+\dots+\ph\si{k})}$, $\pi(\G) = \frac{k}{2^n}$).
    \ENDFOR
    \STATE {\bfseries Output:} The game $\G_{\widehat{k}}$ such that $\widehat{k} = \argmax_k\Lh(\G_k,\qh_k)$.
\end{algorithmic}
\end{small}
\end{algorithm}

Based on few observations, we can obtain an $\O(m\log m)$ algorithm for $m$ samples.
First, note that the above method does not constrain the set of Nash equilibria in any fashion.
Therefore, only joint actions that are observed in the data are candidates of being Nash equilibria in order to maximize the log-likelihood.
This is because the introduction of an unobserved joint action will increase the true proportion of equilibria without increasing the empirical proportion of equilibria and thus leading to a lower log-likelihood in Eq.~\eqref {MLEGame}.
Second, given a fixed number of Nash equilibria $k$, the best strategy would be to pick the $k$ joint actions that appear more frequently in the observed data.
This will maximize the empirical proportion of equilibria, which will maximize the log-likelihood.
Based on these observations, we propose Algorithm~\ref{SamplePicking}.

As an aside note, the fact that general games do not constrain the set of Nash equilibria, makes the method more likely to over-fit.
On the other hand, LIGs will potentially include unobserved
equilibria given the linearity constraints in the search space, and
thus they would be less likely to
over-fit.

\subsection{An Exact Super-Exponential Method for LIGs: Exhaustive Search} \label{SubSecExhaustive}

Note that in the previous subsection, we search in the space of all possible games, not only the LIGs.
First note that \emph{sample-picking} for linear games is NP-hard, i.e., at any iteration of \emph{sample-picking}, checking whether the set of Nash equilibria $\NE$ corresponds to an LIG or not is equivalent to the following constraint satisfaction problem with linear constraints:
\begin{equation} \label{CSP}
\begin{array}{@{}l@{}}
  \displaystyle{\min_{\W,\b}{{\rm\ }1}} \\
  \vspace{-0.12in} \\
  \begin{array}{@{}l@{\hspace{0.08in}}l@{\hspace{0.05in}}l@{\hspace{0.05in}}l@{\hspace{0.05in}}l@{}}
    {\rm s.t.} & (\forall \x \in \NE) & x_1 (\t{\w{1}}\x_\minus{1}-b_1) \geq 0 & \wedge \dots \wedge & x_n (\t{\w{n}}\x_\minus{n}-b_n) \geq 0 \; ,\\
     & (\forall \x \notin \NE) & x_1 (\t{\w{1}}\x_\minus{1}-b_1) < 0 & \vee \dots \vee & x_n (\t{\w{n}}\x_\minus{n}-b_n) < 0 \; .
  \end{array}
\end{array}
\end{equation}
Note that Eq.~\eqref {CSP} contains ``or'' operators in order to account for the non-equilibria.
This makes the problem of finding the $(\W,\b)$ that satisfies such conditions NP-hard for a non-empty complement set $\{-1,+1\}^n - \NE$.
Furthermore, since \emph{sample-picking} only consider observed equilibria, the search is not optimal with respect to the space of LIGs.

Regarding a more refined approach for enumerating LIGs only, note that in an LIG each player separates hypercube vertices with a linear function, i.e., for $\v \equiv (\w{i},b_i)$ and $\y \equiv (x_i \x_\minus{i},-x_i) \in \{-1,+1\}^n$ we have $x_i (\t{\w{i}}\x_\minus{i}-b_i) = \t{\v}\y$.
Assume we assign a binary label to each vertex $\y$, then note that not all possible labelings are linearly separable.
Labelings which are linearly separable are called \emph{linear threshold functions (LTFs)}.
A lower bound of the number of LTFs was first provided in \citet{Muroga65}, which showed that the number of LTFs is at least $\alpha(n) \equiv 2^{0.33048 n^2}$.
Tighter lower bounds were shown later in \citet{Yamija65} for $n \geq 6$ and in \citet{Muroga66} for $n \geq 8$.
Regarding an upper bound, \citet{Winder60} showed that the number of LTFs is at most $\beta(n) \equiv 2^{n^2}$.
By using such bounds for all players, we can conclude that there is at least ${\alpha(n)}^n = 2^{0.33048 n^3}$ and at most ${\beta(n)}^n = 2^{n^3}$ LIGs (which is indeed another upper bound of the VC-dimension of the class of LIGs; the bound in Theorem \ref{ThmVCDim} is tighter and uses bounds of the VC-dimension of neural networks).
The bounds discussed above would bound the time-complexity of a search algorithm if we could easily enumerate all LTFs for a single player.
Unfortunately, this seems to be far from a trivial problem.
By using results in \citet{Muroga71}, a weight vector $\v$ with integer entries such that $(\forall i){\rm\ }|v_i| \leq \beta(n) \equiv {(n+1)}^{(n+1)/2}/2^n$ is sufficient to realize all possible LTFs.
Therefore we can conclude that enumerating LIGs takes at most ${(2\beta(n)+1)}^{n^2} \approx {(\frac{\sqrt{n+1}}{2})}^{n^3}$ steps, and we propose the use of this method only for $n \leq 4$.

For $n=4$ we found that the number of possible PSNE sets induced by LIGs is 23,706.
Experimentally, we did not find differences between this method and \emph{sample-picking} since most of the time, the model with maximum likelihood was an LIG.

\subsection{From Maximum Likelihood to Maximum Empirical Proportion of Equilibria}

We approximately perform maximum likelihood estimation for LIGs, by maximizing the \emph{empirical proportion of equilibria}, i.e., the equilibria in the observed data.
This strategy allows us to avoid computing $\pi(\G)$ as in Eq.~\eqref {PiDef} for maximum likelihood estimation (given its dependence on $|\NE(\G)|$).
We propose this approach for games with small true proportion of equilibria with high probability, i.e., with probability at least $1-\delta$, we have $\pi(\G) \leq \frac{\kappa^n}{\delta}$ for $1/2\leq\kappa<1$.
Particularly, we will show in Section \ref{SecSmallTPE} that for LIGs we have $\kappa=3/4$.
Given this, our approximate problem relies on a bound of the log-likelihood that holds with high probability.
We also show that under very mild conditions, the parameters $(\G,q)$ belong to the hypothesis space of the original problem with high probability.

First, we derive bounds on the log-likelihood function.
\begin{lemma} \label{LemBounds}
Given a non-trivial game $\G$ with $0<\pi(\G)<\pih(\G)$, the KL divergence in the log-likelihood function in Eq.~\eqref {MLEGame} is bounded as follows:
\begin{equation*}
-\pih(\G)\log\pi(\G) - \log2 < KL(\pih(\G)\|\pi(\G)) < -\pih(\G)\log\pi(\G) \; .
\end{equation*}
\end{lemma}
\begin{proof}
Let $\pi \equiv \pi(\G)$ and $\pih \equiv \pih(\G)$. Note that
$\alpha(\pi) \equiv \lim_{\pih \rightarrow
  0}{KL(\pih\|\pi)}=0$,\footnote{Here we are making the implicit
  assumption that $\pi < \pih$. This is sensible. For example, in most models learned from the congressional voting
  data using a variety of learning algorithms we propose, the total
  number of PSNE would range roughly from 100K---1M; using base 2,
  this is roughly from $2^{16}$---$2^{20}$. This may look like a huge number
  until one recognizes that there could potential be $2^{100}$
  PSNE. Hence, we have that $\pi$ would be in the range of
  $2^{-84}$---$2^{-80}$. In fact, we believe this holds more broadly
  because, as a general objective, we want models that can capture as
  many PSNE behavior as possible but no more than needed, which tend
  to reduce the PSNE of the learned models, and thus their $\pi$
  values, while simultaneously trying to increase $\pih$ as much as possible.} and $\beta(\pi) \equiv \lim_{\pih \rightarrow 1}{KL(\pih\|\pi)} = -\log\pi \leq n\log2$. Since the function is convex we can upper-bound it by $\alpha(\pi)+(\beta(\pi)-\alpha(\pi))\pih = -\pih\log\pi$.

To find a lower bound, we find the point in which the derivative of the original function is equal to the slope of the upper bound, i.e., $\frac{\partial KL(\pih\|\pi)}{\partial\pih} = \beta(\pi)-\alpha(\pi) = -\log\pi$, which gives $\pih^*=\frac{1}{2-\pi}$. Then, the maximum difference between the upper bound and the original function is given by $\lim_{\pi \rightarrow 0}{-\pih^*\log\pi - KL(\pih^*\|\pi)} = \log2$.
\qedhere
\end{proof}

Note that the lower and upper bounds are very informative when $\pi(\G) \rightarrow 0$ (or in our setting when $n \rightarrow +\infty$), since $\log2$ becomes small when compared to $-\log\pi(\G)$, as shown in Figure \ref{Gap}.

\begin{figure}
\begin{center}
\includegraphics[width=0.25\linewidth]{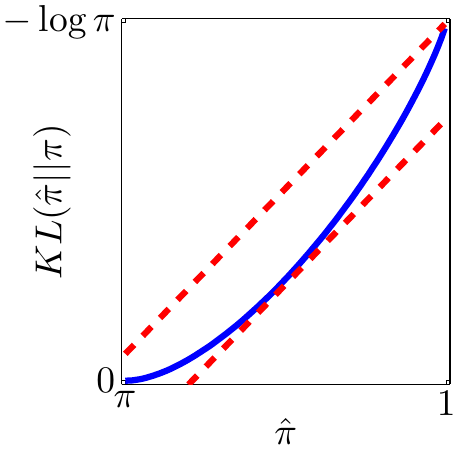}\includegraphics[width=0.25\linewidth]{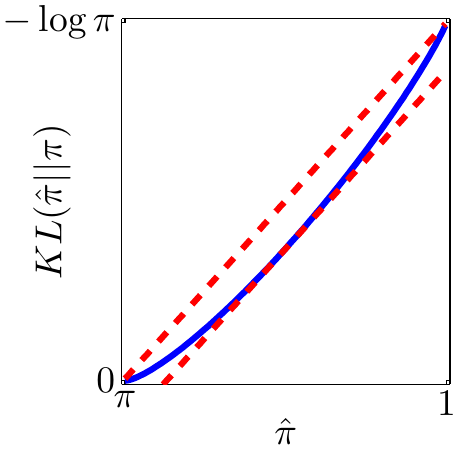}\includegraphics[width=0.25\linewidth]{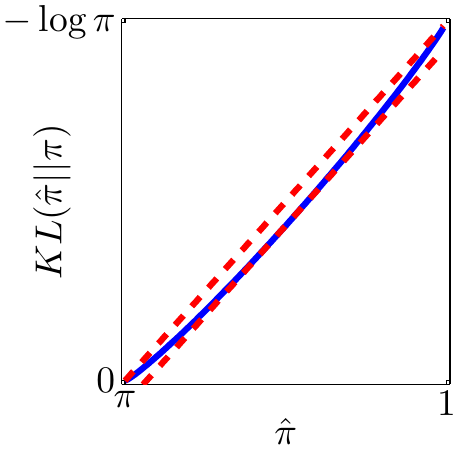}
\end{center}
\vspace{-0.3in}
\captionx{KL divergence (blue) and bounds derived in Lemma \ref{LemBounds} (red) for $\pi=(3/4)^n$ where $n=9$ (left), $n=18$ (center) and $n=36$ (right). Note that the bounds are very informative when $n \rightarrow +\infty$ (or equivalently when $\pi \rightarrow 0$).}
\vspace{-0.1in}
\label{Gap}
\end{figure}

Next, we derive the problem of maximizing the empirical proportion of equilibria from the maximum likelihood estimation problem.
\begin{theorem}
Assume that with probability at least $1-\delta$ we have $\pi(\G) \leq \frac{\kappa^n}{\delta}$ for $1/2\leq\kappa<1$.
Maximizing a lower bound (with high probability) of the log-likelihood in Eq.~\eqref {MLEGame} is equivalent to maximizing the empirical proportion of equilibria:
\begin{equation} \label{MEPEGame}
\max_{\G \in \HS}{{\rm\ }\pih(\G)} \; ,
\end{equation}
\noindent furthermore, for all games $\G$ such that $\pih(\G) \geq \gamma$ for
some $0<\gamma<1/2$, for sufficiently large
$n>\log_\kappa{(\delta\gamma)}$ and optimal mixture parameter
$\qh=\min(\pih(\G),1-\frac{1}{2m})$, we have $(\G,\qh) \in \PS$, where
$\PS = \{(\G,q) \mid \G\in\HS \wedge 0<\pi(\G)<q<1\}$ is the
hypothesis space of non-trivial identifiable games and mixture parameters.
\end{theorem}
\begin{proof}
By applying the lower bound in Lemma \ref{LemBounds} in Eq.~\eqref {MLEGame} to non-trivial games, we have $\Lh(\G,\qh) =
KL(\pih(\G)\|\pi(\G)) - KL(\pih(\G)\|\qh) - n\log2 >
-\pih(\G)\log\pi(\G) - KL(\pih(\G)\|\qh) - (n+1)\log2$.
Since $\pi(\G) \leq \frac{\kappa^n}{\delta}$, we have $-\log\pi(\G) \geq -\log\frac{\kappa^n}{\delta}$.
Therefore $\Lh(\G,\qh) > -\pih(\G)\log\frac{\kappa^n}{\delta} - KL(\pih(\G)\|\qh) - (n+1)\log2$.
Regarding the term $KL(\pih(\G)\|\qh)$, if $\pih(\G)<1 \Rightarrow KL(\pih(\G)\|\qh) = KL(\pih(\G)\|\pih(\G)) = 0$, and if $\pih(\G)=1 \Rightarrow KL(\pih(\G)\|\qh) = KL(1\|1-\frac{1}{2m}) = -\log(1-\frac{1}{2m}) \leq \log2$ and approaches 0 when $m \rightarrow +\infty$.
Maximizing the lower bound of the log-likelihood becomes $\max_{\G \in \HS}{\pih(\G)}$ by removing the constant terms that do not depend on $\G$.

In order to prove $(\G,\qh) \in \PS$ we need to prove $0<\pi(\G)<\qh<1$.
For proving the first inequality $0<\pi(\G)$, note that $\pih(\G) \geq \gamma>0$, and therefore $\G$ has at least one equilibria.
For proving the third inequality $\qh<1$, note that $\qh=\min(\pih(\G),1-\frac{1}{2m})<1$.
For proving the second inequality $\pi(\G)<\qh$, we need to prove $\pi(\G)<\pih(\G)$ and $\pi(\G)<1-\frac{1}{2m}$.
Since $\pi(\G) \leq \frac{\kappa^n}{\delta}$ and $\gamma \leq \pih(\G)$, it suffices to prove $\frac{\kappa^n}{\delta}<\gamma \Rightarrow \pi(\G)<\pih(\G)$.
Similarly we need to prove $\frac{\kappa^n}{\delta}<1-\frac{1}{2m} \Rightarrow \pi(\G)<1-\frac{1}{2m}$.
Putting both together, we have $\frac{\kappa^n}{\delta}<\min(\gamma,1-\frac{1}{2m})=\gamma$ since $\gamma<1/2$ and $1-\frac{1}{2m} \geq 1/2$.
Finally, $\frac{\kappa^n}{\delta}<\gamma \Leftrightarrow n>\log_\kappa{(\delta\gamma)}$.
\qedhere
\end{proof}

\subsection{A Non-Concave Maximization Method: Sigmoidal Approximation} \label{SubSecSigmoidal}

A very simple optimization approach can be devised by using a sigmoid in order to approximate the 0/1 function $\iverson{z\geq 0}$ in the maximum likelihood problem of Eq.~\eqref {MLEGame} as well as when maximizing the empirical proportion of equilibria as in Eq.~\eqref {MEPEGame}.
We use the following sigmoidal approximation:
\begin{equation} \label{Sigmoid}
\textstyle{ \iverson{z\geq 0} \approx H_{\alpha,\beta}(z) \equiv \frac{1}{2}(1 + \tanh(\frac{z}{\beta} - {\rm arctanh}(1-2\alpha^{1/n}))) } \; .
\end{equation}
The additional term $\alpha$ ensures that for $\G=(\W,\b),\W=\zero,\b=\zero$ we get $\iverson{\x\in \NE(\G)} \approx H_{\alpha,\beta}(0)^n = \alpha$.
We perform gradient ascent on these objective functions that have many local maxima.
Note that when maximizing the ``sigmoidal'' likelihood, each step of the gradient ascent is NP-hard due to the ``sigmoidal'' true proportion of equilibria.
Therefore, we propose the use of the sigmoidal maximum likelihood only for $n \leq 15$.

In our implementation, we add an $\ell_1$-norm regularizer $-\rho\|\W\|_1$ where $\rho>0$ to both maximization problems.
The $\ell_1$-norm regularizer encourages sparseness and attempts to lower the generalization error by controlling over-fitting.

\subsection{Our Proposed Approach: Convex Loss Minimization} \label{SubSecCLM}

From an optimization perspective, it is more convenient to minimize a convex objective instead of a sigmoidal approximation in order to avoid the many local minima.

Note that maximizing the empirical proportion of equilibria in Eq.~\eqref {MEPEGame} is equivalent to minimizing the empirical proportion of non-equilibria, i.e., $\min_{\G \in \HS}{(1-\pih(\G))}$.
Furthermore, $1-\pih(\G) = \frac{1}{m}\sum_l{\iverson{\x\si{l}\notin \NE(\G)}}$.
Denote by $\ell$ the 0/1 loss, i.e., $\ell(z)=\iverson{z<0}$.
For LIGs, maximizing the empirical proportion of equilibria in Eq.~\eqref {MEPEGame} is equivalent to solving the loss minimization problem:
\begin{equation} \label{MLoGame}
\min_{\W,\b}{{\rm\ }\frac{1}{m}\sum_l{\max_i{\ell(x_i\si{l}(\t{\w{i}}\x_\minus{i}\si{l}-b_i))}}} \; .
\end{equation}

We can further relax this problem by introducing convex upper bounds of the 0/1 loss.
Note that the use of convex losses also avoids the trivial solution of Eq.~\eqref {MLoGame}, i.e., $\W=\zero,\b=\zero$ (which obtains the lowest log-likelihood as discussed in Remark \ref{RemWorst}).
Intuitively speaking, note that minimizing the logistic loss $\ell(z)=\log(1+\exp{-z})$ will make $z \rightarrow +\infty$, while minimizing the hinge loss $\ell(z)=\max{(0,1-z)}$ will make $z \rightarrow 1$ unlike the 0/1 loss $\ell(z)=\iverson{z<0}$ that only requires $z=0$ in order to be minimized.
In what follows, we develop four efficient methods for solving Eq.~\eqref {MLoGame} under specific choices of loss functions, i.e., hinge and logistic.

In our implementation, we add an $\ell_1$-norm regularizer $\rho\|\W\|_1$ where $\rho>0$ to all the minimization problems.
The $\ell_1$-norm regularizer encourages sparseness and attempts to lower the generalization error by controlling over-fitting.

\subsubsection{Independent Support Vector Machines and Logistic Regression}

We can relax the loss minimization problem in Eq.~\eqref {MLoGame} by using the loose bound $\max_i{\ell(z_i)} \leq \sum_i{\ell(z_i)}$.
This relaxation simplifies the original problem into several independent problems.
For each player $i$, we train the weights $(\w{i},b_i)$ in order to predict independent (disjoint) actions.
This leads to \emph{1-norm SVMs} of~\citet{Bradley98,Zhu03} and $\ell_1$-regularized logistic regression.
We solve the latter with the \emph{$\ell_1$-projection method} of~\citet{Schmidt07b}.
While the training is independent, our goal is not the prediction for independent players but the characterization of joint actions.
The use of these well known techniques in our context is novel, since we interpret the output of SVMs and logistic regression as the parameters of an LIG.
Therefore, we use the parameters to measure empirical and true proportion of equilibria, KL divergence and log-likelihood in our probabilistic model.

\subsubsection{Simultaneous Support Vector Machines}

While converting the loss minimization problem in Eq.~\eqref {MLoGame} by using loose bounds allow to obtain several independent problems with small number of variables, a second reasonable strategy would be to use tighter bounds at the expense of obtaining a single optimization problem with a higher number of variables.

For the hinge loss $\ell(z)=\max{(0,1-z)}$, we have $\max_i{\ell(z_i)} = \max{(0,1-z_1,\dots,1-z_n)}$ and the loss minimization problem in Eq.~\eqref {MLoGame} becomes the following primal linear program:
\begin{equation} \label{SVMPrimal}
\begin{array}{@{}l@{}}
  \displaystyle{\min_{\W,\b,\greekbf{\xi}}{{\rm\ }\frac{1}{m}\sum_l{\xi_l} + \rho\|\W\|_1}} \\
  \vspace{-0.12in} \\
  \begin{array}{@{}l@{\hspace{0.08in}}l@{}l@{\hspace{0.08in}}l@{}}
    {\rm s.t.} & (\forall l,i){\rm\ }x_i\si{l}(\t{\w{i}}\x_\minus{i}\si{l}-b_i) \geq 1 - \xi_l {\rm\ \ , \ \ } & (\forall l){\rm\ }\xi_l \geq 0 \; ,
  \end{array}
\end{array}
\end{equation}
\noindent where $\rho>0$.

Note that Eq.~\eqref {SVMPrimal} is equivalent to a linear program since we can set $\W = \W^+ - \W^-$, $\|\W\|_1 = \sum_{ij}{w_{ij}^+ + w_{ij}^-}$ and add the constraints $\W^+ \geq \zero$ and $\W^- \geq \zero$.
We follow the regular SVM derivation by adding slack variables $\xi_l$ for each sample $l$.
This problem is a generalization of \emph{1-norm SVMs} of~\citet{Bradley98,Zhu03}.

By Lagrangian duality, the dual of the problem in Eq.~\eqref {SVMPrimal} is the following linear program:
\begin{equation} \label{SVMDual}
\begin{array}{@{}l@{}}
  \displaystyle{\max_\greekbf{\alpha}{{\rm\ }\sum_{li}\alpha_{li}}} \\
  \vspace{-0.14in} \\
  \begin{array}{@{}l@{\hspace{0.08in}}l@{}l@{\hspace{0.08in}}l@{}}
    {\rm s.t.} & (\forall i){\rm\ }\|\sum_l{\alpha_{li}x_i\si{l}\x_\minus{i}\si{l}}\|_\infty \leq \rho {\rm\ \ , \ \ } & (\forall l,i){\rm\ }\alpha_{li} \geq 0 \; \; ,\\
     & (\forall i){\rm\ }\sum_l{\alpha_{li}x_i\si{l}} = 0 {\rm\ \ , \ \ } & (\forall l){\rm\ }\sum_i{\alpha_{li}} \leq \frac{1}{m} \; .
  \end{array}
\end{array}
\end{equation}
Furthermore, strong duality holds in this case.
Note that Eq.~\eqref {SVMDual} is equivalent to a linear program since we can transform the constraint $\|\mathbf{c}\|_\infty \leq \rho$ into $-\rho\one \leq \mathbf{c} \leq \rho\one$.

\subsubsection{Simultaneous Logistic Regression}

For the logistic loss $\ell(z)=\log(1+\exp{-z})$, we could use the non-smooth loss $\max_i{\ell(z_i)}$ directly.
Instead, we chose a smooth upper bound,
i.e., $\log(1+\sum_i{\exp{-z_i}})$. The following discussion and
technical lemma provides the reason behind our us of 
this \emph{simultaneous logistic loss}.

Given that any loss $\ell(z)$ is a decreasing function, the following identity holds $\max_i{\ell(z_i)} = \ell(\min_i{z_i})$.
Hence, we can either upper-bound the $\max$ function by the ${\rm logsumexp}$ function or lower-bound the $\min$ function by a negative ${\rm logsumexp}$.
We chose the latter option for the logistic loss for the following reasons:
Claim i of the following technical lemma shows that lower-bounding $\min$ generates a loss that is strictly less than upper-bounding $\max$.
Claim ii shows that lower-bounding $\min$ generates a loss that is strictly less than independently penalizing each player.
Claim iii shows that there are some cases in which upper-bounding $\max$ generates a loss that is strictly greater than independently penalizing each player.
\begin{lemma}
For the logistic loss $\ell(z)=\log(1+\exp{-z})$ and a set of $n>1$ numbers $\{z_1,\dots,z_n\}$:
\begin{equation*}
\begin{array}{@{}l@{\hspace{0.05in}}l@{}}
{\rm i.} & (\forall z_1,\dots,z_n){\rm\ }\max_i{\ell(z_i)} \leq \ell\left(-\log\sum_i{\exp{-z_i}}\right) < \log\sum_i{\exp{\ell(z_i)}} \leq \max_i{\ell(z_i)}+\log n \; ,\\
{\rm ii.} & (\forall z_1,\dots,z_n){\rm\ }\ell\left(-\log\sum_i{\exp{-z_i}}\right) < \sum_i{\ell(z_i)} \; ,\\
{\rm iii.} & (\exists z_1,\dots,z_n){\rm\ }\log\sum_i{\exp{\ell(z_i)}} > \sum_i{\ell(z_i)} \; .
\end{array}
\end{equation*}
\end{lemma}
\begin{proof}
Given a set of numbers $\{a_1,\dots,a_n\}$, the $\max$ function is bounded by the ${\rm logsumexp}$ function by $\max_i{a_i} \leq \log\sum_i{\exp{a_i}} \leq \max_i{a_i}+\log n$~\citep{Boyd06}. Equivalently, the $\min$ function is bounded by $\min_i{a_i}-\log n \leq -\log\sum_i{\exp{-a_i}} \leq \min_i{a_i}$.

These identities allow us to prove two inequalities in Claim i, i.e., $\max_i{\ell(z_i)} = \ell(\min_i{z_i}) \leq \ell\left(-\log\sum_i{\exp{-z_i}}\right)$ and $\log\sum_i{\exp{\ell(z_i)}} \leq \max_i{\ell(z_i)}+\log n$. To prove the remaining inequality $\ell\left(-\log\sum_i{\exp{-z_i}}\right) < \log\sum_i{\exp{\ell(z_i)}}$, note that for the logistic loss $\ell\left(-\log\sum_i{\exp{-z_i}}\right) = \log(1+\sum_i{\exp{-z_i}})$ and $\log\sum_i{\exp{\ell(z_i)}} = \log(n+\sum_i{\exp{-z_i}})$. Since $n>1$, strict inequality holds.

To prove Claim ii, we need to show that $\ell\left(-\log\sum_i{\exp{-z_i}}\right) = \log(1+\sum_i{\exp{-z_i}}) < \sum_i{\ell(z_i)} = \sum_i{\log(1+\exp{-z_i})}$. This is equivalent to $1+\sum_i{\exp{-z_i}} < \prod_i{(1+\exp{-z_i})} = \sum_{\mathbf{c}\in \{0,1\}^n}{\exp{-\t{\mathbf{c}}\mathbf{z}}} = 1+\sum_i{\exp{-z_i}}+\sum_{\mathbf{c}\in \{0,1\}^n,\t{\one}\mathbf{c}>1}{\exp{-\t{\mathbf{c}}\mathbf{z}}}$. Finally, we have $\sum_{\mathbf{c}\in \{0,1\}^n,\t{\one}\mathbf{c}>1}{\exp{-\t{\mathbf{c}}\mathbf{z}}} > 0$ because the exponential function is strictly positive.

To prove Claim iii, it suffices to find set of numbers $\{z_1,\dots,z_n\}$ for which $\log\sum_i{\exp{\ell(z_i)}} = \log(n+\sum_i{\exp{-z_i}}) > \sum_i{\ell(z_i)} = \sum_i{\log(1+\exp{-z_i})}$. This is equivalent to $n+\sum_i{\exp{-z_i}} > \prod_i{(1+\exp{-z_i})}$. By setting $(\forall i){\rm\ }z_i=\log n$, we reduce the claim we want to prove to $n+1 > (1+\frac{1}{n})^n$. Strict inequality holds for $n>1$. Furthermore, note that $\lim_{n \rightarrow +\infty}{(1+\frac{1}{n})^n}=e$.
\qedhere
\end{proof}

Returning to our simultaneous logistic regression formulation, the loss minimization problem in Eq.~\eqref {MLoGame} becomes
\begin{equation} \label{SimulLogReg}
\min_{\W,\b}{{\rm\ }\frac{1}{m}\sum_l{\begin{array}{@{}l@{}}
\log(1+\sum_i{\exp{-x_i\si{l}(\t{\w{i}}\x_\minus{i}\si{l}-b_i)}})
\end{array}} + \rho\|\W\|_1} \; ,
\end{equation}
\noindent where $\rho>0$.

In our implementation, we use the \emph{$\ell_1$-projection method} of~\citet{Schmidt07b} for optimizing Eq.~\eqref {SimulLogReg}.
This method performs a \emph{limited-memory Broyden-Fletcher-Goldfarb-Shanno} (L-BFGS) step in an expanded model (i.e., $\W = \W^+ - \W^-$, $\|\W\|_1 = \sum_{ij}{w_{ij}^+ + w_{ij}^-}$) followed by a projection onto the non-negative orthant to enforce $\W^+ \geq \zero$ and $\W^- \geq \zero$.

\section{On the True Proportion of Equilibria} \label{SecSmallTPE}

In this section, we justify the use of convex loss minimization for learning the structure and parameters of LIGs.
We define \emph{absolute indifference} of players and show that our convex loss minimization approach produces games in which all players are non-absolutely-indifferent.
We then provide a bound of the true proportion of equilibria with high probability.
Our bound assumes independence of weight vectors among players, and applies to a large family of distributions of weight vectors.
Furthermore, we do not assume any connectivity properties of the underlying graph.

Parallel to our analysis, \citet{Daskalakis11} analyzed a different setting: random games which structure is drawn from the Erd\H{o}s-R\'{e}nyi model (i.e., each edge is present independently with the same probability $p$) and utility functions which are random tables.
The analysis in \citet{Daskalakis11}, while more general than ours (which only focus on LIGs), it is at the same time more restricted since it assumes either the Erd\H{o}s-R\'{e}nyi model for random structures or connectivity properties for deterministic structures.

\subsection{Convex Loss Minimization Produces Non-Absolutely-Indifferent Players}

First, we define the notion of \emph{absolute indifference} of players.
Our goal in this subsection is to show that our proposed convex loss algorithms produce LIGs in which all players are non-absolutely-indifferent and therefore every player defines constraints to the true proportion of equilibria.
\begin{definition} \label{DefAbsIndifferent}
Given an LIG $\G=(\W,\b)$, we say a player $i$ is \emph{absolutely indifferent} if and only if $(\w{i},b_i)=\zero$, and \emph{non-absolutely-indifferent} if and only if $(\w{i},b_i) \neq \zero$.
\end{definition}

Next, we concentrate on the first ingredient for our bound of the true proportion of equilibria.
We show that independent and simultaneous SVM and logistic regression produce games in which all players are non-absolutely-indifferent except for some ``degenerate'' cases.
The following lemma applies to independent SVMs for $c\si{l}=0$ and simultaneous SVMs for $c\si{l}=\max(0,\max_{j\neq i}{(1-x_j\si{l}(\t{\w{i}}\x_\minus{i}\si{l}-b_i))})$.
\begin{lemma} \label{LemNonIndSVM}
Given $(\forall l){\rm\ }c\si{l}\geq 0$, the minimization of the hinge training loss $\ellh(\w{i},b_i)=\frac{1}{m}\sum_l{\max(c\si{l},1-x_i\si{l}(\t{\w{i}}\x_\minus{i}\si{l}-b_i))}$ guarantees non-absolutely-indifference of player $i$ except for some ``degenerate'' cases, i.e., the optimal solution $(\w{i}^*,b_i^*)=\zero$ if and only if $(\forall j\neq i){\rm\ }\sum_l{\iverson{x_i\si{l}x_j\si{l}\hns\hns=\hns1}u\si{l}} = \sum_l{\iverson{x_i\si{l}x_j\si{l}\hns\hns=\hns-1}u\si{l}}$ and $\sum_l{\iverson{x_i\si{l}\hns\hns=\hns1}u\si{l}} = \sum_l{\iverson{x_i\si{l}\hns\hns=\hns-1}u\si{l}}$ where $u\si{l}$ is defined as $c\si{l}>1 \Leftrightarrow u\si{l}=0$, $c\si{l}<1 \Leftrightarrow u\si{l}=1$ and $c\si{l}=1 \Leftrightarrow u\si{l}\in [0;1]$.
\end{lemma}
\begin{proof}
Let $f_i(\x_\minus{i}) \equiv \t{\w{i}}\x_\minus{i}-b_i$. By noting that $\max(\alpha,\beta) = \max_{0\leq u\leq 1}{(\alpha+u(\beta-\alpha))}$, we can rewrite $\ellh(\w{i},b_i)=\frac{1}{m}\sum_l{\max_{0\leq u\si{l}\leq 1}{(c\si{l} + u\si{l}(1-x_i\si{l} f_i(\x_\minus{i}\si{l}) - c\si{l}))}}$.

Note that $\ellh$ has the minimizer $(\w{i}^*,b_i^*)=\zero$ if and only if $\zero$ belongs to the subdifferential set of the non-smooth function $\ellh$ at $(\w{i},b_i)=\zero$. In order to maximize $\ellh$, we have $c\si{l}>1-x_i\si{l} f_i(\x_\minus{i}\si{l}) \Leftrightarrow u\si{l}=0$, $c\si{l}<1-x_i\si{l} f_i(\x_\minus{i}\si{l}) \Leftrightarrow u\si{l}=1$ and $c\si{l}=1-x_i\si{l} f_i(\x_\minus{i}\si{l}) \Leftrightarrow u\si{l}\in [0;1]$. The previous rules simplify at the solution under analysis, since $(\w{i},b_i)=\zero \Rightarrow f_i(\x_\minus{i}\si{l})=0$.

Let $g_j(\w{i},b_i) \equiv \frac{\partial\ellh}{\partial w_{ij}}(\w{i},b_i)$ and $h(\w{i},b_i) \equiv \frac{\partial\ellh}{\partial b_i}(\w{i},b_i)$. By making $(\forall j\neq i){\rm\ }0\in g_j(\zero,0)$ and $0\in h(\zero,0)$, we get $(\forall j\neq i){\rm\ }\sum_l{x_i\si{l}x_j\si{l}u\si{l}}=0$ and $\sum_l{x_i\si{l}u\si{l}}=0$. Finally, by noting that $x_i\si{l}\in \{-1,1\}$, we prove our claim.
\qedhere
\end{proof}

\begin{remark}
Note that for independent SVMs, the ``degenerate'' cases in Lemma \ref{LemNonIndSVM} simplify to $(\forall j\neq i){\rm\ }\sum_l{\iverson{x_i\si{l}x_j\si{l}=1}} = \frac{m}{2}$ and $\sum_l{\iverson{x_i\si{l}=1}} = \frac{m}{2}$.
\end{remark}

The following lemma applies to independent logistic regression for $c\si{l}=0$ and simultaneous logistic regression for $c\si{l}=\sum_{j\neq i}\exp{-x_j\si{l}(\t{\w{i}}\x_\minus{i}\si{l}-b_i)}$.
\begin{lemma} \label{LemNonIndLogReg}
Given $(\forall l){\rm\ }c\si{l}\geq 0$, the minimization of the logistic training loss $\ellh(\w{i},b_i)=\frac{1}{m}\sum_l{\log(c\si{l}+1+\exp{-x_i\si{l}(\t{\w{i}}\x_\minus{i}\si{l}-b_i)})}$ guarantees non-absolutely-indifference of player $i$ except for some ``degenerate'' cases, i.e., the optimal solution $(\w{i}^*,b_i^*)=\zero$ if and only if $(\forall j\neq i){\rm\ }\sum_l\frac{\iverson{x_i\si{l}x_j\si{l}=1}}{c\si{l}+2} = \sum_l\frac{\iverson{x_i\si{l}x_j\si{l}=-1}}{c\si{l}+2}$ and $\sum_l\frac{\iverson{x_i\si{l}=1}}{c\si{l}+2} = \sum_l\frac{\iverson{x_i\si{l}=-1}}{c\si{l}+2}$.
\end{lemma}
\begin{proof}
Note that $\ellh$ has the minimizer $(\w{i}^*,b_i^*)=\zero$ if and only if the gradient of the smooth function $\ellh$ is $\zero$ at $(\w{i},b_i)=\zero$. Let $g_j(\w{i},b_i) \equiv \frac{\partial\ellh}{\partial w_{ij}}(\w{i},b_i)$ and $h(\w{i},b_i) \equiv \frac{\partial\ellh}{\partial b_i}(\w{i},b_i)$. By making $(\forall j\neq i){\rm\ }g_j(\zero,0)=0$ and $h(\zero,0)=0$, we get $(\forall j\neq i){\rm\ }\sum_l\frac{x_i\si{l}x_j\si{l}}{c\si{l}+2}=0$ and $\sum_l\frac{x_i\si{l}}{c\si{l}+2}=0$. Finally, by noting that $x_i\si{l}\in \{-1,1\}$, we prove our claim.
\qedhere
\end{proof}

\begin{remark}
Note that for independent logistic regression, the ``degenerate'' cases in Lemma \ref{LemNonIndLogReg} simplify to $(\forall j\neq i){\rm\ }\sum_l{\iverson{x_i\si{l}x_j\si{l}=1}} = \frac{m}{2}$ and $\sum_l{\iverson{x_i\si{l}=1}} = \frac{m}{2}$.
\end{remark}

Based on these results, after termination of our proposed algorithms, we fix cases in which the optimal solution $(\w{i}^*,b_i^*)=\zero$ by setting $b^*_i=1$ if the action of player $i$ was mostly $-1$ or $b^*_i=-1$ otherwise.
We point out to the careful reader that we did not include the $\ell_1$-regularization term in the above proofs since the subdifferential of $\rho\|\w{i}\|_1$ vanishes at $\w{i}=0$, and therefore our proofs still hold.

\subsection{Bounding the True Proportion of Equilibria}

In what follows, we concentrate on the second ingredient for our bound of the true proportion of equilibria.
We show a bound for a single \emph{non-absolutely-indifferent} player and a fixed joint-action $\x$, that interestingly does not depend on the specific joint-action $\x$.
This is a key ingredient for bounding the true proportion of equilibria in our main theorem.
\begin{lemma} \label{LemSingleBound}
Given an LIG $\G=(\W,\b)$ with non-absolutely-indifferent player $i$, assume that $(\w{i},b_i)$ is a random vector drawn from a distribution $\PD_i$.
If for all $\x \in \{-1,+1\}^n$,
${\P_{\PD_i}[x_i (\t{\w{i}}\x_\minus{i}-b_i) = 0]=0}$, then
\begin{eqnarray*}
\text{i. for all } \x\text{, }\P_{\PD_i}[x_i
(\t{\w{i}}\x_\minus{i}-b_i) \geq 0] =  \P_{\PD_i}[x_i
(\t{\w{i}}\x_\minus{i}-b_i) \leq 0]
                     \\
\text{if and only if, } \text{for
  all } \x\text{, }\P_{\PD_i}[x_i (\t{\w{i}}\x_\minus{i}-b_i) \geq 0] =
1/2 \; .
\end{eqnarray*}
If $\PD_i$ is a uniform distribution of support $\{-1,+1\}^n$, then
\begin{equation*}
\text{ii. for all }\x\text{, }\P_{\PD_i}[x_i (\t{\w{i}}\x_\minus{i}-b_i)\geq 0] \in [1/2,{\rm\ }3/4] \; .
\end{equation*}
\end{lemma}
\begin{proof}
Claim i follows immediately from a simple condition we
obtain from the normalization axiom of probability and the hypothesis of the claim:
i.e., for all $\x \in \{-1,+1\}^n$, $\P_{\PD_i}[x_i
(\t{\w{i}}\x_\minus{i}-b_i) \geq 0] + \P_{\PD_i}[x_i
(\t{\w{i}}\x_\minus{i}-b_i) \leq 0] = 1$.

To prove Claim ii, first let $\v \equiv (\w{i},b_i)$ and $\y \equiv (x_i \x_\minus{i},-x_i) \in \{-1,+1\}^n$.
Note that $x_i (\t{\w{i}}\x_\minus{i}-b_i) = \t{\v}\y$.
Then, let $f_1(\v_\minus{1},\y) \equiv \t{\v_\minus{1}}\y_\minus{1}+y_1$.
Note that $(v_1,v_1\v_\minus{1})$ spans all possible vectors in $\{-1,+1\}^n$.
Because $\PD_i$ is a uniform distribution of support $\{-1,+1\}^n$, we have:
\begin{align*}
\P_{\PD_i}[\t{\v}\y \geq 0] & = \textstyle{ \frac{1}{2^n} \sum_\v{\iverson{\t{\v}\y \geq 0}} } \\
 & = \textstyle{ \frac{1}{2^n} \sum_\v{\iverson{v_1 f_1(\v_\minus{1},\y)\geq 0}} } \\
 & = \textstyle{ \frac{1}{2^n} \sum_\v{\left( \iverson{v_1=+1}\iverson{f_1(\v_\minus{1},\y)\geq 0} + \iverson{v_1=-1}\iverson{f_1(\v_\minus{1},\y)\leq 0} \right)} } \\ 
 & = \textstyle{ \frac{1}{2^n} \sum_{\v_\minus{1}}{\left( \iverson{f_1(\v_\minus{1},\y)\geq 0} + \iverson{f_1(\v_\minus{1},\y)\leq 0} \right)} } \\
 & = \textstyle{ \frac{2^{n-1}}{2^n} + \frac{1}{2^n} \sum_{\v_\minus{1}}{\iverson{f_1(\v_\minus{1},\y)=0}} } \\
 & = \textstyle{ 1/2 + \frac{1}{2^n} \alpha(\y) }
\end{align*}
\noindent where $\alpha(\y) \equiv \sum_{\v_\minus{1}}{\iverson{f_1(\v_\minus{1},\y)=0}} = \sum_{\v_\minus{1}}{\iverson{\t{\v_\minus{1}}\y_\minus{1}+y_1=0}}$.
Note that $\alpha(\y)\geq 0$ and thus, $\P_{\PD_i}[\t{\v}\y \geq 0]\geq 1/2$.
Geometrically speaking, $\alpha(\y)$ is the number of vertices of the $(n-1)$-dimensional hypercube that are covered by the hyperplane with normal $\y_\minus{1}$ and bias $y_1$. Recall that $\y\neq \zero$ since $\y\in \{-1,+1\}^n$. By relaxing this fact, as noted in~\citet{Aichholzer96} a hyperplane with $n-2$ zeros on $\y_\minus{1}$ (i.e., a \emph{$(n-2)$-parallel hyperplane}) covers exactly half of the $2^{n-1}$ vertices, the maximum possible. Therefore, $\P_{\PD_i}[\t{\v}\y \geq 0] = 1/2 + \frac{1}{2^n} \alpha(\y) \leq 1/2 +\frac{2^{n-2}}{2^n} = 3/4$.
\qedhere
\end{proof}

\begin{remark}
It is important to note that under the conditions of
Lemma \ref{LemSingleBound}, in a measure-theoretic sense, for almost all
vectors $(\w{i},b_i)$ in the surface of the
hypersphere in $n$-dimensions (i.e., except for a set of
Lebesgue-measure zero), we have that, $x_i (\t{\w{i}}\x_\minus{i}-b_i)
\neq 0$ for all $\x \in \{-1,+1\}^n$. Hence, the hypothesis stated for
Claim i of Lemma \ref{LemSingleBound} holds for
almost all probability measures $\PD_i$ (i.e., except for a set
of probability measures, over the surface of the hypersphere in
$n$-dimensions, with Lebesgue-measure zero). Note that Claim
ii essentially states that we can still upper bound, for all $\x \in
\{-1,+1\}^n$, the probability that such $\x$ is a PSNE of a random LIG even
if we draw the weights and threshold parameters from a $\PD_i$
belonging to such sets of Lebesgue-measure zero.
\end{remark}
\begin{remark}
Note that any distribution that has zero mean and that depends on some norm of $(\w{i},b_i)$ fulfills the requirements for Claim i of Lemma \ref{LemSingleBound}.
This includes, for instance, the multivariate normal distribution with arbitrary covariance which is related to the Bhattacharyya norm.
Additionally, any distribution in which each entry of the vector $(\w{i},b_i)$ is independent and symmetric also fulfills those requirements.
This includes, for instance, the Laplace and uniform distributions.
Furthermore, note that distributions with support on non-empty subsets of entries of $(\w{i},b_i)$, as well as mixtures of the above cases are also allowed.
This includes, for instance, sparse graphs.
\end{remark}

Next, we present our bound for the true proportion of equilibria of games in which all players are non-absolutely-indifferent.
\begin{theorem} \label{ThmExpectedNE}
Assume that all players are non-absolutely-indifferent and that the rows of an LIG $\G=(\W,\b)$ are independent (but not necessarily identically distributed) random vectors, i.e., for every player $i$, $(\w{i},b_i)$ is independently drawn from an arbitrary distribution $\PD_i$.
If for all $i$ and $\x$, ${\P_{\PD_i}[x_i
  (\t{\w{i}}\x_\minus{i}-b_i)\geq 0] \leq \kappa}$ for $1/2 \leq\kappa<1$, then the expected true proportion of equilibria is bounded as 
\begin{equation*}
\E_{\PD_1,\dots,\PD_n}[\pi(\G)] \leq \kappa^n \; .
\end{equation*}
\noindent Furthermore, the following high probability statement 
\begin{equation*}
\P_{\PD_1,\dots,\PD_n}[\pi(\G) \leq \textstyle{ \frac{\kappa^n}{\delta} }] \geq 1-\delta 
\end{equation*}
holds.
\end{theorem}
\begin{proof}
Let $f_i(\w{i},b_i,\x) \equiv \iverson{x_i (\t{\w{i}}\x_\minus{i}-b_i)\geq 0}$ and $\PD \equiv \{\PD_1,\dots,\PD_n\}$.
By Eq.~\eqref {PiDef}, $\E_\PD[\pi(\G)] = \frac{1}{2^n}\sum_\x{\E_\PD[\prod_i{f_i(\w{i},b_i,\x)}]}$.
For any $\x$, $f_1(\w{1},b_1,\x),\dots,f_n(\w{n},b_n,\x)$ are independent since $(\w{1},b_1),\dots,(\w{n},b_n)$ are independently distributed.
Thus, $\E_\PD[\pi(\G)] = \frac{1}{2^n}\sum_\x{\prod_i{\E_{\PD_i}[f_i(\w{i},b_i,\x)]}}$.
Since for all $i$ and $\x$, $\E_{\PD_i}[f_i(\w{i},b_i,\x)] = \P_{\PD_i}[x_i (\t{\w{i}}\x_\minus{i}-b_i)\geq 0] \leq \kappa$, we have $\E_\PD[\pi(\G)] \leq \kappa^n$.

By Markov's inequality, given that $\pi(\G) \geq 0$, we have $\P_\PD[\pi(\G) \geq c] \leq \frac{\E_\PD[\pi(\G)]}{c} \leq \frac{\kappa^n}{c}$.
For $c = \frac{\kappa^n}{\delta} \Rightarrow \P_\PD[\pi(\G) \geq \frac{\kappa^n}{\delta}] \leq \delta \Rightarrow \P_\PD[\pi(\G) \leq \frac{\kappa^n}{\delta}] \geq 1-\delta$.
\qedhere
\end{proof}

\begin{remark}
Under the same assumptions of Theorem \ref{ThmExpectedNE}, it is possible to prove that with probability at least $1-\delta$ we have $\pi(\G) \leq \kappa^n + \sqrt{\frac{1}{2}\log\frac{1}{\delta}}$ by using Hoeffding's lemma.
We point out that such a bound is not better than the Markov's bound derived above.
\end{remark}

\section{Experimental Results} \label{SecResults}

For learning LIGs we used our convex loss methods: independent and simultaneous SVM and logistic regression (See Section~\ref{SubSecCLM}).
Additionally, we used the (super-exponential) exhaustive search method (See Section~\ref{SubSecExhaustive}) only for $n \leq 4$.
As a baseline, we used the (NP-hard) sigmoidal maximum likelihood only for $n \leq 15$ as well as the sigmoidal maximum empirical proportion of equilibria (See Section~\ref{SubSecSigmoidal}).
Regarding the parameters $\alpha$ and $\beta$ our sigmoidal function in Eq.~\eqref {Sigmoid}, we found experimentally that $\alpha = 0.1$ and $\beta = 0.001$ achieved the best results.

\begin{figure}
\begin{center}
\raisebox{0.375in}{\includegraphics[width=0.25\linewidth]{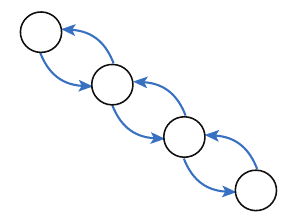}}\includegraphics[width=0.25\linewidth]{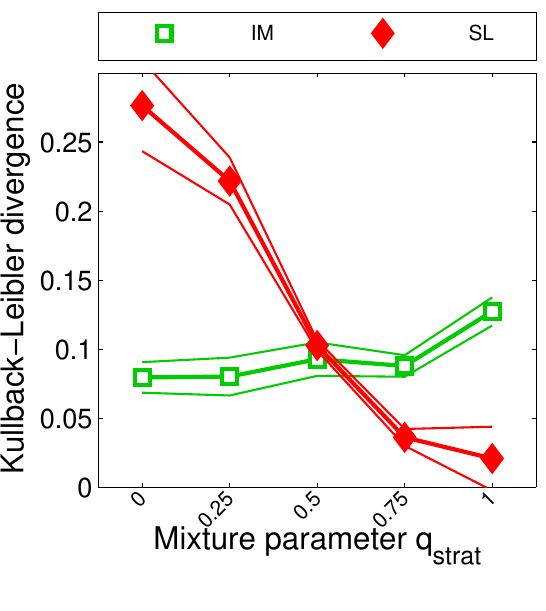}
\end{center}
\vspace{-0.3in}
\captionx{{\bf On the Distinction between Game-Theoretic and Probabilistic
    Models in the Context of ``Probabilistic'' vs. ``Strategic''
    Behavioral Data.} The plot shows the performance of Ising models
(in green) vs. LIGs (in red) when we learn models from each respective
class from data generated by
drawing i.i.d. samples from a mixture model of an Ising model,
$p_{\text{Ising}}$, and our
PSNE-based generative model, $p_{\text{LIG}}$, with mixing parameter
$q_{\text{strat}}$ corresponding to the probability that the sample is
drawn
from the LIG component. Hence, we can view  $q_{\text{strat}}$ as
controlling the proportion of the data that is ``strategic'' in nature. The graph of the Ising model is an
(undirected) chain with 4 variable nodes, while that of the LIG is, as
shown in the left,
also a chain of 4 players with arcs between every consecutive pair of
nodes. The parameters of each mixture component in the ``ground-truth''
mixture model $p_{\text{mix}}(\x) \equiv q_{\text{strat}} \;
p_{\text{LIG}}(\x) + (1-q_{\text{strat}}) \;
p_{\text{Ising}}(\x)$ are the same: node-potential/bias-threshold parameters
are all $0$; weights of all the edges is $+1$. We set the ``signal''
parameter $q$ of our generative model $p_{\text{LIG}}$ to $0.9$. The x-axis of
the plot in the right-hand side above corresponds to the mixture parameter $q_{\text{strat}}$; so that, as we
move from left to right in the plot, more proportion of the data is
``strategic'' in nature: $q_{\text{strat}} = 0$ means the data is ``purely
probabilistic'' while $q_{\text{strat}} = 1$ means it is ``purely
strategic.'' For each value of $q_{\text{strat}} \in
\{0,0.25,0.50,0.75,1\}$, we generated $50$ pairs of data sets from $p_{\text{mix}}$, each of
size $50$, each pair corresponding to a training and a validation
data set, respectively. The learning methods used the validation
data set to estimate their respective $\ell_1$ regularization parameter. The Ising models
learned correspond \emph{exactly} to the optimal penalized
likelihood. We use a simultaneous logistic
regression approach, described in Section~\ref{SecAlgorithm}, to learn
LIGs. In the y-axis of the
plot in the right-hand side is the
average, over the $50$ repetitions, of the \emph{exact} KL-divergence
between the respective learned model and $p_{\text{mix}}(\x)$. We also
include (a linear interpolation of the individual) error bars at 95\% confidence level. The plot clearly shows that the more
``strategic'' the data the better the game-theoretic-based generative
model. We can see that the learned Ising models (1) do
considerably better than the LIG models when the data is purely
probabilistic; and (2)  are more ``robust''
across the spectrum, degrading very gracefully as the data becomes
more strategic in nature; but (3) seem to
need more data to learn when the data comes exclusively from an Ising
model than the LIG model does when the data is purely strategic: The
LIG achieves KL values much closer to $0$ when the data is purely
strategic than the Ising model does when the data is purely
probabilistic.}
\vspace{-0.1in}
\label{fig:datanature}
\end{figure}

For reasons briefly discussed at the end of Section~\ref{sec:probmod}, we have little interest in determining how much worst game-theoretic
models are relative to probabilistic models when applied to data from
purely \emph{probabilistic} processes, without any strategic
component, as we think this to be a futile exercise.
We believe the same is true for
evaluating the quality of a probabilistic graphical model vs. a
game-theoretic model when applied to \emph{strategic behavioral
  data}, resulting from a process defined by game-theoretic concepts
based on the (stable) outcomes of
a game. Nevertheless,
we summarize some experiments in Figure~\ref{fig:datanature} that should help
illustrate the point of discussed at the end of Section~\ref{sec:probmod}.

Still, for scientific curiosity, we compare LIGs to learning Ising models.
Once again, our goal is not to show the superiority of either games or Ising models.
For $n \leq 15$ players, we perform exact $\ell_1$-regularized maximum likelihood estimation by using the FOBOS algorithm~\citep{Duchi09,Duchi09c} and exact gradients of the log-likelihood of the Ising model.
Since the computation of the exact gradient at each step is NP-hard, we used this method only for $n \leq 15$.
For $n > 15$ players, we use the H\"ofling-Tibshirani method~\citep{Hofling09}, which uses a sequence of first-order approximations of the exact log-likelihood.
We also used a two-step algorithm, by first learning the structure by $\ell_1$-regularized logistic regression \citep{Wainwright06} and then using the FOBOS algorithm~\citep{Duchi09,Duchi09c} with belief propagation for gradient approximation.
We did not find a statistically significant difference between the test log-likelihood of both algorithms and therefore we only report the latter.

Our experimental setup is as follows: after learning a model for different values of the regularization parameter $\rho$ in a training set, we select the value of $\rho$ that maximizes the log-likelihood in a validation set, and report statistics in a test set.
For synthetic experiments, we report the Kullback-Leibler (KL) divergence, average precision (one minus the fraction of falsely included equilibria), average recall (one minus the fraction of falsely excluded equilibria) in order to measure the closeness of the recovered models to the ground truth.
For real-world experiments, we report the log-likelihood.
In both synthetic and real-world experiments, we report the number of equilibria and the empirical proportion of equilibria.
Our results are statistically significant, we avoided showing error bars for clarity of presentation since error bars and markers overlapped.

\subsection{Experiments on Synthetic Data}

\begin{figure}
\begin{center}
\raisebox{0.375in}{\includegraphics[width=0.25\linewidth]{synthetic4-truth}}\includegraphics[width=0.25\linewidth]{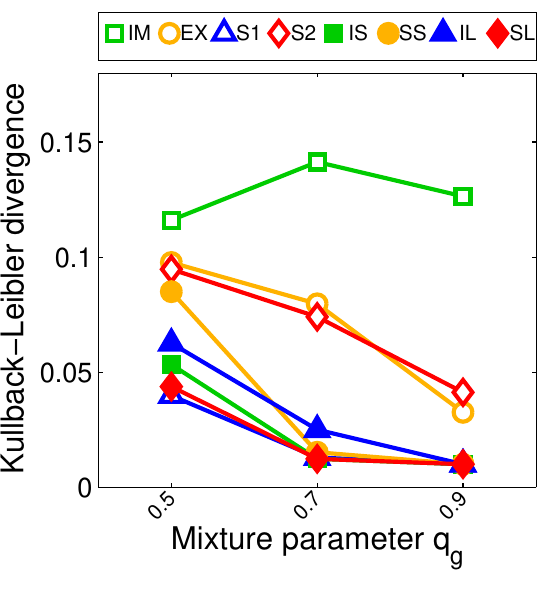}\includegraphics[width=0.25\linewidth]{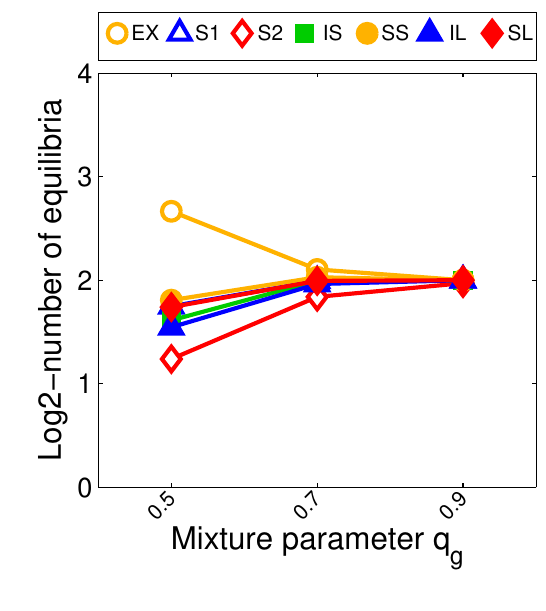}\hspace{-0.75\linewidth}\makebox[0.25\linewidth][c]{\raisebox{0.17in}{\textsf{\scriptsize{First synthetic model}}}}\makebox[0.5\linewidth]{} \\
\includegraphics[width=0.25\linewidth]{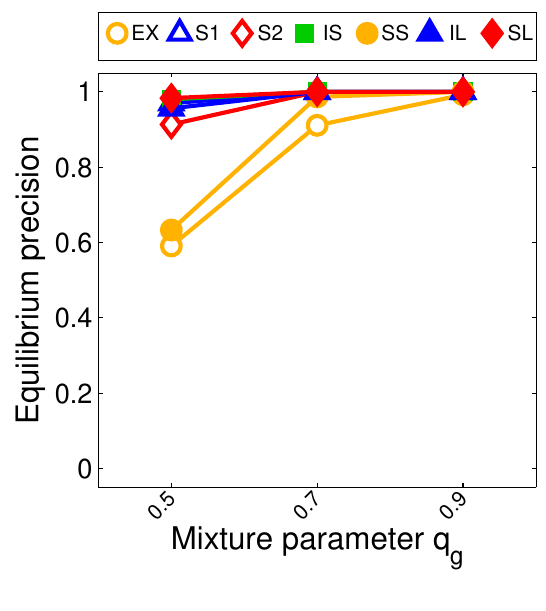}\includegraphics[width=0.25\linewidth]{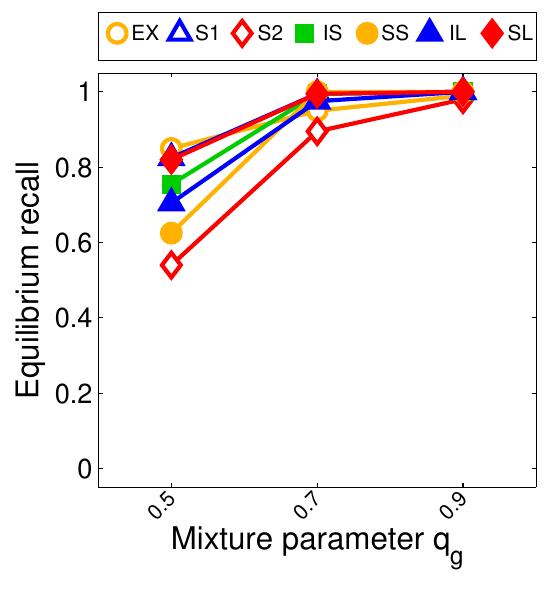}\includegraphics[width=0.25\linewidth]{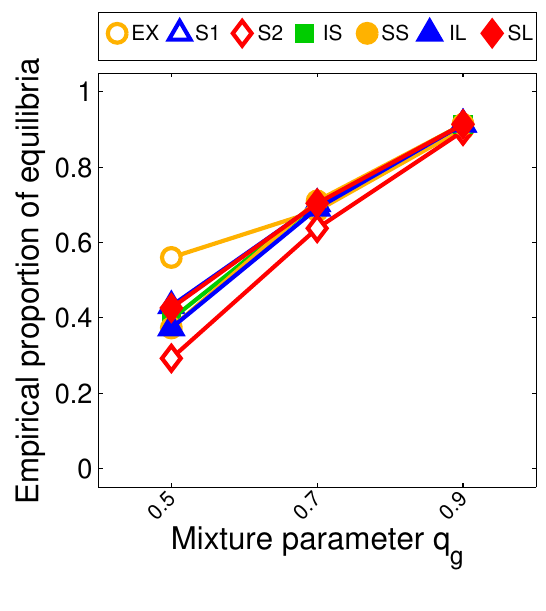} \\
\end{center}
\vspace{-0.3in}
\captionx{{\bf Closeness of the Recovered Models to the Ground-Truth Synthetic Model for Different Mixture Parameters $q_g$.} Our convex loss methods (IS,SS: independent and simultaneous SVM, IL,SL: independent and simultaneous logistic regression) and sigmoidal maximum likelihood (S1) have lower KL than exhaustive search (EX), sigmoidal maximum empirical proportion of equilibria (S2) and Ising models (IM). For all methods, the recovery of equilibria is perfect for $q_g=0.9$ (number of equilibria equal to the ground truth, equilibrium precision and recall equal to 1) and the empirical proportion of equilibria resembles the mixture parameter of the ground truth $q_g$.}
\vspace{-0.1in}
\label{Synthetic4}
\end{figure}

We first test the ability of the proposed methods to recover the
PSNE induced by ground-truth games from data when those games are
LIGs.
We use a small first synthetic model in order to compare with the (super-exponential) exhaustive search method.
The ground-truth model $\G_g=(\W_g,\b_g)$ has $n=4$ players and 4 Nash equilibria (i.e., $\pi(\G_g)$=0.25), $\W_g$ was set according to Figure \ref{Synthetic4} (the weight of each edge was set to $+1$) and $\b_g=\zero$.
The mixture parameter of the ground-truth model $q_g$ was set to 0.5,0.7,0.9.
For each of 50 repetitions, we generated a training, a validation and a test set of 50 samples each.
Figure \ref{Synthetic4} shows that our convex loss methods and sigmoidal maximum likelihood outperform (lower KL) exhaustive search, sigmoidal maximum empirical proportion of equilibria and Ising models.
Note that the exhaustive search method which performs exact maximum
likelihood suffers from over-fitting and consequently does not produce
the lowest KL.
From all convex loss methods, simultaneous logistic regression achieves the lowest KL.
For all methods, the recovery of equilibria is perfect for $q_g=0.9$ (number of equilibria equal to the ground truth, equilibrium precision and recall equal to 1).
Additionally, the empirical proportion of equilibria resembles the
mixture parameter of the ground-truth model $q_g$.

\begin{figure}
\begin{center}
\raisebox{0.375in}{\includegraphics[width=0.25\linewidth]{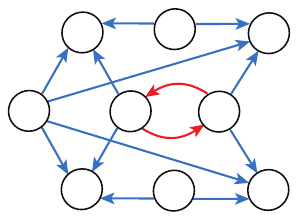}}\includegraphics[width=0.25\linewidth]{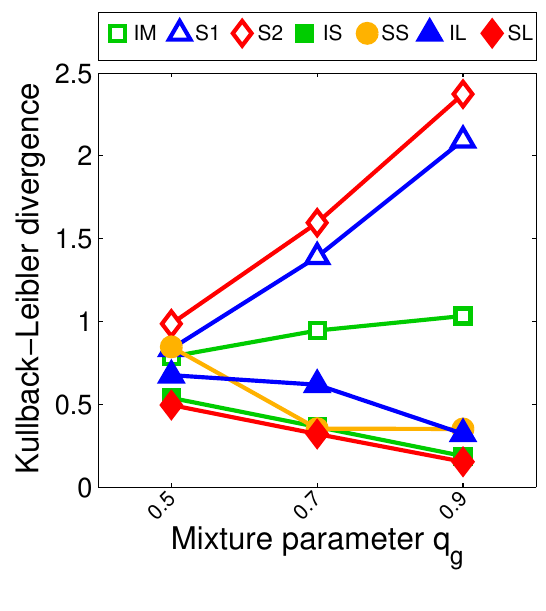}\includegraphics[width=0.25\linewidth]{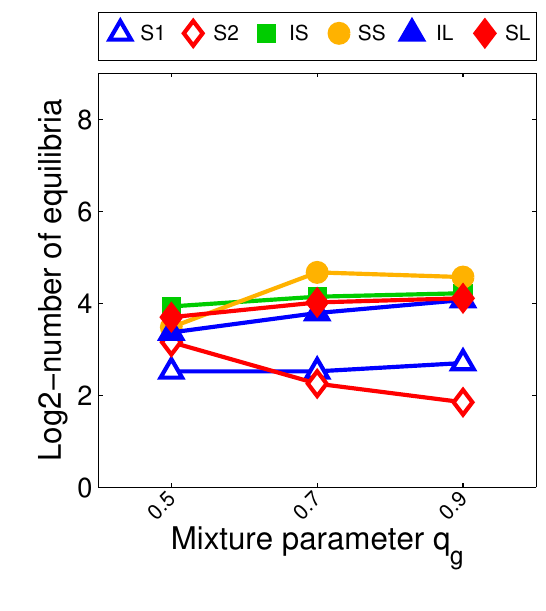}\hspace{-0.75\linewidth}\makebox[0.25\linewidth][c]{\raisebox{0.17in}{\textsf{\scriptsize{Second synthetic model}}}}\makebox[0.5\linewidth]{} \\
\includegraphics[width=0.25\linewidth]{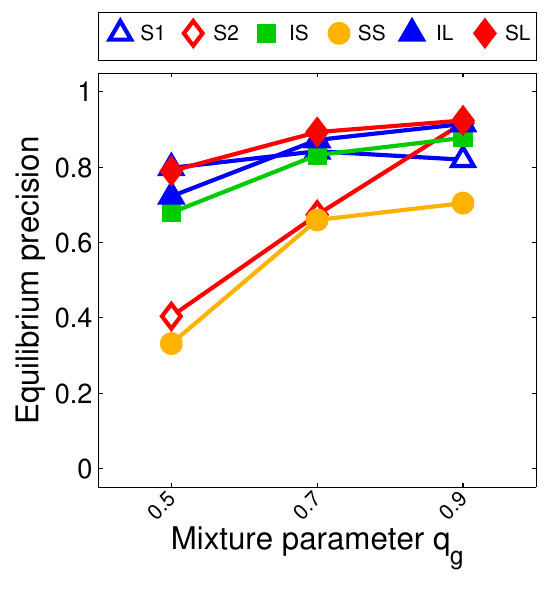}\includegraphics[width=0.25\linewidth]{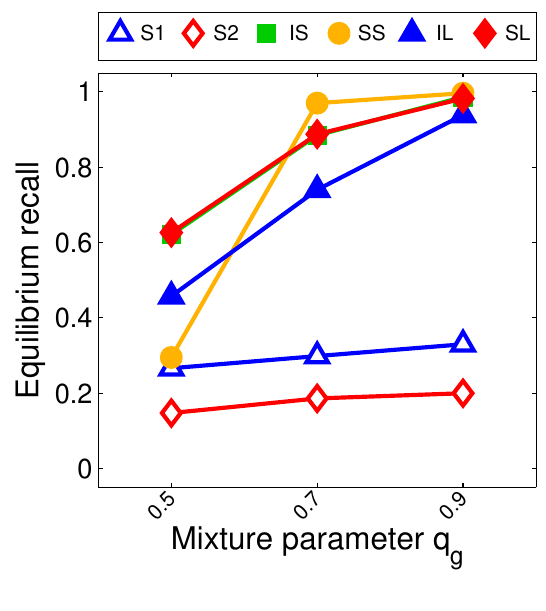}\includegraphics[width=0.25\linewidth]{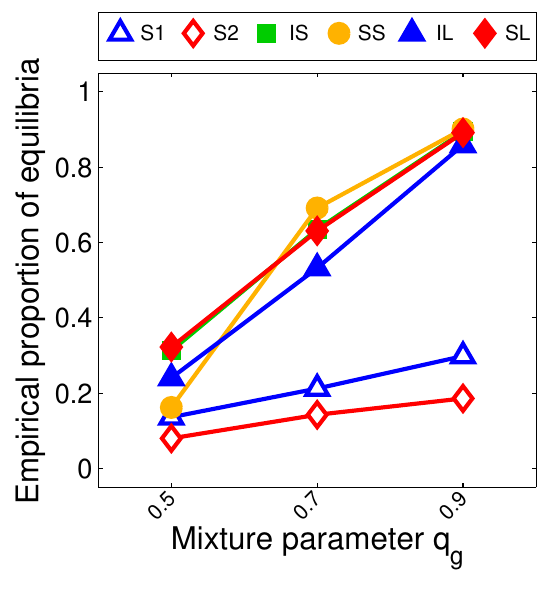} \\
\end{center}
\vspace{-0.3in}
\captionx{Closeness of the recovered models to the ground truth synthetic model for different mixture parameters $q_g$. Our convex loss methods (IS,SS: independent and simultaneous SVM, IL,SL: independent and simultaneous logistic regression) have lower KL than sigmoidal maximum likelihood (S1), sigmoidal maximum empirical proportion of equilibria (S2) and Ising models (IM). For convex loss methods, the equilibrium recovery is better than the remaining methods (number of equilibria equal to the ground truth, higher equilibrium precision and recall) and the empirical proportion of equilibria resembles the mixture parameter of the ground truth $q_g$.}
\vspace{-0.1in}
\label{Synthetic9}
\end{figure}

Next, we use a slightly larger second synthetic model with more complex interactions.
We still keep the model small enough in order to compare with the (NP-hard) sigmoidal maximum likelihood method.
The ground truth model $\G_g=(\W_g,\b_g)$ has $n=9$ players and 16 Nash equilibria (i.e., $\pi(\G_g)$=0.03125), $\W_g$ was set according to Figure \ref{Synthetic9} (the weight of each blue and red edge was set to $+1$ and $-1$ respectively) and $\b_g=\zero$.
The mixture parameter of the ground truth $q_g$ was set to 0.5,0.7,0.9.
For each of 50 repetitions, we generated a training, a validation and a test set of 50 samples each.
Figure \ref{Synthetic9} shows that our convex loss methods outperform (lower KL) sigmoidal methods and Ising models.
From all convex loss methods, simultaneous logistic regression achieves the lowest KL.
For convex loss methods, the equilibrium recovery is better than the remaining methods (number of equilibria equal to the ground truth, higher equilibrium precision and recall).
Additionally, the empirical proportion of equilibria resembles the mixture parameter of the ground truth $q_g$.

\begin{figure}
\begin{center}
\includegraphics[width=0.25\linewidth]{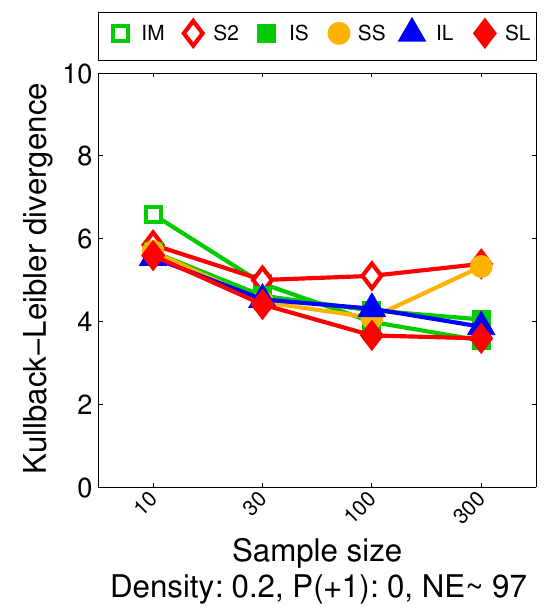}\includegraphics[width=0.25\linewidth]{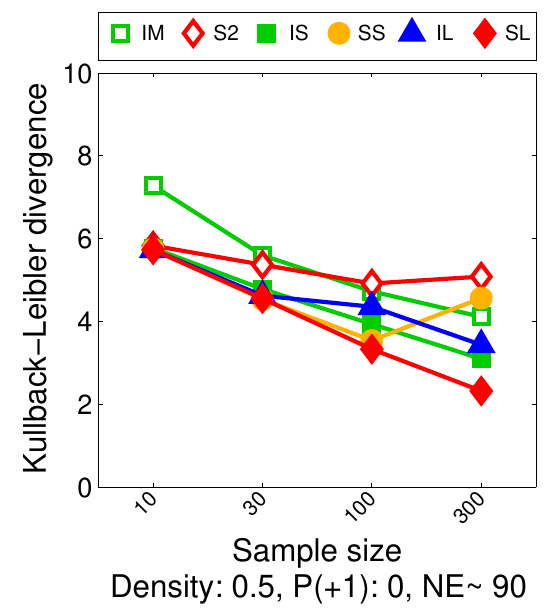}\includegraphics[width=0.25\linewidth]{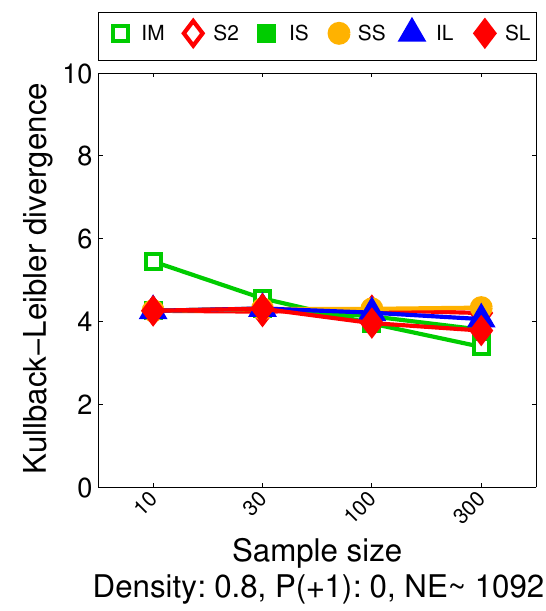} \\
\includegraphics[width=0.25\linewidth]{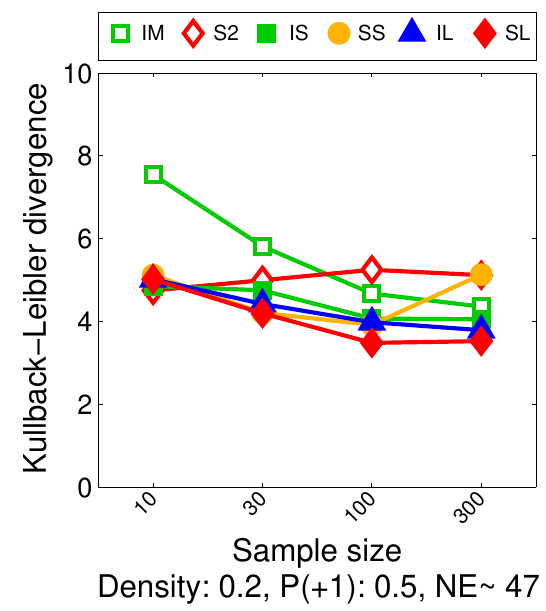}\includegraphics[width=0.25\linewidth]{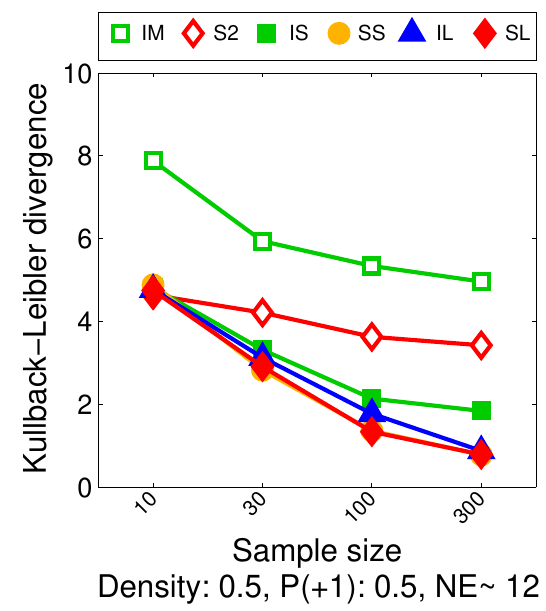}\includegraphics[width=0.25\linewidth]{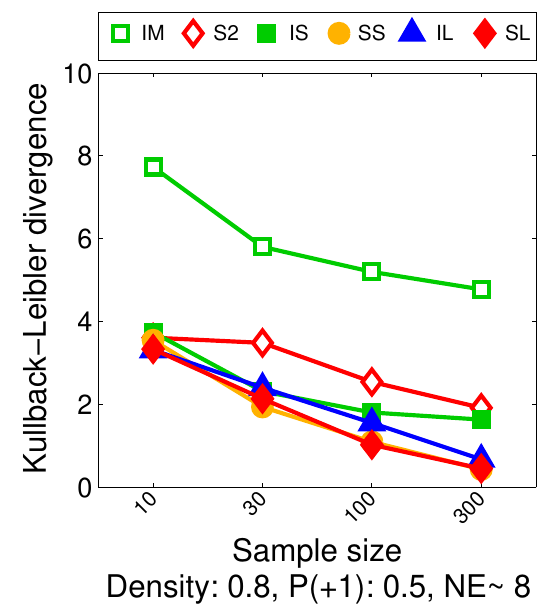} \\
\includegraphics[width=0.25\linewidth]{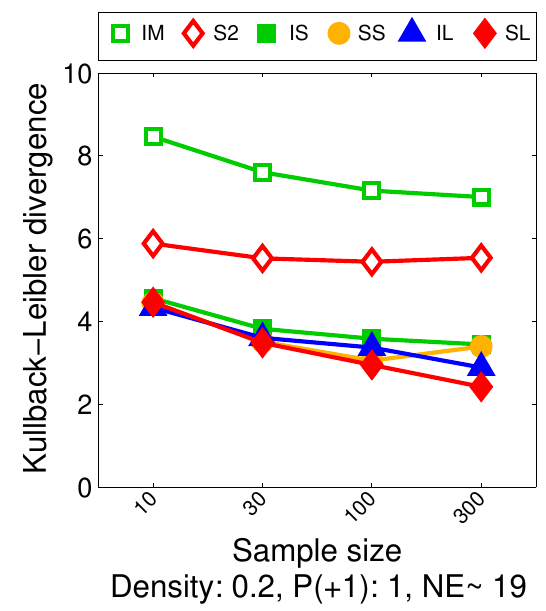}\includegraphics[width=0.25\linewidth]{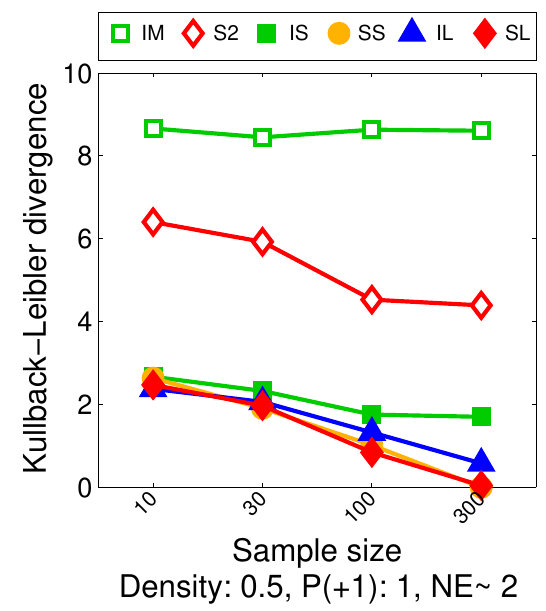}\includegraphics[width=0.25\linewidth]{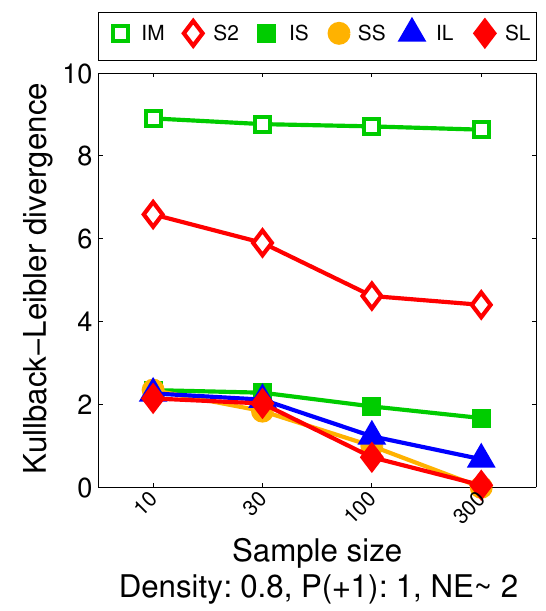} \\
\end{center}
\vspace{-0.2in}
\captionx{KL divergence between the recovered models and the ground truth for data sets of different number of samples. Each chart shows the density of the ground truth, probability $P(+1)$ that an edge has weight +1, and average number of equilibria (NE). Our convex loss methods (IS,SS: independent and simultaneous SVM, IL,SL: independent and simultaneous logistic regression) have lower KL than sigmoidal maximum empirical proportion of equilibria (S2) and Ising models (IM). The results are remarkably better when the number of equilibria in the ground truth model is small (e.g., for NE$<20$).}
\vspace{-0.1in}
\label{SyntheticRandomS}
\end{figure}

\begin{figure}
\begin{center}
\includegraphics[width=0.25\linewidth]{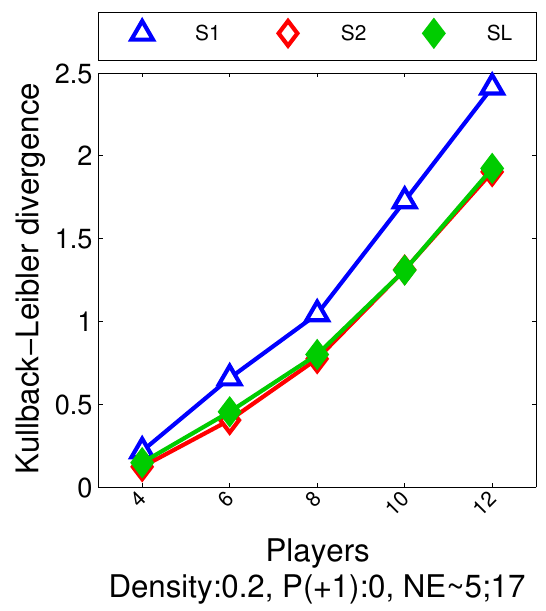}\includegraphics[width=0.25\linewidth]{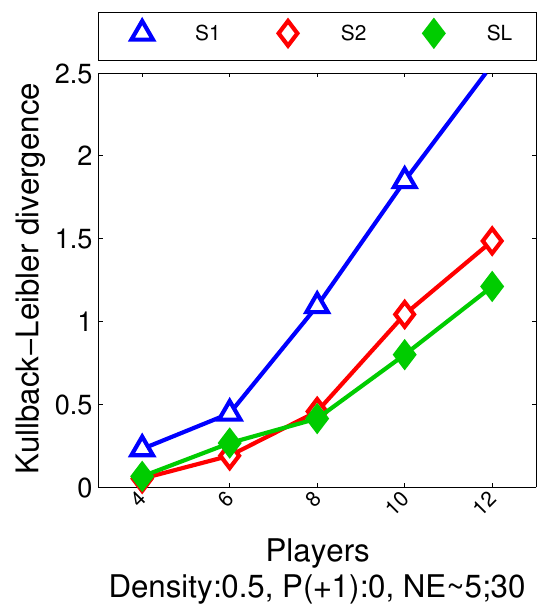}\includegraphics[width=0.25\linewidth]{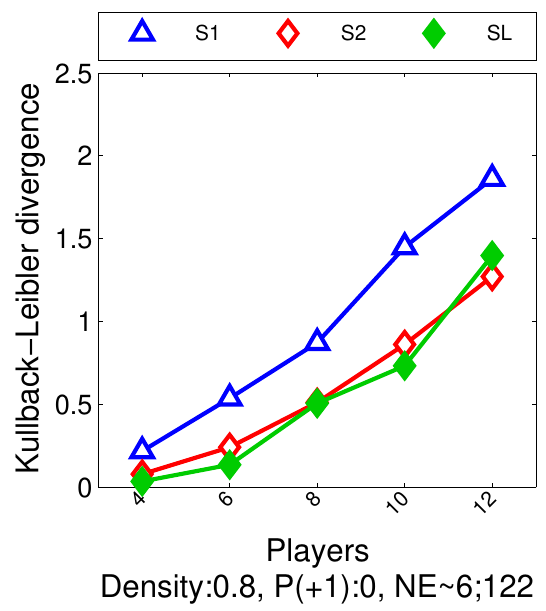} \\
\includegraphics[width=0.25\linewidth]{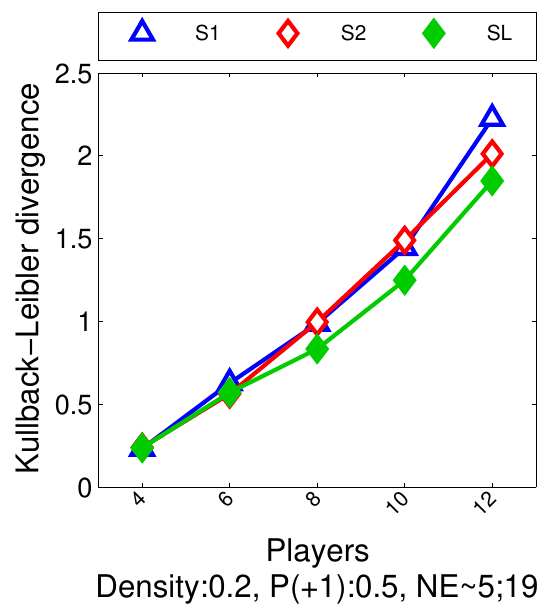}\includegraphics[width=0.25\linewidth]{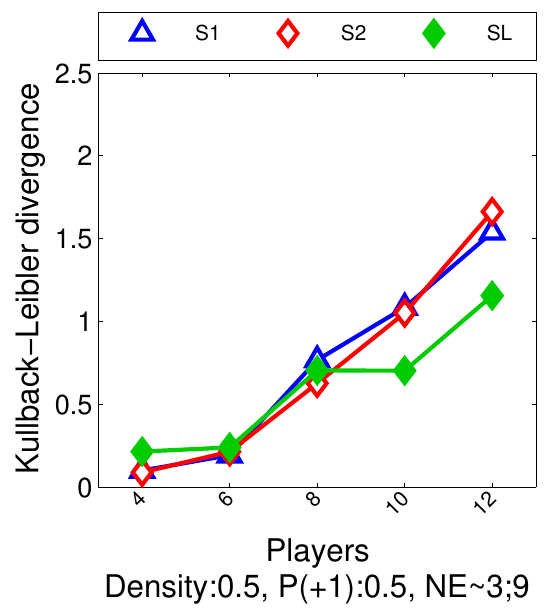}\includegraphics[width=0.25\linewidth]{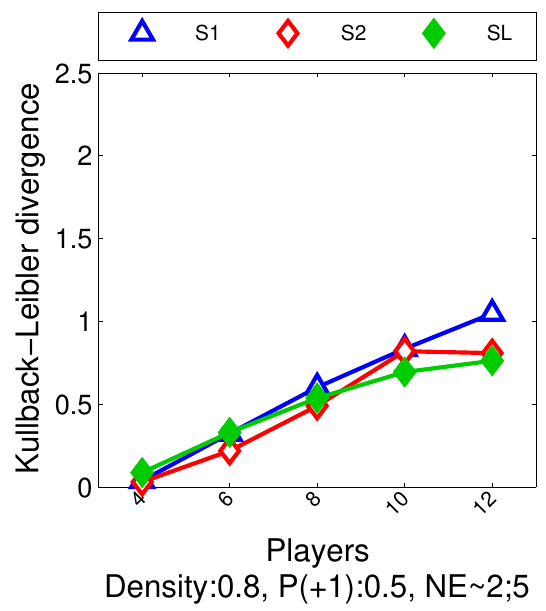} \\
\includegraphics[width=0.25\linewidth]{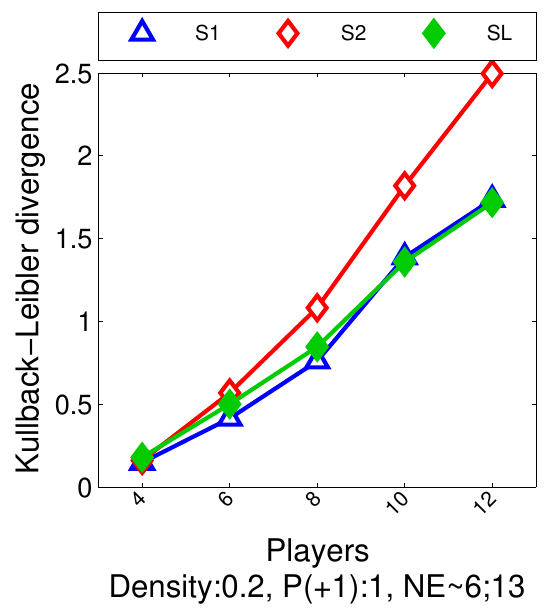}\includegraphics[width=0.25\linewidth]{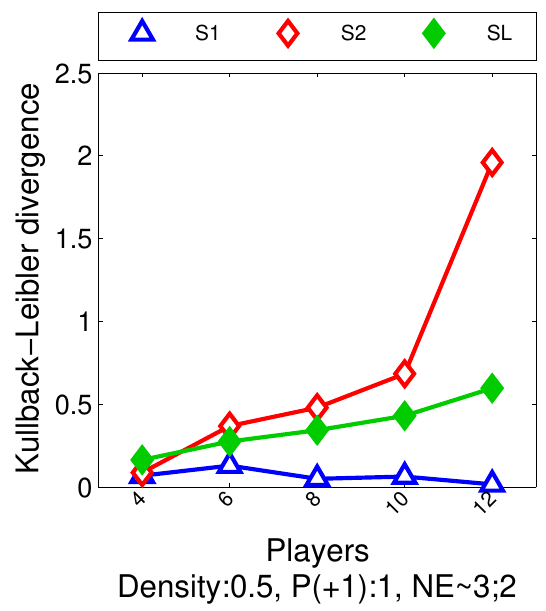}\includegraphics[width=0.25\linewidth]{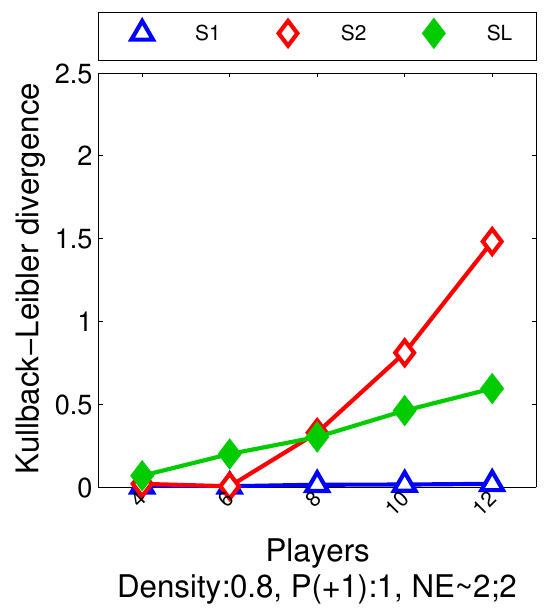} \\
\end{center}
\vspace{-0.2in}
\captionx{KL divergence between the recovered models and the ground truth for data sets of different number of players. Each chart shows the density of the ground truth, probability $P(+1)$ that an edge has weight +1, and average number of equilibria (NE) for $n=2$;$n=14$. In general, simultaneous logistic regression (SL) has lower KL than sigmoidal maximum empirical proportion of equilibria (S2), and the latter one has lower KL than sigmoidal maximum likelihood (S1). Other convex losses behave the same as simultaneous logistic regression (omitted for clarity of presentation).}
\vspace{-0.1in}
\label{SyntheticRandomP}
\end{figure}

In the next experiment, we show that the performance of convex loss minimization improves as the number of samples increases.
We used random graphs with slightly more variables and varying number of samples (10,30,100,300).
The ground truth model $\G_g=(\W_g,\b_g)$ contains $n=20$ players.
For each of 20 repetitions, we generate edges in the ground truth model $\W_g$ with a required density (either 0.2,0.5,0.8).
For simplicity, the weight of each edge is set to $+1$ with
probability $P(+1)$ and to $-1$ with probability
$1-P(+1)$.\footnote{Part of the reason for using such
  ``simple''/''limited'' binary set of weight values in this synthetic
  experiment regards the ability to generate ``interesting'' LIGs;
  that is, games with interesting sets of PSNE. As a word of caution,
  this is not as simple as it appears at first glance. LIGs with
  weights and biases generated \emph{uniformly} at random from some set
  of real values are almost always 
  not interesting, often having only 1 or 2 PSNE~\citep{Irfan13}. It is
  not until we move to more ``special''/restricted classes of games, such as that used in this
  experiments, that more interesting PSNE structure arises from
  randomly generated LIGs. That is in large part
  why we concentrated our experiments in games with those
  simple properties. (Simplicity itself also had a role in our
  decision, of course.) 

Please understand that we are not saying that LIGs that use a larger
  set of integers, or
  non-integer real-valued weights $w_{ij}$'s or $b_i$'s are not
  interesting, as the LIGs we learn from the real-world data demonstrate. What
  we \emph{are} saying is
  that we do not yet have a good understanding on how to randomly
  generate ``interesting'' synthetic games from the standpoint of their induced PSNE.
We leave a \emph{comprehensive} evaluation of our MLE-based algorithms'
\emph{ability to recover the PSNE of randomly generated synthetic
  LIGs}, which would involve a diversity of synthetic 
game graph structures, influence weights and biases that induce
``interesting'' sets of PSNE, for future work.}
Hence, the Nash equilibria of the generated games does not depend on the magnitude of the weights, just on their sign.
We set the bias $\b_g=\zero$ and the mixture parameter of the ground truth $q_g=0.7$.
We then generated a training and a validation set with the same number of samples.
Figure \ref{SyntheticRandomS} shows that our convex loss methods outperform (lower KL) sigmoidal maximum empirical proportion of equilibria and Ising models (except for the synthetic model with high true proportion of equilibria: density 0.8, $P(+1)=0$, NE$>1000$).
The results are remarkably better when the number of equilibria in the ground truth model is small (e.g., for NE$<20$).
From all convex loss methods, simultaneous logistic regression achieves the lowest KL.

In the next experiment, we evaluate two effects in our approximation methods.
First, we evaluate the impact of removing the true proportion of equilibria from our objective function, i.e., the use of maximum empirical proportion of equilibria instead of maximum likelihood.
Second, we evaluate the impact of using convex losses instead of a sigmoidal approximation of the 0/1 loss.
We used random graphs with varying number of players and 50 samples.
The ground truth model $\G_g=(\W_g,\b_g)$ contains $n=4,6,8,10,12$ players.
For each of 20 repetitions, we generate edges in the ground truth model $\W_g$ with a required density (either 0.2,0.5,0.8).
As in the previous experiment, the weight of each edge is set to $+1$ with probability $P(+1)$ and to $-1$ with probability $1-P(+1)$.
We set the bias $\b_g=\zero$ and the mixture parameter of the ground truth $q_g=0.7$.
We then generated a training and a validation set with the same number of samples.
Figure \ref{SyntheticRandomP} shows that in general, convex loss methods outperform (lower KL) sigmoidal maximum empirical proportion of equilibria, and the latter one outperforms sigmoidal maximum likelihood.
A different effect is observed for mild (0.5) to high (0.8) density and $P(+1)=1$ in which the sigmoidal maximum likelihood obtains the lowest KL.
In a closer inspection, we found that the ground truth games usually have only 2 equilibria: $(+1,\dots,+1)$ and $(-1,\dots,-1)$, which seems to present a challenge for convex loss methods.
It seems that for these specific cases, removing the true proportion of equilibria from the objective function negatively impacts the estimation process, but note that sigmoidal maximum likelihood is not computationally feasible for $n > 15$.

\subsection{Experiments on Real-World Data: U.S. Congressional Voting Records}

\begin{figure}
\begin{center}
\includegraphics[width=0.25\linewidth]{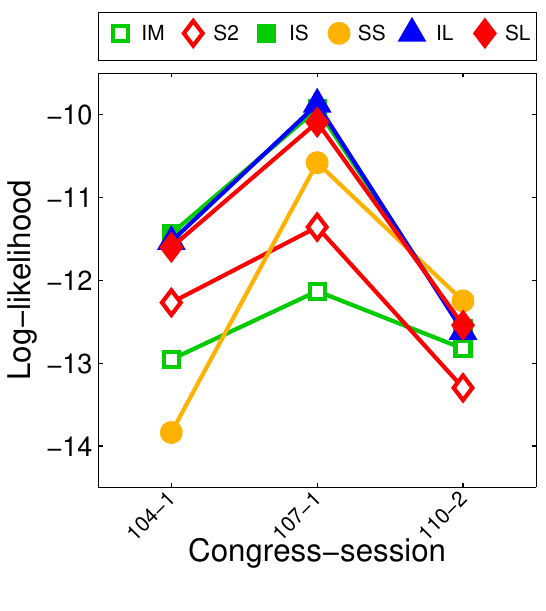}\includegraphics[width=0.25\linewidth]{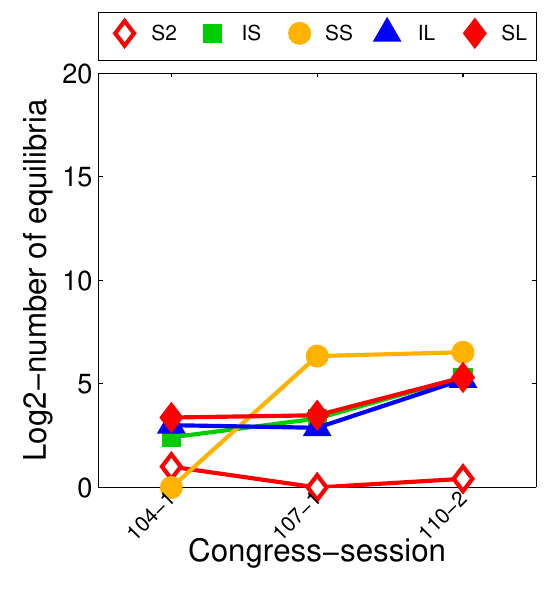}\includegraphics[width=0.25\linewidth]{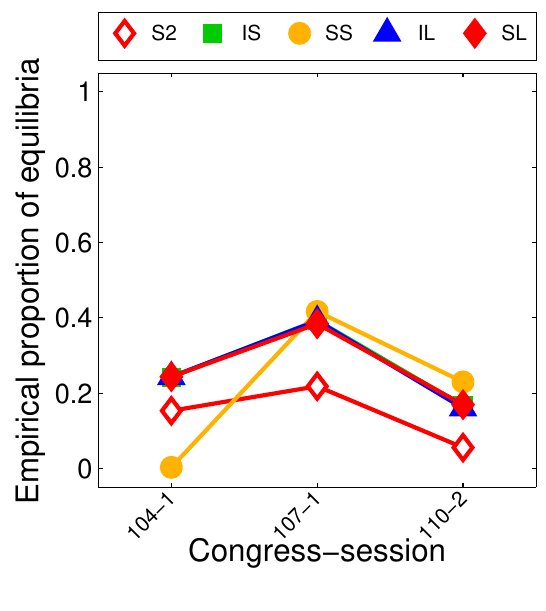} \\
\end{center}
\vspace{-0.3in}
\captionx{Statistics for games learnt from 20 senators from the first session of the 104th congress, first session of the 107th congress and second session of the 110th congress. The log-likelihood of our convex loss methods (IS,SS: independent and simultaneous SVM, IL,SL: independent and simultaneous logistic regression) is higher than sigmoidal maximum empirical proportion of equilibria (S2) and Ising models (IM). For all methods, the number of equilibria (and so the true proportion of equilibria) is low.}
\vspace{-0.1in}
\label{Votes}
\end{figure}

\begin{figure}
\begin{center}
\raisebox{0.1in}{\footnotesize{(a)}}\includegraphics[width=0.25\linewidth,height=0.85in]{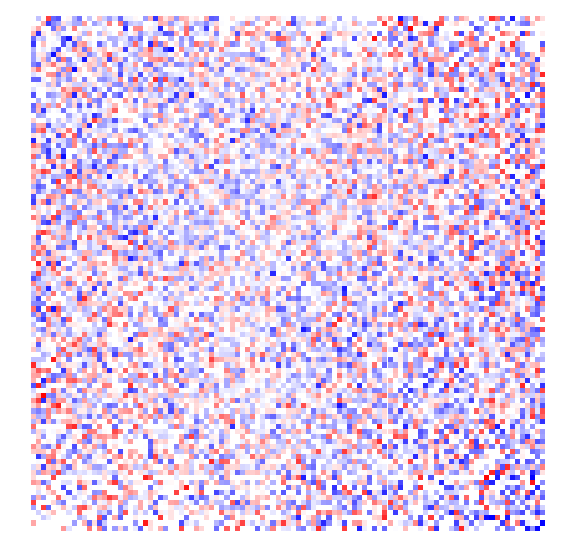}\includegraphics[width=0.25\linewidth,height=0.85in]{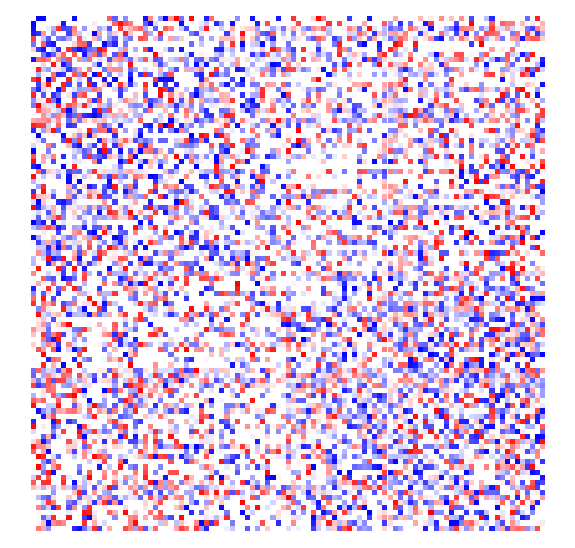}\includegraphics[width=0.25\linewidth,height=0.85in]{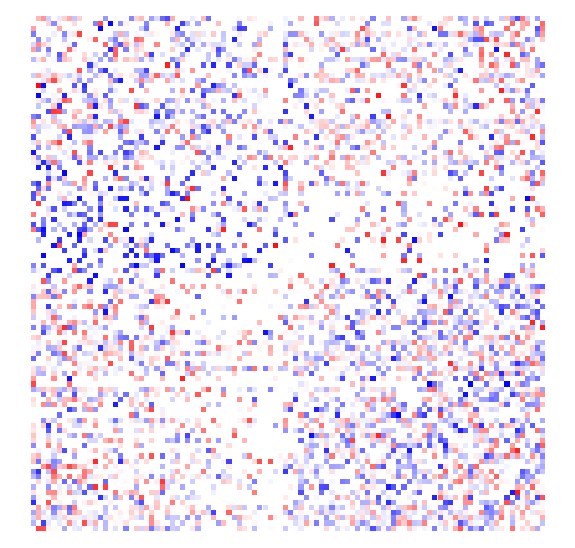} \\
\raisebox{0.1in}{\footnotesize{(b)}}\includegraphics[width=0.25\linewidth,height=0.85in]{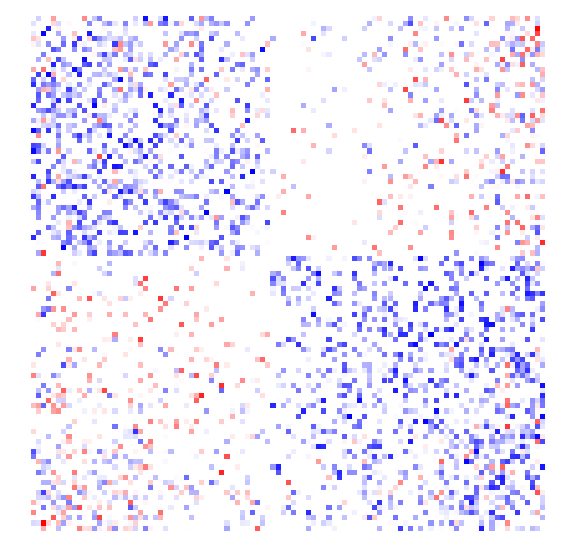}\includegraphics[width=0.25\linewidth,height=0.85in]{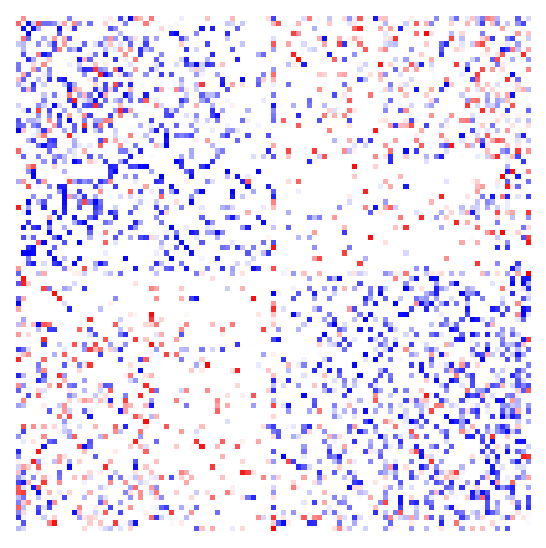}\includegraphics[width=0.25\linewidth,height=0.85in]{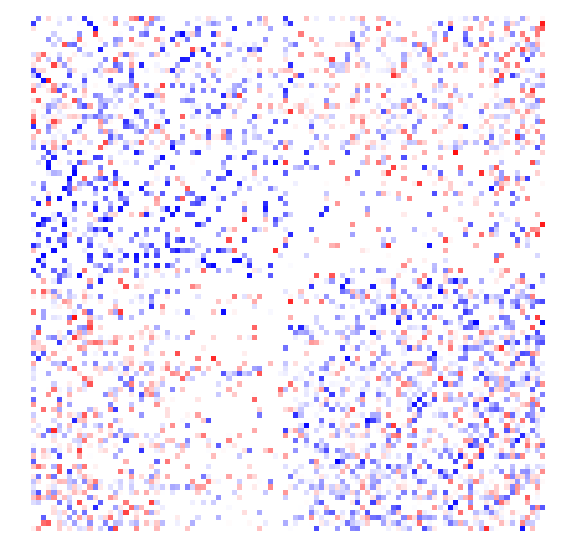} \\
\raisebox{0.1in}{\footnotesize{(c)}}\includegraphics[width=0.25\linewidth,height=1.5in]{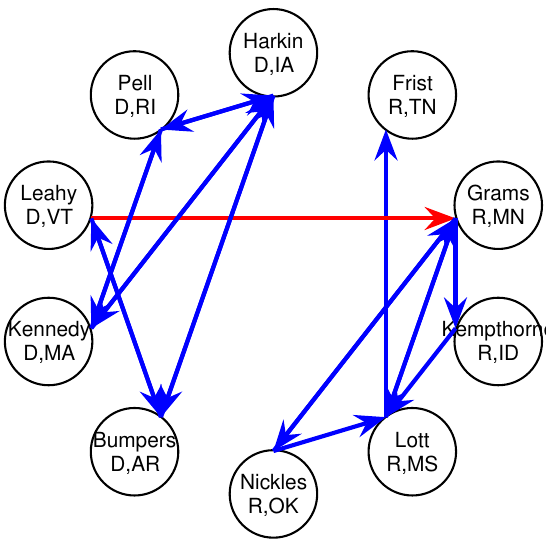}\includegraphics[width=0.25\linewidth,height=1.5in]{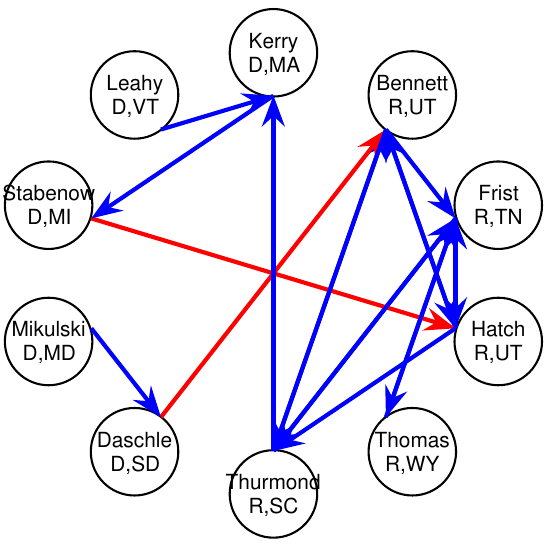}\includegraphics[width=0.25\linewidth,height=1.5in]{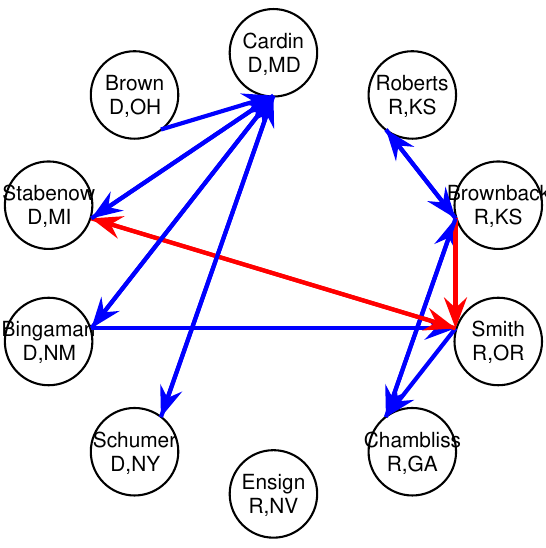} \\
\raisebox{0.1in}{\footnotesize{(d)}}\includegraphics[width=0.25\linewidth,height=1.5in]{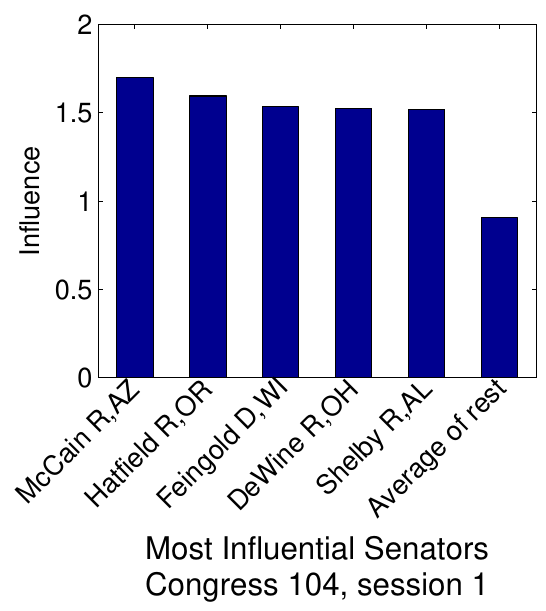}\includegraphics[width=0.25\linewidth,height=1.5in]{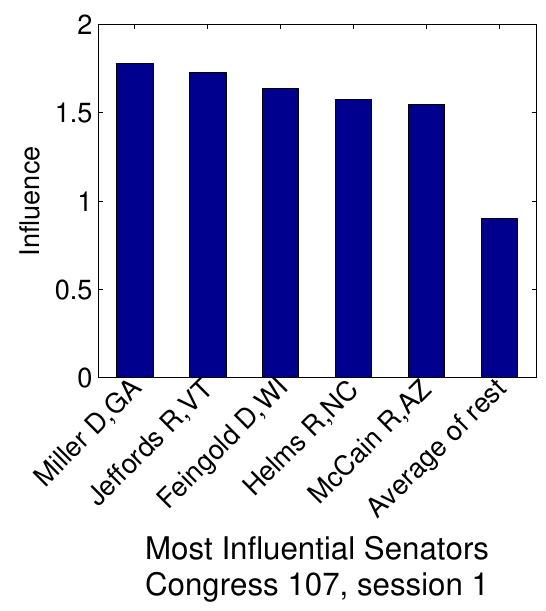}\includegraphics[width=0.25\linewidth,height=1.5in]{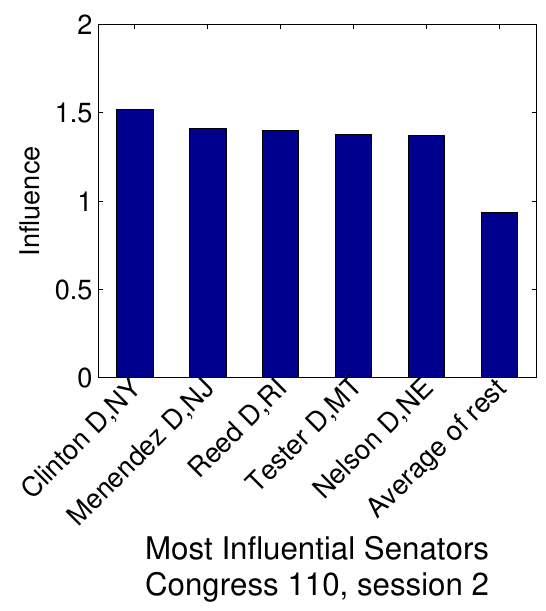} \\
\raisebox{0.1in}{\footnotesize{(e)}}\includegraphics[width=0.25\linewidth,height=1.5in]{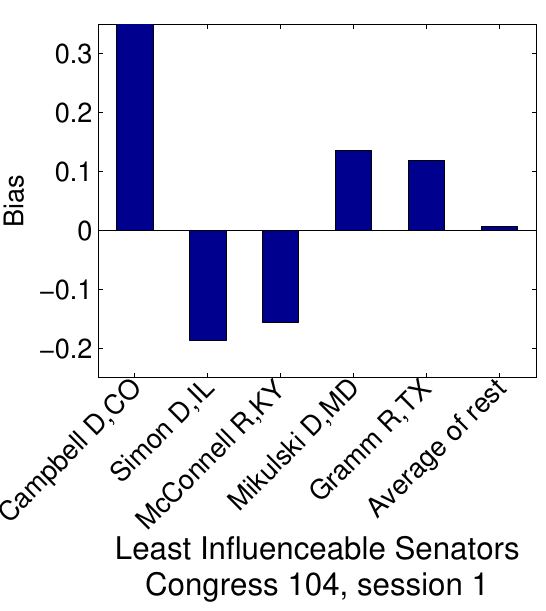}\includegraphics[width=0.25\linewidth,height=1.5in]{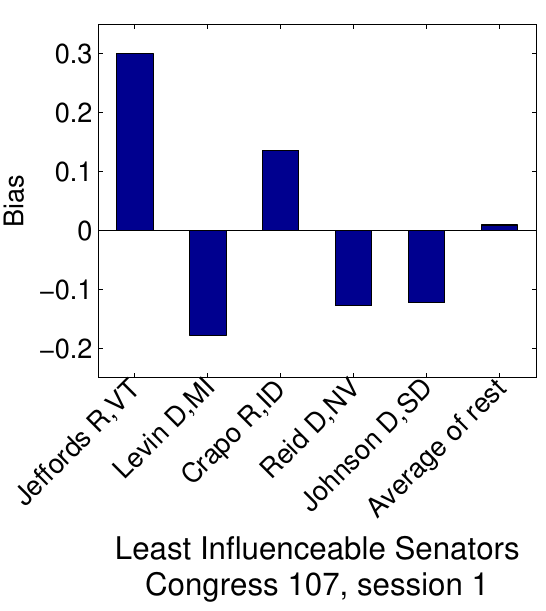}\includegraphics[width=0.25\linewidth,height=1.5in]{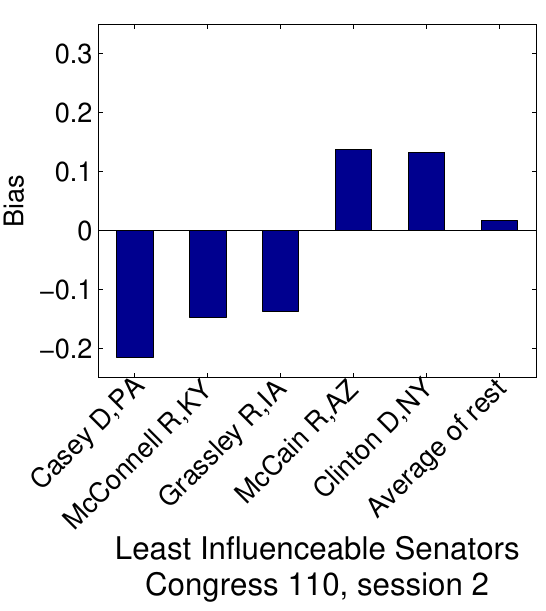} \\
\end{center}
\vspace{-0.2in}
\captionx{(Top) {\bf Matrices of (direct) influence weights $\W$ for games learned from all 100
  senators}, from the first session of the 104th congress (left), first
  session of the 107th congress (center) and second session of the
  110th congress (right), by using our independent (a) and
  simultaneous (b) logistic regression methods. A row represents how
  much every other senator directly-influence the senator in such row,
  in terms of the influence weights of the learned LIG. Positive 
  influence-weight parameter values are shown in blue; negative values
  are in red. Democrats are shown in the top/left corner, while
  Republicans are shown in the bottom/right corner. Note that
  simultaneous method produce structures that are sparser than its
  independent counterpart. (c) {\bf Partial view of the graph for
    simultaneous logistic regression}. (d) {\bf Most
    directly-influential senators} and (e) {\bf least directly-influenceable
    senators}.
  Regularization parameter $\rho=0.0006$.}
\vspace{-0.1in}
\label{MatrixGraphs}
\end{figure}

\begin{figure}
\begin{center}
\includegraphics[width=0.25\linewidth]{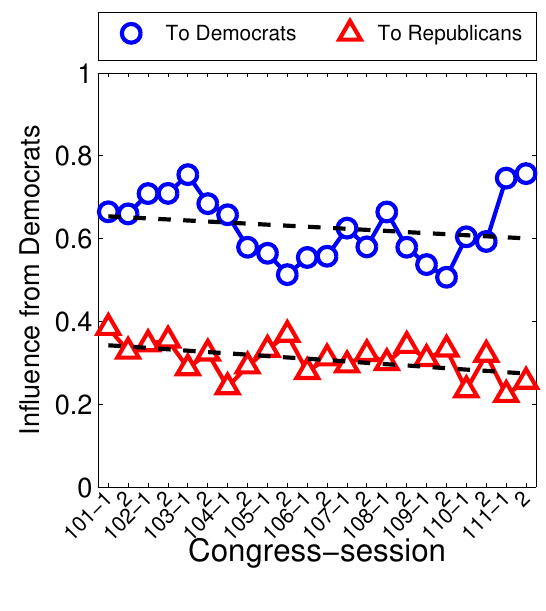}\includegraphics[width=0.25\linewidth]{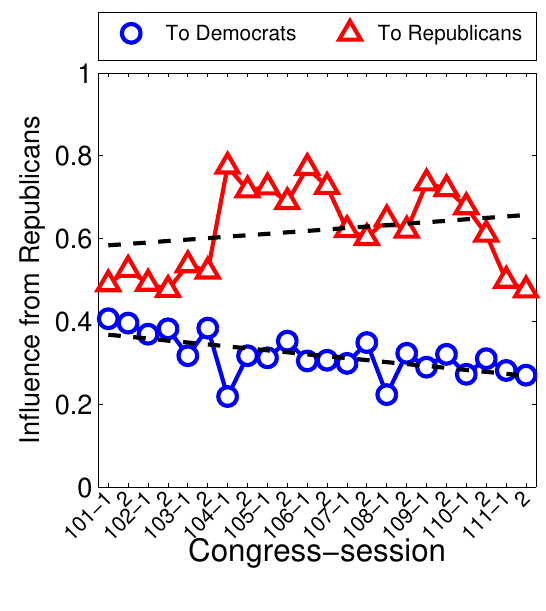}\includegraphics[width=0.25\linewidth]{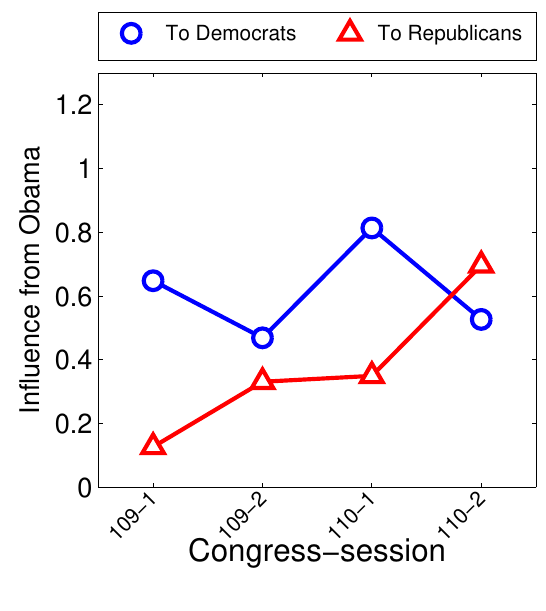}\includegraphics[width=0.25\linewidth]{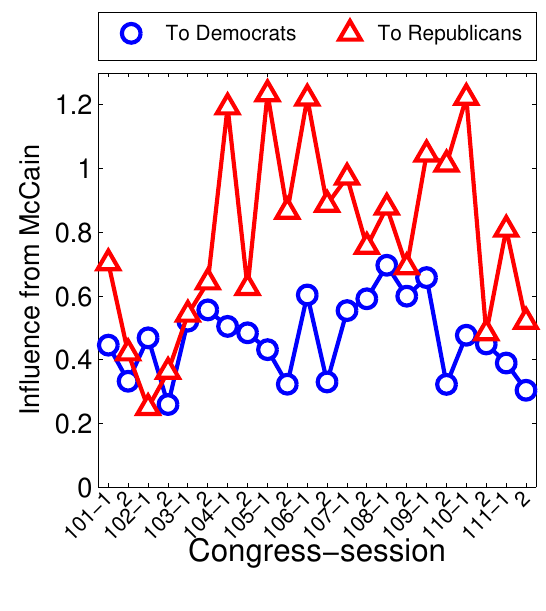} \\
\end{center}
\vspace{-0.3in}
\captionx{Direct influence between parties and direct influences from
  Obama and McCain. Games were learned from all 100 senators from the
  101th congress (Jan 1989) to the 111th congress (Dec 2010) by using
  our simultaneous logistic regression method. Direct influence
  between senators of the same party are stronger than senators of
  different party, which is also decreasing over time. In the last
  sessions, direct influence from Obama to Republicans increased, and influence from McCain to both parties decreased. Regularization parameter $\rho=0.0006$.}
\vspace{-0.1in}
\label{Influence}
\end{figure}

We used the U.S. congressional voting records in order to measure the generalization performance of convex loss minimization in a real-world data set.
The data set is publicly available at \url{http://www.senate.gov/}.
We used the first session of the 104th congress (Jan 1995 to Jan 1996, 613 votes), the first session of the 107th congress (Jan 2001 to Dec 2001, 380 votes) and the second session of the 110th congress (Jan 2008 to Jan 2009, 215 votes).
Following on other researchers who have experimented with this data
set (e.g., \citealt{Banerjee08}), abstentions were replaced with negative votes.
Since reporting the log-likelihood requires computing the number of equilibria (which is NP-hard), we selected only 20 senators by stratified random sampling.
We randomly split the data into three parts.
We performed six repetitions by making each third of the data take turns as training, validation and testing sets.
Figure \ref{Votes} shows that our convex loss methods outperform (higher log-likelihood) sigmoidal maximum empirical proportion of equilibria and Ising models.
From all convex loss methods, simultaneous logistic regression achieves the lowest KL.
For all methods, the number of equilibria (and so the true proportion of equilibria) is low.

We apply convex loss minimization to larger problems, by learning
structures of games from all 100 senators.
Figure \ref{MatrixGraphs} shows that simultaneous logistic regression produce structures that are sparser than its independent counterpart.
The simultaneous method better elicits the bipartisan structure of the congress.
We define the \emph{(aggregate) direct influence} of player $j$ to all other players as $\sum_i{|w_{ij}|}$ after normalizing all weights, i.e., for each player $i$ we divide $(\w{i},b_i)$ by $\|\w{i}\|_1+|b_i|$.
Note that Jeffords and Clinton are one of the 5 most directly-influential as well as 5 least directly-influenceable (high bias) senators, in the 107th and 110th congress respectively.
McCain and Feingold are both in the list of 5 most directly-influential senators in the 104th and 107th congress.
McCain appears again in the list of 5 least \emph{directly} influenceable
senators in the 110th congress (as defined above in the context of the
LIG model).

We test the hypothesis that the aggregate direct influence, \emph{as defined by
  our model}, between senators of the same party are stronger than senators of different party.
We learn structures of games from all 100 senators from the 101th congress to the 111th congress (Jan 1989 to Dec 2010).
The number of votes cast for each session were average: 337, minimum: 215, maximum: 613.
Figure \ref{Influence} validates our hypothesis and more interestingly, it shows that influence between different parties is decreasing over time.
Note that the influence from Obama to Republicans increased in the last sessions, while McCain's influence to Republicans decreased.

Since the U.S. Congressional voting data is observational, we used the log-likelihood as an adequate measure of predictive performance.
We argue that the log-likelihood of joint actions provides a more ``global view'' compared to predicting the action of a single agent.
Furthermore, predicting the action of a single agent (i.e., $x_i$) works under the assumption that we have access to the decisions of the other agents (i.e., $\x_\minus{i}$), which is in contrast to our framework.
Regarding causal strategic inference, \citet{Irfan13} use the games that we produce in this section in order to address problems such as the identification of most influential senators. (We refer the reader to their paper for further details.)

\section{Concluding Remarks} \label{SecConclusions}

In Section~\ref{SecAlgorithm}, we present a variety of algorithms to
learn LIGs from strictly behavioral data, including what we call \emph{independent
logistic regression (ILR).} There is a very popular technique for learning
Ising models that uses independent regularized logistic regression to compute
the individual conditional probabilities \emph{as a step toward computing
  a globally coherent joint probability distribution}.  However, this approach is
inherently problematic, as some authors have previously pointed out
(see, e.g., \citealt{Guo10}). Without getting too technical, the main
roadblock is that there is no guarantee that estimates of the weights
produced by the individual regressions be symmetric: $\widehat{w}_{ij}
= \widehat{w}_{ji}$ for all $i,j$. Learning an Ising model requires the enforcement
of this condition, and a variety of heuristics have been
proposed. (Please see Section~\ref{sec:probmod} for relevant work and references in this area.)

We also \emph{apply} ILR in \emph{exactly} the same manner but for a
\emph{different objective}: learning LIGs. Some seem to think that
this diminishes the significance of our contributions. We strongly
believe the opposite is true: That we can
learn games by using such simple, practical, efficient and well-studied
techniques is a significant plus in our view. Again, without getting too technical, the estimates of 
ILR need not be symmetric for LIG models, and are always perfectly
consistent with the LIG definition. In fact, asymmetric estimates are common
in practice (the LIG for the 110th Congress depicted in
Figure~\ref{Congress110} is an example). And we believe this makes the model more interesting. In
the ILR-learned LIG, a player may have a positive, negative or no direct effect on
another players utility, and \emph{vice versa}.\footnote{It is easy to come
up with examples of such opposing/antagonistic interactions between individuals in
real-world settings. (See, e.g., ``parents with teenagers,'' perhaps
a more timely example is 
the U.S. Congress in recent times.)}

Thus, despite the \emph{process} of estimation of model parameters
being similar, the \emph{view} of the output of that
estimation process is radically different in each
case. Our experiments show that our generative
model with LIGs built from ILR estimates achieves higher
generalization likelihoods than standard probabilistic models such
as Ising models that may also use ILR. This fact, that
the generative model defined in terms of game-theoretic
equilibrium concepts can explain the data better than traditional
probabilistic models, provides further evidence
supporting such a game-theoretic ``view'' of the ILR estimated parameters  and yields additional confidence
in their use in game-theoretic models. 

In short, ILR is a thoroughly studied method with a long tradition and an
extensive literature from which we can only benefit. We find it to a be a good, unexpected outcome of our research in
this work, and thus a reasonably significant contribution, that we can
successfully and effectively use ILR, a very simple and practical estimation
technique for learning probabilistic graphical models, to learn
game-theoretic graphical models too.

\subsection{Extensions and Future Work}

There are several ways of extending this research.
We can extend our approach to $\epsilon$-approximate
PSNE.\footnote{By definition, given $\epsilon \geq 0$, a joint pure-strategy $\x^*$ is an
  \emph{$\epsilon$-approximate PSNE} if for each player $i$, we
  have $u_i(\x^*) \geq \max_{x_i} u_i(x_i,\x_{-i}^*) - \epsilon$; in other
  words, no players can gain more than $\epsilon$ in payoff/utility value from unilaterally
  deviating from $x_i^*$, assuming the other players play $\x_{-i}^*$. Using
  this definition, we can see that a
  PSNE is simply a $0$-approximate PSNE.}
In this case, for each player instead of one condition, we will have two best-response conditions which are still linear in $\W$ and $\b$.
Additionally,
we can extend our approach to a broader class of graphical games and non-Boolean actions.
Note that our analysis does not rely on binary actions, but on binary features of one player $\iverson{x_i=1}$ or two players $\iverson{x_i=x_j}$.
We can use features of three players $\iverson{x_i=x_j=x_k}$ or of non-Boolean actions $\iverson{x_i=3,x_j=7}$.
This kernelized version is still linear in $\W$ and $\b$.
These extensions are possible because our algorithms and analysis rely
on linearity and binary features; additionally, we can obtain a new
upper-bound on the ``VC-dimension'' by changing the inputs of the
neural-network architecture.
We can easily extend our approach to parameter learning for \emph{fixed}
structures by using a $\ell_2^2$ regularizer instead.

Future work should also consider and study more sophisticated noise
processes, MSNE, and the analysis of
different upper bounds for the 0/1 loss (e.g., exponential, smooth
hinge).
Finally, we should consider other slightly more complex versions of
our model based on Bayesian or stochastic games to account for
possible variations of the influence-weights and bias-threshold
parameters. As an example, we may consider versions of our
model for congressional voting that would explicitly capture
game differences in terms of influences and biases that depend on the
nature or topic of each specific bill being voted on, as well as
Senators' time-changing preferences and trends.

\section*{Acknowledgements}
We are grateful to Christian Luhmann for extensive discussion of
many aspects of this work and the multiple comments and feedback he provided to
improve both the work and the presentation. We warmly
thank Tommi Jaakkola for several informal discussions on the topic of
learning games and the ideas behind causal strategic inference, as
well as suggestions on how to improve the presentation. We
also thank Mohammad Irfan for his help with motivating causal
strategic inference in inference games and examples to illustrate their
use. We would also like to thank Alexander Berg, Tamara Berg
and Dimitris Samaras for their comments and questions during informal
presentations of this work, which
helped us tailor the presentation to a wider audience.
We are very grateful to Nicolle Gruzling for sharing her Master's thesis \citep{Gruzling06} which contains valuable references used in Section \ref{SubSecExhaustive}.
Finally, we
thank several anonymous reviewers for their comments and feedback, and in particular, an anonymous reviewer for the reference
to an  overview on the literature on composite marginal likelihoods.

Shorter versions of the work presented here appear as Chapter 7 of the first author's
Ph.D. dissertation~\citep{Honorio_thesis12} and as e-print {\tt
  arXiv:1206.3713 [cs.LG]}~\citep{Honorio_arxiv12}.

This work was supported in part by NSF CAREER Award IIS-1054541.

\appendix

\section{Additional Discussion}

In this section, we discuss our choice of modeling end-state predictions without modeling dynamics.
We also discuss some alternative noise models to the one studied in this manuscript.

\subsection{On End-State Predictions without Explicitly Modeling Dynamics}
\label{App:EndState}

Certainly, in cases
  in which information about the dynamics is available, the learned
  model \emph{may} use such information while still making end-state
  predictions. But no such information, either via data sequences or
  prior knowledge, is available in any of the
  publicly-available real-world data sets we study here. Take
  congressional voting as an example. Considering the voting records as sequence of votes
  does not seem sensible in our context from a modeling/engineering
  perspective because the data set does not have any detailed
  information about the nature of the vote: we just have each
  senator's vote on whatever bill they considered, and little to no
  information about the detailed dynamics that might have lead to the senators'
  final votes. Indeed, one may go further and
  argue that assuming the the availability of information about the
  dynamics of the process is a considerable burden on the modeler and
  something of a wishful thinking in many practical, real-world
  settings. 

Besides the nature of our setting, the lack of data or information, and the CSI
  motivation, there are other more fundamental ML reasons why we have no
interest in considering dynamics in this paper. First, we view the additional complexity of a dynamic/temporal model as
  providing the wrong tradeoff:
      dynamic/temporal models are often inherently more complex to express and
      learn from data. Second, it is common
      ML practice to separate single example/outcome problems from
      sequence problems; said differently, ML generally treats the problem of learning from individual
      i.i.d. examples different from that of learning from sequences or sequence
      prediction. Third, we invoke Vapnik's Principle of Transduction
      (\citeauthor{vapnikSLT}, \citeyear{vapnikSLT}, page 477):\footnote{See also
        \url{http://www.cs.man.ac.uk/~jknowles/transductive.html}
        additional information.}
\begin{quote}
      ``When solving a problem of interest, do not solve a more
      general problem as an intermediate step. Try to get the answer
      that you really need but not a more general one.'' 
\end{quote}

We believe this additional complexity
    of temporal dynamics, while possibly more ``realistic,'' might easily weaken the power of the ``bottom-line'' prediction
  of the possible stable final outcomes because the resulting models
  can get side-tracked by the details of the
  interaction. We believe the difficulty of modeling such details, specially
  with the relatively limited amount of the data available, leads to
  poor prediction performance on what we really care about from an
  \emph{engineering} stand point: We would like to know or predict \emph{what} will end up
  happening, and have little or no interest on the \emph{how} or
  \emph{why} this happens. 

We recognize the \emph{scientific} significance
  and importance of research in social and behavioral sciences such as
  sociology, psychology and, in some
  cases economics, on explaining, at a higher level of abstraction,
  often going as low as the ``cognitive'' or ``neuroscience'' level,  the process by which final decisions
  are reached. 

We believe the sudden
 growth of interest from industry (both online and physical
 companies), government and other national and international institutions
 on predicting ``behavior'' for the purpose of revenue, improving
 efficiency, instituting effective policies with minimal regulations,
 etc., should shift the
 focus of the study of ``behavior''  closer to an engineering
 endeavor. We believe such entities are after the ``bottom line'' and
 will care more about the end-goal
 than \emph{how or why a specific outcome is achieved}, modulo, of course, having
 simple enough and tractable computational models that provide
 reasonably accurate predictions of final end-state behavior, or at least accurate enough for their purposes.

\subsection{On Alternative Noise Models}

Next, we discuss some alternative noise models to the one studied in this manuscript.
Specifically, we discuss an extension of the PSNE-based mixture noise model as well as individual-player noise models.

\subsubsection{On Generalizations of the PSNE-based Mixture Noise Models}
\label{App:GenDist}
A simple extension to our model in Eq.~\eqref {Prob} is to allow for more general distributions for the PSNE and the non-PSNE sets.
That is, with some probability $0<q<1$, a joint action $\x$ is chosen from $\NE(\G)$ by following a distribution $\PD_\alpha$ parameterized by $\alpha$;
otherwise, $\x$ is chosen from its complement set $\{-1,+1\}^n - \NE(\G)$ by following a distribution $\PD_\beta$ parameterized by $\beta$.
The corresponding probability mass function (PMF) over joint-behaviors $\{-1,+1\}^n$ parameterized by $(\G,q,\alpha,\beta)$ is
\begin{equation*}
\textstyle
p_{(\G,q,\alpha,\beta)}(\x) = q \, \frac{p_\alpha(\x) \iverson{\x \in \NE(\G)}}{\sum_{\z \in \NE(\G)}{p_\alpha(\z)}} + (1-q) \, \frac{p_\beta(\x) \iverson{\x \notin \NE(\G)}}{\sum_{\z \notin \NE(\G)}{p_\beta(\z)}} \; ,
\end{equation*}
\noindent where $\G$ is a game, $p_\alpha(\x)$ and $p_\beta(\x)$ are PMFs over $\{-1,+1\}^n$.

One reasonable technique to learn such a model from data is to perform an alternate method.
Compared to our simpler model, this model requires a step that maximizes the likelihood by changing $\alpha$ and $\beta$ while keeping $\G$ and $q$ constant.
The complexity of this step will depend on how we parameterize the PMFs $\PD_\alpha$ and $\PD_\beta$ but will very likely be an NP-hard problem because of the partition function.
Furthermore, the problem of maximizing the likelihood by changing $\NE(\G)$ (while keeping $\alpha$, $\beta$ and $q$ constant) is combinatorial in nature and in this paper, we provided a tractable approximation method with provable guarantees (for the case of uniform distributions).
Other approaches for maximizing the likelihood by changing $\G$ are
very likely to be exponential as we discuss briefly at the start of Section~\ref{SecAlgorithm}.

\subsubsection{On Individual-Player Noise Models}
\label{App:IndNoise}
As an example, consider a the generative model in which we first
randomly select a PSNE $\x$ of the game from a distribution $\PD_\alpha$ parameterized by $\alpha$, and then each player $i$, independently, acts according to $x_i$ with probability $q_i$ and switches its action with probability $1-q_i$.
The corresponding probability mass function (PMF) over joint-behaviors $\{-1,+1\}^n$ parameterized by $(\G,q_1,\dots,q_n,\alpha)$ is
\begin{equation*}
\textstyle
p_{\G,\q,\alpha}(\x) = \sum_{\y \in \NE(\G)}{\frac{p_\alpha(\y)}{\sum_{\z \in \NE(\G)}{p_\alpha(\z)}}}\prod_i{q_i^{\iverson{x_i=y_i}}(1-q_i)^{\iverson{x_i \ne y_i}}} \; ,
\end{equation*}
\noindent where $\G$ is a game, $\q=(q_1,\dots,q_n)$ and $p_\alpha(\y)$ is a PMF over $\{-1,+1\}^n$.

One reasonable technique to learn such a model from data is to perform an alternate method.
In one step, we maximize the likelihood by changing $\NE(\G)$ while keeping $\q$ and $\alpha$ constant, which could be tractably performed by applying Jensen's inequality since the problem is combinatorial in nature.
In another step, we maximize the likelihood by changing $\q$ while keeping $\G$ and $\alpha$ constant, which could be performed by gradient ascent (the complexity of this step depends on the size of $\NE(\G)$ which could be exponential!).
In yet another step, we maximize the likelihood by changing $\alpha$ while keeping $\G$ and $\q$ constant (the complexity of this step will depend on how we parameterize the PMF $\PD_\alpha$ but will very likely be an NP-hard problem because of the partition function).
The main problem with this technique is that it can be formally proved that the first step (Jensen's inequality) will almost surely pick a single equilibria for the model (i.e., $|\NE(\G)|=1$).

\bibliography{references}

\begin{thebibliography}{105}
\expandafter\ifx\csname natexlab\endcsname\relax\def\natexlab#1{#1}\fi

\bibitem[Ackermann and Skopalik(2007)]{Ackermann07}
H.~Ackermann and A.~Skopalik.
\newblock On the Complexity of Pure {Nash} Equilibria in Player-Specific
  Network Congestion Games.
\newblock In X.~Deng and F.~Graham, editors, {\em Internet and Network
  Economics}, volume 4858 of {\em Lecture Notes in Computer Science}, pages
  419--430. Springer Berlin Heidelberg, 2007.

\bibitem[Aichholzer and Aurenhammer(1996)]{Aichholzer96}
O.~Aichholzer and F.~Aurenhammer.
\newblock Classifying Hyperplanes in Hypercubes.
\newblock {\em SIAM Journal on Discrete Mathematics}, 9\penalty0 (2):\penalty0
  225--232, 1996.

\bibitem[Aumann(1974)]{aumann74}
R.~Aumann.
\newblock Subjectivity and Correlation in Randomized Strategies.
\newblock {\em Journal of Mathematical Economics}, 1:\penalty0 67--96, 1974.

\bibitem[Ballester et~al.(2004)Ballester, Calv{\'o}-Armengol, and
  Zenou]{ballesteretal04}
C.~Ballester, A.~Calv{\'o}-Armengol, and Y.~Zenou.
\newblock Who's Who in Crime Networks. Wanted: the Key Player.
\newblock Working Paper Series 617, Research Institute of Industrial Economics,
  2004.

\bibitem[Ballester et~al.(2006)Ballester, Calv{\'o}-Armengol, and
  Zenou]{ballesteretal06}
C.~Ballester, A.~Calv{\'o}-Armengol, and Y.~Zenou.
\newblock Who's Who in Networks. Wanted: The Key Player.
\newblock {\em Econometrica}, 74\penalty0 (5):\penalty0 1403--1417, 2006.

\bibitem[Banerjee et~al.(2008)Banerjee, {El~Ghaoui}, and
  d'Aspremont]{Banerjee08}
O.~Banerjee, L.~{El~Ghaoui}, and A.~d'Aspremont.
\newblock Model Selection Through Sparse Maximum Likelihood Estimation for
  Multivariate {Gaussian} or Binary Data.
\newblock {\em Journal of Machine Learning Research}, 9:\penalty0 485--516,
  2008.

\bibitem[Banerjee et~al.(2006)Banerjee, {El~Ghaoui}, d'Aspremont, and
  Natsoulis]{Banerjee06}
O.~Banerjee, L.~{El~Ghaoui}, A.~d'Aspremont, and G.~Natsoulis.
\newblock Convex Optimization Techniques for Fitting Sparse {Gaussian}
  Graphical Models.
\newblock In W.~Cohen and A.~Moore, editors, {\em Proceedings of the 23nd
  International Machine Learning Conference}, pages 89--96. Omni Press, 2006.

\bibitem[Blum et~al.(2006)Blum, Shelton, and Koller]{blumetal06}
B.~Blum, C.R. Shelton, and D.~Koller.
\newblock A Continuation Method for \uppercase{N}ash Equilibria in Structured
  Games.
\newblock {\em Journal of Artificial Intelligence Research}, 25:\penalty0
  457--502, 2006.

\bibitem[Boyd and Vandenberghe(2006)]{Boyd06}
S.~Boyd and L.~Vandenberghe.
\newblock {\em Convex Optimization}.
\newblock Cambridge University Press, 2006.

\bibitem[Bradley and Mangasarian(1998)]{Bradley98}
P.~S. Bradley and O.~L. Mangasarian.
\newblock Feature Selection via Concave Minimization and Support Vector
  Machines.
\newblock In {\em Proceedings of the 15th International Conference on Machine
  Learning}, ICML '98, pages 82--90, San Francisco, CA, USA, 1998. Morgan
  Kaufmann Publishers Inc.

\bibitem[Brock and Durlauf(2001)]{brock_and_durlauf01}
W.~Brock and S.~Durlauf.
\newblock Discrete Choice with Social Interactions.
\newblock {\em The Review of Economic Studies}, 68\penalty0 (2):\penalty0
  235--260, 2001.

\bibitem[Camerer(2003)]{Camerer03}
C.~Camerer.
\newblock {\em Behavioral Game Theory: Experiments on Strategic Interaction}.
\newblock Princeton University Press, 2003.

\bibitem[Cao et~al.(2011)Cao, Wu, Hu, and Wang]{caoetal11}
T.~Cao, X.~Wu, T.~Hu, and S.~Wang.
\newblock Active Learning of Model Parameters for Influence Maximization.
\newblock In D.~Gunopulos, T.~Hofmann, D.~Malerba, and M.~Vazirgiannis,
  editors, {\em Machine Learning and Knowledge Discovery in Databases}, volume
  6911 of {\em Lecture Notes in Computer Science}, pages 280--295. Springer
  Berlin Heidelberg, 2011.

\bibitem[Chapman et~al.(2010)Chapman, Farinelli, {Munoz~de~Cote}, Rogers, and
  Jennings]{Chapman10}
A.~Chapman, A.~Farinelli, E.~{Munoz~de~Cote}, A.~Rogers, and N.~Jennings.
\newblock A Distributed Algorithm for Optimising over Pure Strategy {Nash}
  Equilibria.
\newblock In {\em AAAI Conference on Artificial Intelligence}, 2010.

\bibitem[Chickering(2002)]{Chickering02}
D.~Chickering.
\newblock Learning Equivalence Classes of \uppercase{B}ayesian-Network
  Structures.
\newblock {\em Journal of Machine Learning Research}, 2:\penalty0 445--498,
  2002.

\bibitem[Chow and Liu(1968)]{Chow68}
C.~Chow and C.~Liu.
\newblock Approximating Discrete Probability Distributions with Dependence
  Trees.
\newblock {\em IEEE Transactions on Information Theory}, 14\penalty0
  (3):\penalty0 462--467, 1968.

\bibitem[Daskalakis et~al.(2011)Daskalakis, Dimakisy, and Mossel]{Daskalakis11}
C.~Daskalakis, A.~Dimakisy, and E.~Mossel.
\newblock Connectivity and Equilibrium in Random Games.
\newblock {\em Annals of Applied Probability}, 21\penalty0 (3):\penalty0
  987--1016, 2011.

\bibitem[Daskalakis et~al.(2009)Daskalakis, Goldberg, and
  Papadimitriou]{daskalakisetal09}
C.~Daskalakis, P.~Goldberg, and C.~Papadimitriou.
\newblock The complexity of computing a \uppercase{N}ash equilibrium.
\newblock {\em Communications of the ACM}, 52\penalty0 (2):\penalty0 89--97,
  2009.

\bibitem[Daskalakis and Papadimitriou(2006)]{DP06}
C.~Daskalakis and C.H. Papadimitriou.
\newblock Computing Pure {Nash} Equilibria in Graphical Games via {Markov}
  Random Fields.
\newblock In {\em Proceedings of the 7th ACM Conference on Electronic
  Commerce}, EC '06, pages 91--99, New York, NY, USA, 2006. ACM.

\bibitem[Dilkina et~al.(2007)Dilkina, Gomes, and Sabharwal]{dilkinaetal07}
B.~Dilkina, C.~P. Gomes, and A.~Sabharwal.
\newblock The Impact of Network Topology on Pure {Nash} Equilibria in Graphical
  Games.
\newblock In {\em Proceedings of the 22nd National Conference on Artificial
  Intelligence - Volume 1}, AAAI'07, pages 42--49. AAAI Press, 2007.

\bibitem[Domingos(2005)]{domingos05}
P.~Domingos.
\newblock Mining Social Networks for Viral Marketing.
\newblock {\em IEEE Intelligent Systems}, 20\penalty0 (1):\penalty0 80--82,
  2005.
\newblock Short Paper.

\bibitem[Domingos and Richardson(2001)]{domingos01}
P.~Domingos and M.~Richardson.
\newblock Mining the Network Value of Customers.
\newblock In {\em Proceedings of the 7th ACM SIGKDD International Conference on
  Knowledge Discovery and Data Mining}, KDD '01, pages 57--66, New York, NY,
  USA, 2001. ACM.

\bibitem[Duchi and Singer(2009{\natexlab{a}})]{Duchi09}
J.~Duchi and Y.~Singer.
\newblock Efficient Learning using Forward-Backward Splitting.
\newblock In Y.~Bengio, D.~Schuurmans, J.D. Lafferty, C.K.I. Williams, and
  A.~Culotta, editors, {\em Advances in Neural Information Processing Systems
  22}, pages 495--503. Curran Associates, Inc., 2009{\natexlab{a}}.

\bibitem[Duchi and Singer(2009{\natexlab{b}})]{Duchi09c}
J.~Duchi and Y.~Singer.
\newblock Efficient Online and Batch Learning using Forward Backward Splitting.
\newblock {\em Journal of Machine Learning Research}, 10:\penalty0 2899--2934,
  2009{\natexlab{b}}.

\bibitem[Dunkel(2007)]{Dunkel06}
J.~Dunkel.
\newblock Complexity of Pure-Strategy \uppercase{N}ash Equilibria in
  Non-Cooperative Games.
\newblock In K.-H. Waldmann and U.~M. Stocker, editors, {\em Operations
  Research Proceedings}, volume 2006, pages 45--51. Springer Berlin Heidelberg,
  2007.

\bibitem[Dunkel and Schulz(2006)]{DunkelSchulz06}
J.~Dunkel and A.~Schulz.
\newblock On the Complexity of Pure-Strategy {Nash} Equilibria in Congestion
  and Local-Effect Games.
\newblock In P.~Spirakis, M.~Mavronicolas, and S.~Kontogiannis, editors, {\em
  Internet and Network Economics}, volume 4286 of {\em Lecture Notes in
  Computer Science}, pages 62--73. Springer Berlin Heidelberg, 2006.

\bibitem[Duong et~al.(2009)Duong, Vorobeychik, Singh, and Wellman]{Duong09}
Q.~Duong, Y.~Vorobeychik, S.~Singh, and M.P. Wellman.
\newblock Learning Graphical Game Models.
\newblock In {\em Proceedings of the 21st International Joint Conference on
  Artificial Intelligence}, IJCAI'09, pages 116--121, San Francisco, CA, USA,
  2009. Morgan Kaufmann Publishers Inc.

\bibitem[Duong et~al.(2008)Duong, Wellman, and Singh]{Duong08}
Q.~Duong, M.~Wellman, and S.~Singh.
\newblock Knowledge Combination in Graphical Multiagent Model.
\newblock In {\em Proceedings of the 24th Annual Conference on Uncertainty in
  Artificial Intelligence (UAI-08)}, pages 145--152, Corvallis, Oregon, 2008.
  AUAI Press.

\bibitem[Duong et~al.(2012)Duong, Wellman, Singh, and Kearns]{Duong12}
Q.~Duong, M.P. Wellman, S.~Singh, and M.~Kearns.
\newblock Learning and Predicting Dynamic Networked Behavior with Graphical
  Multiagent Models.
\newblock In {\em Proceedings of the 11th International Conference on
  Autonomous Agents and Multiagent Systems - Volume 1}, AAMAS '12, pages
  441--448, Richland, SC, 2012. International Foundation for Autonomous Agents
  and Multiagent Systems.

\bibitem[Duong et~al.(2010)Duong, Wellman, Singh, and Vorobeychik]{Duong10}
Q.~Duong, M.P. Wellman, S.~Singh, and Y.~Vorobeychik.
\newblock History-dependent Graphical Multiagent Models.
\newblock In {\em Proceedings of the 9th International Conference on Autonomous
  Agents and Multiagent Systems: Volume 1}, AAMAS '10, pages 1215--1222,
  Richland, SC, 2010. International Foundation for Autonomous Agents and
  Multiagent Systems.

\bibitem[Even-Dar and Shapira(2007)]{even-dar07}
E.~Even-Dar and A.~Shapira.
\newblock A Note on Maximizing the Spread of Influence in Social Networks.
\newblock In X.~Deng and F.~Graham, editors, {\em Internet and Network
  Economics}, volume 4858 of {\em Lecture Notes in Computer Science}, pages
  281--286. Springer Berlin Heidelberg, 2007.

\bibitem[Fabrikant et~al.(2004)Fabrikant, Papadimitriou, and
  Talwar]{fabrikantetal04}
A.~Fabrikant, C.~Papadimitriou, and K.~Talwar.
\newblock The Complexity of Pure {Nash} Equilibria.
\newblock In {\em Proceedings of the 36th Annual ACM Symposium on Theory of
  Computing}, STOC '04, pages 604--612, New York, NY, USA, 2004. ACM.

\bibitem[Ficici et~al.(2008)Ficici, Parkes, and Pfeffer]{Ficici08}
S.~Ficici, D.~Parkes, and A.~Pfeffer.
\newblock Learning and Solving Many-Player Games through a Cluster-Based
  Representation.
\newblock In {\em Proceedings of the 24th Annual Conference on Uncertainty in
  Artificial Intelligence (UAI-08)}, pages 188--195, Corvallis, Oregon, 2008.
  AUAI Press.

\bibitem[Fudenberg and Levine(1999)]{fudenberg99}
D.~Fudenberg and D.~Levine.
\newblock {\em The Theory of Learning in Games}.
\newblock MIT Press, 1999.

\bibitem[Fudenberg and Tirole(1991)]{fudenbergandtirole91}
D.~Fudenberg and J.~Tirole.
\newblock {\em Game Theory}.
\newblock The MIT Press, 1991.

\bibitem[Gao and Pfeffer(2010)]{Gao10}
X.~Gao and A.~Pfeffer.
\newblock Learning Game Representations from Data Using Rationality
  Constraints.
\newblock In {\em Proceedings of the 26th Annual Conference on Uncertainty in
  Artificial Intelligence (UAI-10)}, pages 185--192, Corvallis, Oregon, 2010.
  AUAI Press.

\bibitem[Gilboa and Zemel(1989)]{gilboa89}
I.~Gilboa and E.~Zemel.
\newblock Nash and correlated equilibria: some complexity considerations.
\newblock {\em Games and Economic Behavior}, 1\penalty0 (1):\penalty0 80--93,
  1989.

\bibitem[Gomez~Rodriguez et~al.(2010)Gomez~Rodriguez, Leskovec, and
  Krause]{GomezRodriguezetal10}
M.~Gomez~Rodriguez, J.~Leskovec, and A.~Krause.
\newblock Inferring Networks of Diffusion and Influence.
\newblock In {\em Proceedings of the 16th ACM SIGKDD International Conference
  on Knowledge Discovery and Data Mining}, KDD '10, pages 1019--1028, New York,
  NY, USA, 2010. ACM.

\bibitem[Gottlob et~al.(2005)Gottlob, Greco, and Scarcello]{gottlobetal05}
G.~Gottlob, G.~Greco, and F.~Scarcello.
\newblock Pure \uppercase{N}ash equilibria: Hard and easy games.
\newblock {\em Journal of Artificial Intelligence Research}, 24\penalty0
  (1):\penalty0 357--406, 2005.

\bibitem[Goyal et~al.(2010)Goyal, Bonchi, and Lakshmanan]{Goyaletal10}
A.~Goyal, F.~Bonchi, and L.V.S. Lakshmanan.
\newblock Learning Influence Probabilities in Social Networks.
\newblock In {\em Proceedings of the 3rd ACM International Conference on Web
  Search and Data Mining}, WSDM '10, pages 241--250, New York, NY, USA, 2010.
  ACM.

\bibitem[Granovetter(1978)]{granovetter78}
M.~Granovetter.
\newblock Threshold Models of Collective Behavior.
\newblock {\em The American Journal of Sociology}, 83\penalty0 (6):\penalty0
  1420--1443, 1978.

\bibitem[Gruzling(2006)]{Gruzling06}
N.~Gruzling.
\newblock Linear Separability of the Vertices of an $n$-Dimensional Hypercube.
\newblock Master's thesis, The University of British Columbia, 2006.

\bibitem[Guo et~al.(2010)Guo, Levina, Michailidis, and Zhu]{Guo10}
J.~Guo, E.~Levina, G.~Michailidis, and J.~Zhu.
\newblock Joint structure estimation for categorical \uppercase{M}arkov
  networks.
\newblock Technical report, University of Michigan, Department of Statistics,
  2010.
\newblock Submitted. \url{http://www. stat. lsa. umich. edu/\~elevina}.

\bibitem[Guo and Schuurmans(2006)]{Guo06}
Y.~Guo and D.~Schuurmans.
\newblock Convex Structure Learning for {Bayesian} Networks: Polynomial Feature
  Selection and Approximate Ordering.
\newblock In {\em Proceedings of the 22nd Annual Conference on Uncertainty in
  Artificial Intelligence (UAI-06)}, pages 208--216, Arlington, Virginia, 2006.
  AUAI Press.

\bibitem[Hasan and Galiana(2008)]{HasanGaliana08}
E.~Hasan and F.~Galiana.
\newblock Electricity Markets Cleared by Merit Order--Part \uppercase{II}:
  Strategic Offers and Market Power.
\newblock {\em IEEE Transactions on Power Systems}, 23\penalty0 (2):\penalty0
  372--379, 2008.

\bibitem[Hasan and Galiana(2010)]{HasanGaliana10}
E.~Hasan and F.~Galiana.
\newblock Fast Computation of Pure Strategy \uppercase{N}ash Equilibria in
  Electricity Markets Cleared by Merit Order.
\newblock {\em IEEE Transactions on Power Systems}, 25\penalty0 (2):\penalty0
  722--728, 2010.

\bibitem[Hasan et~al.(2008)Hasan, Galiana, and Conejo]{Hasan08}
E.~Hasan, F.~Galiana, and A.~Conejo.
\newblock Electricity Markets Cleared by Merit Order--Part \uppercase{I}:
  Finding the Market Outcomes Supported by Pure Strategy \uppercase{N}ash
  Equilibria.
\newblock {\em IEEE Transactions on Power Systems}, 23\penalty0 (2):\penalty0
  361--371, 2008.

\bibitem[Heal and Kunreuther(2003)]{healandkunreuther03}
G.~Heal and H.~Kunreuther.
\newblock You Only Die Once: Managing Discrete Interdependent Risks.
\newblock Working Paper W9885, National Bureau of Economic Research, 2003.

\bibitem[Heal and Kunreuther(2006)]{healandkunreuther06}
G.~Heal and H.~Kunreuther.
\newblock Supermodularity and Tipping.
\newblock Working Paper 12281, National Bureau of Economic Research, 2006.

\bibitem[Heal and Kunreuther(2007)]{healandkunreuther07}
G.~Heal and H.~Kunreuther.
\newblock Modeling Interdependent Risks.
\newblock {\em Risk Analysis}, 27:\penalty0 621--634, 2007.

\bibitem[H\"ofling and Tibshirani(2009)]{Hofling09}
H.~H\"ofling and R.~Tibshirani.
\newblock Estimation of Sparse Binary Pairwise \uppercase{M}arkov Networks
  using Pseudo-likelihoods.
\newblock {\em Journal of Machine Learning Research}, 10:\penalty0 883--906,
  2009.

\bibitem[Honorio(2012)]{Honorio_thesis12}
J.~Honorio.
\newblock {\em Tractable Learning of Graphical Model Structures from Data}.
\newblock PhD thesis, Stony Brook University, Department of Computer Science,
  2012.

\bibitem[Honorio and Ortiz(2012)]{Honorio_arxiv12}
J.~Honorio and L.~Ortiz.
\newblock Learning the Structure and Parameters of Large-Population Graphical
  Games from Behavioral Data.
\newblock {\em Computer Research Repository}, 2012.
\newblock \url{http://arxiv.org/abs/1206.3713}.

\bibitem[Irfan and Ortiz(2014)]{Irfan13}
M.~Irfan and L.~E. Ortiz.
\newblock On Influence, Stable Behavior, and the Most Influential Individuals
  in Networks: A Game-Theoretic Approach.
\newblock {\em Artificial Intelligence}, 215:\penalty0 79--119, 2014.

\bibitem[Janovskaja(1968)]{Janovskaja_1968_MR_by_Isbell}
E.~Janovskaja.
\newblock Equilibrium situations in multi-matrix games.
\newblock {\em Litovski\u\i\ Matematicheski\u\i\ Sbornik}, 8:\penalty0
  381--384, 1968.

\bibitem[Jiang and Leyton-Brown(2008)]{AGG-full}
A.~Jiang and K.~Leyton-Brown.
\newblock Action-Graph Games.
\newblock Technical Report TR-2008-13, University of British Columbia,
  Department of Computer Science, 2008.

\bibitem[Jiang and Leyton-Brown(2011)]{jiangandleytonbrown11}
A.X. Jiang and K.~Leyton-Brown.
\newblock Polynomial-time Computation of Exact Correlated Equilibrium in
  Compact Games.
\newblock In {\em Proceedings of the 12th ACM Conference on Electronic
  Commerce}, EC '11, pages 119--126, New York, NY, USA, 2011. ACM.

\bibitem[Kakade et~al.(2003)Kakade, Kearns, Langford, and Ortiz]{kakadeetal03}
S.~Kakade, M.~Kearns, J.~Langford, and L.~Ortiz.
\newblock Correlated Equilibria in Graphical Games.
\newblock In {\em Proceedings of the 4th ACM Conference on Electronic
  Commerce}, EC '03, pages 42--47, New York, NY, USA, 2003. ACM.

\bibitem[Kearns(2005)]{kearns05}
M.~Kearns.
\newblock Economics, Computer Science, and Policy.
\newblock {\em Issues in Science and Technology}, 2005.

\bibitem[Kearns et~al.(2001)Kearns, Littman, and Singh]{kearns01}
M.~Kearns, M.~Littman, and S.~Singh.
\newblock Graphical Models for Game Theory.
\newblock In {\em Proceedings of the 17th Annual Conference on Uncertainty in
  Artificial Intelligence (UAI-01)}, pages 253--260, San Francisco, CA, 2001.
  Morgan Kaufmann.

\bibitem[Kearns and Vazirani(1994)]{Kearns94}
M.~Kearns and U.~Vazirani.
\newblock {\em An Introduction to Computational Learning Theory}.
\newblock The MIT Press, 1994.

\bibitem[Kearns and Wortman(2008)]{KearnsW08}
M.~Kearns and J.~Wortman.
\newblock Learning from Collective Behavior.
\newblock In R.A. Servedio and T.~Zhang, editors, {\em 21st Annual Conference
  on Learning Theory - COLT 2008, Helsinki, Finland, July 9-12, 2008}, COLT
  '08, pages 99--110. Omnipress, 2008.

\bibitem[Kleinberg(2007)]{kleinberg07}
J.~Kleinberg.
\newblock Cascading Behavior in Networks: Algorithmic and Economic Issues.
\newblock In N.~Nisan, T.~Roughgarden, \'{E}. Tardos, and V.~V. Vazirani,
  editors, {\em Algorithmic Game Theory}, chapter~24, pages 613--632. Cambridge
  University Press, 2007.

\bibitem[Koller and Friedman(2009)]{Koller09}
D.~Koller and N.~Friedman.
\newblock {\em Probabilistic Graphical Models: Principles and Techniques}.
\newblock The MIT Press, 2009.

\bibitem[Koller and Milch(2003)]{koller03}
D.~Koller and B.~Milch.
\newblock Multi-agent influence diagrams for representing and solving games.
\newblock {\em Games and Economic Behavior}, 45\penalty0 (1):\penalty0
  181--221, 2003.

\bibitem[Kunreuther and Michel-Kerjan(2007)]{kunreutherandmichel-kerjan07}
H.~Kunreuther and E.~Michel-Kerjan.
\newblock Assessing, Managing and Benefiting from Global Interdependent Risks:
  The Case of Terrorism and Natural Disasters, 2007.
\newblock CREATE Symposium:
  \url{http://opim.wharton.upenn.edu/risk/library/AssessingRisks-2007.pdf}.

\bibitem[{La~Mura}(2000)]{lamura00}
P.~{La~Mura}.
\newblock Game Networks.
\newblock In {\em Proceedings of the 16th Annual Conference on Uncertainty in
  Artificial Intelligence (UAI-00)}, pages 335--342, San Francisco, CA, 2000.
  Morgan Kaufmann.

\bibitem[Lee et~al.(2007)Lee, Ganapathi, and Koller]{Lee06}
S.~Lee, V.~Ganapathi, and D.~Koller.
\newblock Efficient Structure Learning of {Markov} Networks using
  ${L}_1$-Regularization.
\newblock In B.~Sch\"{o}lkopf, J.C. Platt, and T.~Hoffman, editors, {\em
  Advances in Neural Information Processing Systems 19}, pages 817--824. MIT
  Press, 2007.

\bibitem[L\'opez-Pintado and Watts(2008)]{Lopez-Pintado01112008}
D.~L\'opez-Pintado and D.~Watts.
\newblock Social Influence, Binary Decisions and Collective Dynamics.
\newblock {\em Rationality and Society}, 20\penalty0 (4):\penalty0 399--443,
  2008.

\bibitem[McKelvey and Palfrey(1995)]{mckelvey95}
R.~McKelvey and T.~Palfrey.
\newblock Quantal Response Equilibria for Normal Form Games.
\newblock {\em Games and Economic Behavior}, 10\penalty0 (1):\penalty0 6--38,
  1995.

\bibitem[Morris(2000)]{morris00}
S.~Morris.
\newblock Contagion.
\newblock {\em The Review of Economic Studies}, 67\penalty0 (1):\penalty0
  57--78, 2000.

\bibitem[Muroga(1965)]{Muroga65}
S.~Muroga.
\newblock Lower bounds on the number of threshold functions and a maximum
  weight.
\newblock {\em IEEE Transactions on Electronic Computers}, 14:\penalty0
  136--148, 1965.

\bibitem[Muroga(1971)]{Muroga71}
S.~Muroga.
\newblock {\em Threshold Logic and Its Applications}.
\newblock John Wiley \& Sons, 1971.

\bibitem[Muroga and Toda(1966)]{Muroga66}
S.~Muroga and I.~Toda.
\newblock Lower Bound of the Number of Threshold Functions.
\newblock {\em IEEE Transactions on Electronic Computers}, 5:\penalty0
  805--806, 1966.

\bibitem[Nash(1951)]{nash51}
J.~Nash.
\newblock Non-cooperative games.
\newblock {\em Annals of Mathematics}, 54\penalty0 (2):\penalty0 286--295,
  1951.

\bibitem[Ng and Russell(2000)]{ngandrussell00}
A.~Y. Ng and S.~J. Russell.
\newblock Algorithms for Inverse Reinforcement Learning.
\newblock In {\em Proceedings of the 17th International Conference on Machine
  Learning}, ICML '00, pages 663--670, San Francisco, CA, USA, 2000. Morgan
  Kaufmann Publishers Inc.

\bibitem[Nisan et~al.(2007)Nisan, Roughgarden, Tardos, and Vazirani]{nisan07}
N.~Nisan, T.~Roughgarden, \'{E}. Tardos, and V.~V. Vazirani, editors.
\newblock {\em Algorithmic Game Theory}.
\newblock Cambridge University Press, 2007.

\bibitem[Ortiz and Kearns(2003)]{ortizandkearns02}
L.~E. Ortiz and M.~Kearns.
\newblock Nash Propagation for Loopy Graphical Games.
\newblock In S.~Becker, S.~Thrun, and K.~Obermayer, editors, {\em Advances in
  Neural Information Processing Systems 15}, pages 817--824. MIT Press, 2003.

\bibitem[Papadimitriou and Roughgarden(2008)]{papadimitriou08}
C.~Papadimitriou and T.~Roughgarden.
\newblock Computing correlated equilibria in multi-player games.
\newblock {\em Journal of the ACM}, 55\penalty0 (3):\penalty0 1--29, 2008.

\bibitem[Rinott and Scarsini(2000)]{Rinott2000274}
Y.~Rinott and M.~Scarsini.
\newblock On the Number of Pure Strategy \uppercase{N}ash Equilibria in Random
  Games.
\newblock {\em Games and Economic Behavior}, 33\penalty0 (2):\penalty0
  274--293, 2000.

\bibitem[Rosenthal(1973)]{rosenthal73}
R.~Rosenthal.
\newblock A class of games possessing pure-strategy \uppercase{N}ash
  equilibria.
\newblock {\em International Journal of Game Theory}, 2\penalty0 (1):\penalty0
  65--67, 1973.

\bibitem[Ryan et~al.(2010)Ryan, Jiang, and Leyton-Brown]{Ryan10}
C.T. Ryan, A.X. Jiang, and K.~Leyton-Brown.
\newblock Computing Pure Strategy {Nash} Equilibria in Compact Symmetric Games.
\newblock In {\em Proceedings of the 11th ACM Conference on Electronic
  Commerce}, EC '10, pages 63--72, New York, NY, USA, 2010. ACM.

\bibitem[Saito et~al.(2009)Saito, Kimura, Ohara, and Motoda]{Saitoetal09}
K.~Saito, M.~Kimura, K.~Ohara, and H.~Motoda.
\newblock Learning Continuous-Time Information Diffusion Model for Social
  Behavioral Data Analysis.
\newblock In Z.~Zhou and T.~Washio, editors, {\em Advances in Machine
  Learning}, volume 5828 of {\em Lecture Notes in Computer Science}, pages
  322--337. Springer Berlin Heidelberg, 2009.

\bibitem[Saito et~al.(2010)Saito, Kimura, Ohara, and Motoda]{Saitoetal10}
K.~Saito, M.~Kimura, K.~Ohara, and H.~Motoda.
\newblock Selecting Information Diffusion Models over Social Networks for
  Behavioral Analysis.
\newblock In {J.~L.} Balc\'{a}zar, F.~Bonchi, A.~Gionis, and M.~Sebag, editors,
  {\em Machine Learning and Knowledge Discovery in Databases}, volume 6323 of
  {\em Lecture Notes in Computer Science}, pages 180--195. Springer Berlin
  Heidelberg, 2010.

\bibitem[Saito et~al.(2008)Saito, Nakano, and Kimura]{Saitoetal08}
K.~Saito, R.~Nakano, and M.~Kimura.
\newblock Prediction of Information Diffusion Probabilities for Independent
  Cascade Model.
\newblock In I.~Lovrek, {R.~J.} Howlett, and {L.~C.} Jain, editors, {\em
  Knowledge-Based Intelligent Information and Engineering Systems}, volume 5179
  of {\em Lecture Notes in Computer Science}, pages 67--75. Springer Berlin
  Heidelberg, 2008.

\bibitem[Schmidt et~al.(2007{\natexlab{a}})Schmidt, Fung, and
  Rosales]{Schmidt07b}
M.~Schmidt, G.~Fung, and R.~Rosales.
\newblock Fast Optimization Methods for L1 Regularization: A Comparative Study
  and Two New Approaches.
\newblock In {\em Proceedings of the 18th European Conference on Machine
  Learning}, ECML '07, pages 286--297, Berlin, Heidelberg, 2007{\natexlab{a}}.
  Springer-Verlag.

\bibitem[Schmidt and Murphy(2009)]{Schmidt09}
M.~Schmidt and K.~Murphy.
\newblock Modeling Discrete Interventional Data using Directed Cyclic Graphical
  Models.
\newblock In {\em Proceedings of the 25th Annual Conference on Uncertainty in
  Artificial Intelligence (UAI-09)}, pages 487--495, Corvallis, Oregon, 2009.
  AUAI Press.

\bibitem[Schmidt et~al.(2007{\natexlab{b}})Schmidt, Niculescu-Mizil, and
  Murphy]{Schmidt07}
M.~Schmidt, A.~Niculescu-Mizil, and K.~Murphy.
\newblock Learning Graphical Model Structure Using ${\ell }_1$-regularization
  Paths.
\newblock In {\em Proceedings of the 22nd National Conference on Artificial
  Intelligence - Volume 2}, pages 1278--1283, 2007{\natexlab{b}}.

\bibitem[Shoham(2008)]{shoham08}
Y.~Shoham.
\newblock Computer Science and Game Theory.
\newblock {\em Communications of the ACM}, 51\penalty0 (8):\penalty0 74--79,
  2008.

\bibitem[Shoham and Leyton-Brown(2009)]{shoham09}
Y.~Shoham and K.~Leyton-Brown.
\newblock {\em Multiagent Systems: {A}lgorithmic, Game-Theoretic, and Logical
  Foundations}.
\newblock Cambridge University Press, 2009.

\bibitem[Sontag(1998)]{sontag98vcdimension}
E.~Sontag.
\newblock \uppercase{VC} Dimension of Neural Networks.
\newblock {\em Neural Networks and Machine Learning}, pages 69--95, 1998.

\bibitem[Srebro(2001)]{Srebro01}
N.~Srebro.
\newblock Maximum Likelihood Bounded Tree-Width {Markov} Networks.
\newblock In {\em Proceedings of the 17th Annual Conference on Uncertainty in
  Artificial Intelligence (UAI-01)}, pages 504--511, San Francisco, CA, 2001.
  Morgan Kaufmann.

\bibitem[Stanford(1995)]{Stanford1995238}
W.~Stanford.
\newblock A Note on the Probability of $k$ Pure \uppercase{N}ash Equilibria in
  Matrix Games.
\newblock {\em Games and Economic Behavior}, 9\penalty0 (2):\penalty0 238--246,
  1995.

\bibitem[Sureka and Wurman(2005)]{Sureka05}
A.~Sureka and P.R. Wurman.
\newblock Using Tabu Best-response Search to Find Pure Strategy {Nash}
  Equilibria in Normal Form Games.
\newblock In {\em Proceedings of the 4th International Joint Conference on
  Autonomous Agents and Multiagent Systems}, AAMAS '05, pages 1023--1029, New
  York, NY, USA, 2005. ACM.

\bibitem[Vapnik(1998)]{vapnikSLT}
V.~Vapnik.
\newblock {\em Statistical Learning Theory}.
\newblock Wiley, 1998.

\bibitem[Vickrey and Koller(2002)]{vickreyandkoller02}
D.~Vickrey and D.~Koller.
\newblock Multi-agent Algorithms for Solving Graphical Games.
\newblock In {\em 18th National Conference on Artificial Intelligence}, pages
  345--351, Menlo Park, CA, USA, 2002. American Association for Artificial
  Intelligence.

\bibitem[Vorobeychik et~al.(2007)Vorobeychik, Wellman, and
  Singh]{Vorobeychik07}
Y.~Vorobeychik, M.~Wellman, and S.~Singh.
\newblock Learning payoff functions in infinite games.
\newblock {\em Machine Learning}, 67\penalty0 (1-2):\penalty0 145--168, 2007.

\bibitem[Wainwright et~al.(2007)Wainwright, Lafferty, and
  Ravikumar]{Wainwright06}
M.J. Wainwright, J.D. Lafferty, and P.K. Ravikumar.
\newblock High-Dimensional Graphical Model Selection Using ${\ell
  }_1$-Regularized Logistic Regression.
\newblock In B.~Sch\"{o}lkopf, J.C. Platt, and T.~Hoffman, editors, {\em
  Advances in Neural Information Processing Systems 19}, pages 1465--1472. MIT
  Press, 2007.

\bibitem[Waugh et~al.(2011)Waugh, Ziebart, and Bagnell]{Waugh11}
K.~Waugh, B.~Ziebart, and D.~Bagnell.
\newblock Computational Rationalization: The Inverse Equilibrium Problem.
\newblock In Lise Getoor and Tobias Scheffer, editors, {\em Proceedings of the
  28th International Conference on Machine Learning (ICML-11)}, ICML '11, pages
  1169--1176, New York, NY, USA, 2011. ACM.

\bibitem[Winder(1960)]{Winder60}
R.~Winder.
\newblock Single state threshold logic.
\newblock {\em Switching Circuit Theory and Logical Design}, S-134:\penalty0
  321--332, 1960.

\bibitem[Wright and Leyton-Brown(2010)]{Wright10}
J.~Wright and K.~Leyton-Brown.
\newblock Beyond Equilibrium: Predicting Human Behavior in Normal-Form Games.
\newblock In {\em AAAI Conference on Artificial Intelligence}, 2010.

\bibitem[Wright and Leyton-Brown(2012)]{Wright12}
J.R. Wright and K.~Leyton-Brown.
\newblock Behavioral Game Theoretic Models: A {Bayesian} Framework for
  Parameter Analysis.
\newblock In {\em Proceedings of the 11th International Conference on
  Autonomous Agents and Multiagent Systems - Volume 2}, AAMAS '12, pages
  921--930, Richland, SC, 2012. International Foundation for Autonomous Agents
  and Multiagent Systems.

\bibitem[Yamija and Ibaraki(1965)]{Yamija65}
S.~Yamija and T.~Ibaraki.
\newblock A lower bound of the number of threshold functions.
\newblock {\em IEEE Transactions on Electronic Computers}, 14:\penalty0
  926--929, 1965.

\bibitem[Zhu et~al.(2004)Zhu, Rosset, Tibshirani, and Hastie]{Zhu03}
J.~Zhu, S.~Rosset, R.~Tibshirani, and T.J. Hastie.
\newblock 1-norm Support Vector Machines.
\newblock In S.~Thrun, L.K. Saul, and B.~Sch\"{o}lkopf, editors, {\em Advances
  in Neural Information Processing Systems 16}, pages 49--56. MIT Press, 2004.

\bibitem[Ziebart et~al.(2010)Ziebart, Bagnell, and Dey]{Ziebart10}
B.~Ziebart, J.~A. Bagnell, and A.~K. Dey.
\newblock Modeling Interaction via the Principle of Maximum Causal Entropy.
\newblock In Johannes F{\"u}rnkranz and Thorsten Joachims, editors, {\em
  Proceedings of the 27th International Conference on Machine Learning
  (ICML-10)}, pages 1255--1262, Haifa, Israel, 2010. Omnipress.

\end{thebibliography}

\end{document}